\documentclass[nohyperref]{article}

\usepackage{microtype}
\usepackage{graphicx}
\usepackage{subfigure}
\usepackage{booktabs} % for professional tables

\usepackage[hypertexnames=false]{hyperref}

\usepackage[accepted]{icml2022}

\usepackage{amsmath}
\usepackage{amssymb}
\usepackage{mathtools}
\usepackage{amsthm}

\usepackage[capitalize,noabbrev]{cleveref}

\theoremstyle{plain}
\newtheorem{theorem}{Theorem}[section]
\newtheorem{proposition}[theorem]{Proposition}

\theoremstyle{definition}
\newtheorem{definition}[theorem]{Definition}

\theoremstyle{remark}
\newtheorem{remark}[theorem]{Remark}

\icmltitlerunning{Nonparametric Involutive Markov Chain Monte Carlo}

\newif\ifdraft
\newif\ifapxAppended
\apxAppendedtrue

\usepackage[T1]{fontenc}
\usepackage[utf8]{inputenc}

\usepackage{amsthm} % advanced maths symbols
\usepackage{amssymb}
\usepackage{stmaryrd}
\usepackage{mathtools}
\usepackage{listings} % listings for writing code
\usepackage{tikz}
\usetikzlibrary{trees}
\usetikzlibrary{calc}
\usetikzlibrary{fit}
\usetikzlibrary{shapes.geometric}
\usetikzlibrary{arrows}

\tikzstyle{state}=[draw, thick, inner sep=0.2em]
\tikzstyle{vector}=[draw, fill=blue!20, minimum height=10em, minimum width=1.3em]
\tikzstyle{elem}=[draw, fill=blue!20, minimum height=1.3em, minimum width=1.3em]

\usepackage{calculation}
\usepackage{hyperref}
\usepackage{bussproofs} % proof trees
\usepackage{paralist}
\usepackage{bbold} % nicer bb maths
\usepackage{titletoc}

\ifdraft
\usepackage[draft,nompar]{commenting}
\newcommand\lochanged[1]{{\color{red!60!yellow} #1}}
\else
\usepackage[final]{commenting}
\newcommand\lochanged[1]{}
\fi
\declareauthor{lo}{lo}{red!60!yellow}
\authorcommand{lo}{comment}
\declareauthor{cm}{cm}{green!50!black}
\authorcommand{cm}{comment}
\setdefaultauthor{cm}
\declareauthor{fz}{fz}{purple}
\authorcommand{fz}{comment}
\declareauthor{anon}{anon}{blue}
\authorcommand{anon}{comment}

\Crefname{theorem}{Thm.}{Thm.}
\Crefname{corollary}{Cor.}{Corollary}
\crefname{proposition}{Prop.}{Propositions}
\Crefname{claim}{Claim}{Claims}
\Crefname{definition}{Def.}{Definitions}
\Crefname{fact}{Fact}{Facts}
\Crefname{conjecture}{Conj.}{Conjectures}
\Crefname{example}{Ex.}{Ex.}
\Crefname{remark}{Rem.}{Remarks}
\Crefname{convention}{Convention}{Conventions}
\Crefname{lemma}{Lem.}{Lemmas}
\Crefname{assumption}{Ass.}{Ass.}
\Crefname{section}{Sec.}{Sec.}
\Crefname{appendix}{App.}{App.}
\Crefname{figure}{Fig.}{Fig.}
\Crefname{algorithm}{Alg.}{Alg.}
\Crefname{listing}{Code}{Codes}
\Crefname{lstlisting}{Listing}{Listings}
\Crefname{trick}{Trick}{Tricks}

\usepackage{thmtools}
\usepackage{thm-restate}

\newtheorem{example}{Example}

\usepackage{url}

\Crefname{step}{Step}{Step}
\newcounter{step}
\newcommand*\step[1][]{\refstepcounter{step}~}

\Crefname{hstep}{Step}{Step}
\newcounter{hstep}
\newcommand*\hstep[1][]{\refstepcounter{hstep}~}

\Crefname{vass}{V\!}{V\!}
\newcounter{vass}
\newcommand*\vass[1][]{\refstepcounter{vass} \noindent (V\thevass) {#1}}

\Crefname{hass}{H\!}{H\!}
\newcounter{hass}
\newcommand*\hass[1][]{\refstepcounter{hass} \noindent (H\thehass) {#1}}

\Crefname{cass}{C\!}{C\!}
\newcounter{cass}

\newcommand{\allvass}{\cref{vass: integrable tree,vass: almost surely terminating tree,vass: partial block diagonal inv}}
\newcommand{\allhass}{\cref{hass: integrable tree,hass: almost surely terminating tree,hass: partial block diagonal inv}}

\newcommand{\Nat}{\mathsf{\mathbb{N}}}

\newcommand{\Real}{\mathbb{R}}
\newcommand{\Bool}{\mathbb{2}}
\newcommand{\Dyadic}{\mathbb{D}}
\newcommand{\pReal}{[0,\infty)}

\newcommand{\Borel}{\mathcal{B}}

\newcommand{\dif}{\textrm{d}}
\newcommand{\expandint}[4]{\int_{#1}\ {#2}\ {#3}(\dif{#4})}
\newcommand{\shortint}[3]{\int_{#1}\ {#2}\ \dif{#3}}
\newcommand{\concat}{\mathbin{+\mkern-8mu+}}
\newcommand{\partialto}{\rightharpoonup}

\newcommand{\inv}[1]{{{#1}^{-1}}}
\newcommand{\abs}[1]{|{#1}|}
\newcommand{\id}{\mathsf{id}}
\newcommand{\proj}[2]{\mathsf{take}_{{#2}}}
\newcommand{\take}[1]{\mathsf{take}_{{#1}}}
\newcommand{\drop}[1]{\mathsf{drop}_{{#1}}}
\newcommand{\placeholder}{\cdot}

\newcommand{\domain}[1]{\mathsf{Dom}(#1)}
\newcommand{\support}[1]{\text{Supp}(#1)}
\newcommand{\nsupport}[2]{\text{Supp}^{#2}(#1)}

\newcommand{\anbr}[1]{\langle #1\rangle}

\newcommand{\set}[1]{\{#1\}}
\newcommand{\floor}[1]{\lfloor {#1}\rfloor}

\newcommand{\true}{\mathsf{T}}
\newcommand{\false}{\mathsf{F}}

\newcommand{\kernelto}{\leadsto}

\newcommand{\measure}[1]{\mu_{#1}}
\newcommand{\pdf}[1]{\mathsf{pdf}_{#1}}

\newcommand{\Gau}{\mathcal{N}}
\newcommand{\pdfGau}{\varphi}

\newcommand{\Uni}{\mathcal{U}}
\newcommand{\Bern}{\mathsf{Bern}}

\newcommand{\algebra}[1]{\Sigma_{#1}}

\newcommand{\powerset}{\mathcal{P}}

\renewcommand\vec[1]{\boldsymbol{#1}}

\newcommand{\grad}[1]{{\nabla #1}}
\newcommand{\idmat}{\boldsymbol{I}}

\newcommand{\range}[3]{{#1}^{{#2}..{#3}}}
\newcommand{\seqindex}[2]{{#1}^{#2}}

\newcommand{\len}[1]{|{#1}|}

\newcommand{\states}{\mathbb{S}}
\newcommand{\state}{\boldsymbol{s}}
\newcommand{\stateij}[3]{S_{{#1}{#2}}^{({#3})}}
\newcommand{\stateset}{S}
\newcommand{\smeasure}{\mu_{\states}}
\newcommand{\validstates}{\states^{\textrm{valid}}}

\newcommand{\transkernel}[1]{T_{#1}}
\newcommand{\marg}{\mathfrak{m}}

\newcommand{\sdist}{\pi}
\newcommand{\spdf}{\zeta}
\newcommand{\tdist}{\nu}
\newcommand{\tpdf}{\rho}
\newcommand{\tree}{w}
\newcommand{\trunc}[1]{\extree}
\newcommand{\extree}{\hat{\tree}}
\newcommand{\invo}{{\Phi}}
\newcommand{\ninvo}[1]{{\Phi^{(#1)}}}
\newcommand{\ninvoij}[3]{\Phi_{{#1}{#2}}^{({#3})}}

\newcommand{\samspace}{\mathbb{\Omega}}
\newcommand{\entspace}{\mathbb{X}}
\newcommand{\entstockpdf}{\varphi_\entspace}
\newcommand{\instances}[1]{\mathsf{instance}(#1)}

\newcommand{\parspace}{\mathbb{X}}
\newcommand{\parstockpdf}{\varphi_\parspace}
\newcommand{\auxspace}{\mathbb{Y}}
\newcommand{\auxstockpdf}{\varphi_\auxspace}

\newcommand{\nauxkernel}[1]{{K^{(#1)}}}
\newcommand{\nauxkernelpdf}[1]{\mathsf{pdf}{K^{(#1)}}}

\newcommand{\exauxkernel}{{\hat{K}}}
\newcommand{\exauxpdf}{\mathsf{pdf}{\hat{K}}}

\newcommand{\nstockpdf}[1]{\varphi_{\entspace^{#1}\times\auxspace^{#1}}}

\newcommand{\ikernel}{{K}}
\newcommand{\iinv}{{\Phi}}
\newcommand{\ikpdf}{{\mathsf{pdf}}\ikernel}

\newcommand{\invoass}{projection commutation property}

\newcommand{\accept}{\alpha}

\newcommand{\enta}{\vec{x}}
\newcommand{\entb}{\vec{y}}
\newcommand{\entc}{\vec{z}}
\newcommand{\entelem}{e}
\newcommand{\entseta}{X}

\newcommand{\auxa}{\vec{v}}
\newcommand{\auxb}{\vec{u}}
\newcommand{\auxc}{\vec{w}}
\newcommand{\auxelem}{p}
\newcommand{\auxseta}{V}
\newcommand{\auxsetb}{U}

\newcommand{\imcmca}{\iota}

\newcommand{\mixkernel}{K_M}
\newcommand{\mixpdf}{{\mathsf{pdf}K_M}}

\newcommand{\mixa}{m}

\newcommand{\dirspace}{\mathbb{D}}
\newcommand{\dira}{d}
\newcommand{\leapfrog}{\boldsymbol{L}}
\newcommand{\data}{\mathcal{D}}

\newcommand{\samsp}{\mathbb{\Omega}}

\newcommand{\nstates}[1]{\mathbb{S}^{(#1)}}

\newcommand{\fij}[2]{f_{{#1}\to{#2}}}

\newcommand{\entsp}{{\mathbb{E}}}
\newcommand{\entpdf}[1]{\varphi_{\entsp^{#1}}}
\newcommand{\parsp}{{\mathbb{X}}}
\newcommand{\nparsp}[1]{{\parsp^{({#1})}}}

\newcommand{\nparpdf}[1]{\varphi_{\nparsp{#1}}}
\newcommand{\iparsp}{\iota_{\parsp}}
\newcommand{\auxsp}{{\mathbb{Y}}}
\newcommand{\nauxsp}[1]{{\auxsp^{({#1})}}}
\newcommand{\nauxpdf}[1]{\varphi_{\nauxsp{#1}}}
\newcommand{\iauxsp}{\iota_{\auxsp}}

\newcommand{\nkernel}[1]{K^{({#1})}}
\newcommand{\nkernelpdf}[1]{{\mathsf{pdf}K^{({#1})}}}

\newcommand{\posa}{\vec{q}}
\newcommand{\moma}{\vec{p}}

\newcommand{\slicefn}[1]{s^{({#1})}}

\newcommand{\nelem}[2]{{f^{({#1})}_{{#2}}}}

\newcommand{\mixsp}{{\entsp^{\alpha}}}

\newcommand{\nbij}[1]{f^{(#1)}}

\newcommand{\momflip}{M}
\newcommand{\momstep}{\phi^M}
\newcommand{\posstep}{\phi^P}

\tikzstyle{decision} = [diamond, draw, fill=blue!20,
    text width=4.5em, text badly centered, inner sep=0pt,font=\footnotesize]
\tikzstyle{block} = [rectangle, draw, fill=blue!20,
    text width=6em, text centered, rounded corners, inner sep=5pt, font=\footnotesize]
\tikzstyle{algblock} = [rectangle, thick, inner sep=25pt, rounded corners]
\tikzstyle{line} = [draw,-stealth]

\definecolor{oxfordblue}{RGB}{0, 33, 71}
\definecolor{teal}{RGB}{0, 128, 128}
\definecolor{midnightblue}{RGB}{25, 25, 112}

\definecolor{deepblue}{rgb}{0,0,0.5}
\definecolor{purple}{rgb}{0.5,0,0.5}
\definecolor{deepred}{rgb}{0.6,0,0}
\definecolor{deepgreen}{rgb}{0,0.5,0}
\definecolor{teal}{rgb}{0.0, 0.51, 0.5}
\definecolor{background}{rgb}{0.95,0.95,0.95}
\definecolor{orange-red}{rgb}{1.0, 0.27, 0.0}
\definecolor{phthaloblue}{rgb}{0.0, 0.06, 0.54}

\makeatletter
\newenvironment{btHighlight}[1][]
{\begingroup\tikzset{bt@Highlight@par/.style={#1}}\begin{lrbox}{\@tempboxa}}
{\end{lrbox}\bt@HL@box[bt@Highlight@par]{\@tempboxa}\endgroup}

\newcommand\btHL[1][]{%
  \begin{btHighlight}[#1]\bgroup\aftergroup\bt@HL@endenv%
}
\def\bt@HL@endenv{%
  \end{btHighlight}%
  \egroup
}
\newcommand{\bt@HL@box}[2][]{%
  \tikz[#1]{%
    \pgfpathrectangle{\pgfpoint{1pt}{0pt}}{\pgfpoint{\wd #2}{\ht #2}}%
    \pgfusepath{use as bounding box}%
    \node[anchor=base west, fill=yellow!30,outer sep=0pt,inner xsep=1pt, inner ysep=0pt, rounded corners=3pt, minimum height=\ht\strutbox+1pt,#1]{\raisebox{1pt}{\strut}\strut\usebox{#2}};
  }%
}
\makeatother

\newcommand\pythonstyle{\lstset{
  language=Python,
  basicstyle=\ttfamily\footnotesize,
  keywordstyle=\em\color{purple},
  commentstyle=\color{white!55!black},
  showstringspaces=false,
  breaklines=true,
  basewidth=0.55em,
  backgroundcolor = \color{background},
  numberstyle=\ttfamily\color{white!55!black},
  numbers=left,
  numbersep=15pt,
  rulecolor=\color{white!90!black},
  xleftmargin=.05\textwidth, xrightmargin=.05\textwidth,
  frame=single,
  frameround=tttt,
  framesep=10pt,
  moredelim=**[is][\btHL]{`}{`},
  moredelim=**[is][{\btHL[fill=green!30,draw=red,dashed,thin]}]{@}{@},
  escapeinside={(@}{@)},
  morekeywords={
    append,range,min,len,sum,pop,from,log,domain,type,product,then,let,True,False,intersect
  },
  emph={[1]
    score,normal,uniform,sample,observe,coin,grad
  },
  emphstyle={[1]\em\ttfamily\color{teal}},
  emph={[2]
    leapfrog, momflip, nleapfrog,LAkernel,
    extend,NPint,
    w, auxkernel, pdfauxkernel, involution, absdetjacinv,
    pdfnormal,
    q, pdfq,
    mixkernel, pdfmixkernel,
    bijection, absdetjacbij,
    slice,
    f, absdetjacf,
    momflip, momstep, posstep,
    momflipslice, momslice, posslice,
    leapfrogslice,
  },
  emphstyle={[2]\ttfamily\color{orange-red}},
  emph={[3]
    pdfpar, pdfaux, indexX, indexY,
    dim,proj,support,instance
  },
  emphstyle={[3]\ttfamily\color{phthaloblue}},
  emph={[4]
    iMCMC,MixtureiMCMC,DirectioniMCMC,PersistentiMCMC,
    MH,HMC,
    NPHMC,NPHMCstep,validstate,HMCint,eNPHMC,eNPHMCstep,accept,supported,
    GenHMC,PersistentHMC,CorruptMom, GenHMCw,
    NPiMCMCstep, eNPiMCMC, eNPiMCMCstep, NPinv, fill, empty, perm,
    AuxUpdate,extree,expdfaux,
    iGMM, NPiMCMC, NPMH, MixtureNPiMCMC, DirectionNPiMCMC, PersistentNPiMCMC,
    MultistepNPiMCMC,
    MixtureMSNPiMCMC,
    DirectionMSNPiMCMC,
    PersistentMSNPiMCMC,
    NPHMC, NPDHMC,
    NPHMCwPersistent,CorruptMom,PersistMom, HMCw,
    NPLookAheadHMC,ExtraLeapfrog,
    LiftedNPMH,NPMHwP,
  },
  emphstyle={[4]\ttfamily\color{deepgreen}},
  emph={[5]
    mixauxkernel, mixinvolution, mixindexX, mixindexY, mixproj,pdfmixauxkernel,
    mixf, mixslice, absdetjacmixf,
    dirauxkernel, dirinvolution, dirindexX, dirindexY, dirproj,pdfdirauxkernel,
    dirf, dirslice, absdetjacdirf, dirL,
    perauxkernel, perinvolution, perindexX, perindexY, perproj,pdfperauxkernel,
    perf, perslice, absdetjacperf, perL, flipdir,
  },
  emphstyle={[5]\ttfamily\color{deepred}},
}}

\lstnewenvironment{code}[1][]
{
\pythonstyle
\lstset{#1}
}
{}

\newcommand\codeinline[1]{\colorbox{background}{\pythonstyle\lstinline!#1!}}
\newcommand\codein[1]{\colorbox{background}{\pythonstyle\lstinline!#1!}}

\newcommand{\terms}{\Lambda}
\newcommand{\closedterms}{\Lambda^0}

\newcommand{\closedvalues}{\Lambda^0_v}
\newcommand{\sk}{\mathsf{SK}}

\newcommand{\primitives}{\mathcal{F}}
\newcommand{\pcf}[1]{\underline{#1}}
\newcommand{\pcfif}[3]{\mathsf{if}(#1, #2, #3)}
\newcommand{\normal}{\mathsf{normal}}
\newcommand{\coin}{\mathsf{coin}}
\newcommand{\score}[1]{\mathsf{score}(#1)}
\newcommand{\Y}[1]{\mathsf{Y}{#1}}

\newcommand{\freevar}[1]{\mathsf{FV}(#1)}

\newcommand{\terma}{M}
\newcommand{\termb}{N}
\newcommand{\termc}{L}
\newcommand{\reala}{r}
\newcommand{\realb}{q}
\newcommand{\realc}{p}
\newcommand{\realveca}{\vec{r}}
\newcommand{\realvecb}{\vec{q}}

\newcommand{\realset}{R}
\newcommand{\boola}{a}
\newcommand{\boolb}{b}
\newcommand{\boolveca}{\vec{i}}
\newcommand{\boolvecb}{\vec{j}}
\newcommand{\boolset}{B}
\newcommand{\constant}{c}
\newcommand{\funca}{f}
\newcommand{\funcb}{g}
\newcommand{\funcc}{h}

\newcommand{\tyreal}{\mathsf{R}}
\newcommand{\tybool}{\mathsf{B}}
\newcommand{\tyarrow}{\Rightarrow}
\newcommand{\type}[1]{\mathsf{Type}(#1)}
\newcommand{\typair}[2]{\mathsf{Pair}(#1,#2)}
\newcommand{\tylist}[1]{\mathsf{List}(#1)}

\newcommand{\typea}{\sigma}
\newcommand{\typeb}{\tau}
\newcommand{\groundspace}{G}

\newcommand{\valuea}{V}
\newcommand{\redexa}{R}
\newcommand{\contra}{\Delta}
\newcommand{\evalcon}{E}
\newcommand{\config}[3]{\anbr{#1,#2,#3}}
\newcommand{\red}{\longrightarrow}
\newcommand{\emptytrace}{[]}
\newcommand{\Fail}{\mathsf{fail}}
\newcommand{\redplus}{\red^+}
\newcommand{\redstar}{\red^*}

\newcommand{\traces}{\mathbb{T}}
\newcommand{\maxtraces}{\mathbb{T}_{\mathsf{max}}}
\newcommand{\tertraces}{\mathbb{T}_{\mathsf{ter}}}
\newcommand{\tmeasure}{\measure{\traces}}

\newcommand{\valuefn}[1]{\mathsf{value}_{#1}}
\newcommand{\weightfn}[1]{\mathsf{weight}_{#1}}
\newcommand{\oper}[1]{\langle\!\langle{#1}\rangle\!\rangle}

\newcommand{\zip}{\mathsf{zip}}

\newcommand{\trace}{t}
\newcommand{\traceb}{\boldsymbol{\trace}}
\newcommand{\traceset}{A}
\usepackage[normalem]{ulem}

\newcommand{\mathhl}[1]{\colorbox{yellow!50}{$\displaystyle #1$}}
\newcommand{\defn}[1]{\emph{\textbf {#1}}}

\newcommand\mypara[1]{\paragraph{#1}}
\begin{document}

\twocolumn[
\icmltitle{Nonparametric Involutive Markov Chain Monte Carlo}



\icmlsetsymbol{equal}{*}

\begin{icmlauthorlist}
\icmlauthor{Carol Mak}{oxford}
\icmlauthor{Fabian Zaiser}{oxford}
\icmlauthor{Luke Ong}{oxford}
\end{icmlauthorlist}

\icmlaffiliation{oxford}{Department of Computer Science, University of Oxford, United Kingdom}

\icmlcorrespondingauthor{Carol Mak}{pui.mak@cs.ox.ac.uk}

\icmlkeywords{Inference algorithm, MCMC, Involutive MCMC}

\vskip 0.3in
]



\printAffiliationsAndNotice{}  % leave blank if no need to mention equal contribution

\begin{abstract}
A challenging problem in probabilistic programming is to develop inference algorithms that work for arbitrary programs in a universal probabilistic programming language (PPL).
%The density functions defined by such programs, which may use stochastic branching and recursion, are in general \emph{nonparametric}, in that they express models on an infinite-dimensional parameter space.
%However standard inference algorithms \changed[fz]{usually} target distributions with a fixed finite number of parameters.
We present the \emph{nonparametric involutive Markov chain Monte Carlo} (NP-iMCMC) algorithm as a method for constructing MCMC inference algorithms for nonparametric models expressible in universal PPLs.
Building on the unifying \emph{involutive MCMC} framework, and by providing a \emph{general} procedure for driving state movement between dimensions,
we show that NP-iMCMC can generalise numerous existing iMCMC algorithms to work on nonparametric models.
We prove the correctness of the NP-iMCMC sampler.
Our empirical study shows that the existing strengths of several iMCMC algorithms carry over to their nonparametric extensions.
Applying our method to the recently proposed Nonparametric HMC,
\changed[cm]{an instance of (Multiple Step) NP-iMCMC,}
we have constructed several nonparametric extensions (all of which new) that exhibit significant performance improvements.
\end{abstract}

\section{Introduction}
\label{sec: introduction}
% !TEX root = ./../icml2022.tex

\emph{Universal probabilistic programming} \cite{DBLP:conf/uai/GoodmanMRBT08} is the idea of writing probabilistic models in a Turing-complete programming language.
A universal probabilistic programming language (PPL) can express all computable probabilistic models \cite{VakarKS19}, using only a handful of basic programming constructs such as branching and recursion.
In particular, \defn{nonparametric models}, where the number of random variables is not determined \emph{a priori} and possibly unbounded, can be described naturally in a universal PPL.
In programming language terms, this means the number of \codeinline{sample} \changed[fz]{statements} %calls
is unknown prior to execution.
%These models include probabilistic models with an unknown number of components (such as
%Bayesian nonparametric models \cite{https://doi.org/10.1111/1467-9868.00095},
%variable selection in regression \cite{article},
%and models for signal processing \cite{DBLP:conf/aistats/MurrayLKBS18});
%and models that are defined on infinite-dimensional spaces (such as probabilistic context free grammars \cite{Manning99},
%birth-death models of evolution \cite{DBLP:conf/uai/KudlickaMRS19} and statistical phylogenetics \cite{Ronquist2020.06.16.154443}).
\changed[fz]{On the one hand, such programs can describe probabilistic models with an unknown number of components,} such as
Bayesian nonparametric models \cite{https://doi.org/10.1111/1467-9868.00095},
variable selection in regression \cite{article},
and models for signal processing \cite{DBLP:conf/aistats/MurrayLKBS18}.
\changed[fz]{On the other hand, there are models} defined on infinite-dimensional spaces, such as probabilistic context free grammars \cite{Manning99},
birth-death models of evolution \cite{DBLP:conf/uai/KudlickaMRS19} and statistical phylogenetics \cite{Ronquist2020.06.16.154443}.
\iffalse
Prominent examples of practical universal PPLs include Anglican \cite{DBLP:conf/aistats/WoodMM14}, Venture \cite{DBLP:journals/corr/MansinghkaSP14}, Web PPL \cite{dippl}, Hakaru \cite{narayanan2020symbolic}, Pyro \cite{pyro}, Gen \cite{DBLP:conf/pldi/Cusumano-Towner19} and Turing \cite{DBLP:conf/aistats/GeXG18}.
\fz{Is the last sentence necessary? I think the LAFI PC should be familiar with these.}
\lo{Agreed}
\fi

However, since universal PPLs are expressively complete, it is challenging to design and implement inference engines that work for arbitrary programs written in them.
The parameter space of a nonparametric model is a disjoint union of spaces of varying dimensions.
To approximate the posterior distribution via a Markov chain Monte Carlo (MCMC) algorithm (say), the transition kernel will have to switch between (possibly an unbounded number of) states of different dimensions, and to do so reasonably efficiently.
This explains why \changed[cm]{providing theoretical guarantees for MCMC algorithms that work for universal PPLs \cite{DBLP:journals/jmlr/WingateSG11,DBLP:conf/aistats/WoodMM14,DBLP:conf/pkdd/TolpinMPW15,HurNRS15,DBLP:conf/icml/MakZO21}
%\lo{Should include a reference to SMC.}
is very challenging.}
\changed[fz]{For instance,} the original version of Lightweight MH \cite{DBLP:journals/jmlr/WingateSG11} \changed[fz]{was} incorrect \cite{Kiselyov16}.
In fact, most applications requiring Bayesian inference rely on custom MCMC kernels, which are error-prone and time-consuming to design and build.

\paragraph{Contributions}

We introduce \defn{Nonparametric Involutive MCMC} (NP-iMCMC) for designing MCMC samplers for universal PPLs.
It is an extension of the involutive MCMC (iMCMC) framework \cite{DBLP:conf/icml/NeklyudovWEV20,cusumanotowner2020automating} to densities arising from nonparametric models (for background on both, see \cref{sec: background}).
We explain how NP-iMCMC moves between dimensions and how a large class of existing iMCMC samplers can be extended for universal PPLs (\cref{sec: np-imcmc}).
We also discuss necessary assumptions and prove its correctness.
Furthermore, there are general transformations and combinations of NP-iMCMC, to derive more powerful samplers systematically (\cref{sec: combination}), \changed[fz]{for example by making them nonreversible to reduce mixing time}.
%As a concrete example, we explain how a nonparametric extension of Lookahead HMC \cite{DBLP:conf/icml/Sohl-DicksteinMD14} can be understood in our framework (\cref{sec: case studies}).
Finally, our experimental results show that our method yields significant performance improvements over existing general MCMC approaches (\cref{sec: experiments}).
%}

\emph{All missing proofs are presented in the appendix.}

\paragraph{Notation}

%We denote $\Gau(\entelem,\sigma^2)$ as the $\entelem$-mean $\sigma^2$-variance Gaussian distribution with density function $\pdfGau(\entelem,\sigma^2):\Real \to\pReal$;
%$\Gau_n(\enta,\Sigma)$ as the $\enta$-mean $\Sigma$-covariance $n$-dimensional Gaussian distribution with density function $\pdfGau_n(\enta,\Sigma):\Real^n \to\pReal$.
% For ease of reference, we write $\Gau$ for the standard Gaussian distribution $\Gau(0,1)$ and $\pdfGau$ for its density;
% similarly $\Gau_n$ for the $n$-dimensional standard Gaussian distribution $\Gau_n(\boldsymbol{0},\idmat)$ and $\pdfGau_n$ for its density. \lo{TODO: make it more concise.}
We write $\Gau_n(\enta,\Sigma)$ for the $\enta$-mean $\Sigma$-covariance $n$-dimensional Gaussian with pdf $\pdfGau_n(\enta,\Sigma)$.
%$\pdfGau_n(\enta,\Sigma):\Real^n \to\pReal$.
For the standard Gaussian $\Gau_n(\boldsymbol{0},\idmat)$, we abbreviate them to $\Gau_n$ and $\pdfGau_n$.
In case $n = 1$, we simply write $\Gau$ and $\pdfGau$.

\iffalse
A \defn{probability kernel} $\ikernel: X \kernelto Y$ is a map
$\ikernel : X \times \algebra{Y} \to \pReal$ where
for any $x \in X$, $\ikernel(x, \placeholder):\algebra{Y} \to \pReal$ is
a probability measure on $Y$ and
for any $U \in \algebra{Y}$, $\ikernel(\placeholder, U):X \to \pReal$ is a measurable function.
\fz{Do we really need to define this?}
We denote $\ikpdf(x,y)$ as the density of $y \in Y$ in the measure $\ikernel(x, \placeholder)$.
\cm{We say a probability kernel is \defn{probabilistic}.}
\fz{What's the point of that? We only defined probabilistic kernels, not general measure kernels.}
\lo{I propose we replace this paragraph by the following.}
\fz{I think that proposal is better.}
% For any probability distribution $\mathcal{D}$, we write $\pdf{\mathcal{D}}$ to be its probability density function and $\cdf{\mathcal{D}}$ to be its cumulative density function.
\fi

Given measurable spaces $(X, \Sigma_X)$ and $(Y, \Sigma_Y)$, we write $\ikernel: X \kernelto Y$ to mean a \defn{kernel} of type $\ikernel : X \times \algebra{Y} \to \pReal$.
We say that $\ikernel$ is a \defn{probability kernel} if for all $x \in X$, $\ikernel(x, \placeholder):\algebra{Y} \to \pReal$ is a probability measure.
We write $\ikpdf(x,y)$ as the density of $y \in Y$ in the measure $\ikernel(x, \placeholder)$ assuming a derivative w.r.t.~some reference measure exists.

Unless otherwise specified,
the real space $\Real$ is endowed with the Borel measurable sets $\Borel$ and the  standard Gaussian $\Gau$ measure;
the boolean space $\Bool := \set{\true,\false}$ is endowed with the discrete measurable sets $\algebra{\Bool} := \powerset(\Bool)$ and the measure $\measure{\Bool}$ which assigns either boolean the probability $0.5$.
We write $\range{\enta}{1}{n}$ to mean the $n$-long prefix of the sequence $\enta$.
For any real-valued function $f:X \to \Real$, we define its \defn{support} as $\support{f} := \set{x \in X \mid f(x) > 0}$.

\section{Background}
\label{sec: background}
% !TEX root = ./../icml2022.tex

\subsection{Involutive MCMC}
\label{sec: imcmc}
% iMCMC algorithm might be a candidate but is too flexible

\begin{figure}[h]
\fbox{
\noindent\begin{minipage}{\columnwidth}
Given a target density $\tpdf$ on a measure space $(X,\algebra{X},\measure{X})$, the iMCMC algorithm generates a Markov chain of samples $\{\enta^{(i)}\}_{i\in\Nat}$ by proposing the next sample $\enta$ using the current sample $\enta_0$, in three steps:
\begin{compactenum}[1.]
  \item $\auxa_0 \sim \ikernel(\enta_0, \placeholder)$:
  sample a value $\auxa_0$ on an auxiliary measure space $(Y,\algebra{Y},\measure{Y})$
  from an \emph{auxiliary kernel} $\ikernel:X \kernelto Y$
  applied to the current sample $\enta_0$.

  \item $(\enta,\auxa) \leftarrow \iinv(\enta_0,\auxa_0)$:
  compute the new state $(\enta,\auxa)$ by applying
  an \emph{involution}\footnote{i.e.~$\iinv = \inv{\iinv}$.} $\iinv : X \times Y \to X\times Y$
  to $(\enta_0,\auxa_0)$.
  % \changed[fz]{a function $\iinv : X \times Y \to X\times Y$ that is an \emph{involution}}
  % \footnote{An \defn{involution} is just a function $f$ that is its own inverse, i.e., $f(f(x)) = x$ for all $x \in \domain{f}$.}

  \item Accept the proposed sample $\enta$ as the next step with probability given by the \emph{acceptance ratio}
  \[ \min\bigg\{1; \; \frac{\tpdf(\enta)\cdot\ikpdf(\enta, \auxa)}{\tpdf(\enta_0)\cdot\ikpdf(\enta_0, \auxa_0)}\cdot\abs{\det(\grad{\iinv(\enta_0, \auxa_0)})}\bigg\}; \]
  otherwise
  %$\enta_0$ is repeated. %(i.e.~the next sample is still $\enta_0$).
  reject the proposal $\enta$ and repeat $\enta_0$.
\end{compactenum}
\end{minipage}
}
\caption{iMCMC Algorithm \label{fig:imcmc algo}}
\end{figure}

% no space and not that importnat
\iffalse
\begin{figure}[h]
\fbox{
\noindent\begin{minipage}{\columnwidth}
Given a target density $\tpdf$ on a measure space $(\Real^n,\Borel^n,\Gau_n)$,
the Metropolis-Hastings algorithm generates a Markov chain of samples $\{\enta^{(i)}\}_{i\in\Nat}$
by proposing the next sample $\enta$ using the current sample $\enta_0$, in three steps:
\begin{compactenum}[1.]
  \item $\auxa_0 \sim \Gau_n(\enta_0,\idmat)$: sample a value $\auxa_0$ from
  the proposal distribution $\Gau(\enta_0,\idmat)$ conditioned on current sample $\enta_0$;

  \item $(\enta,\auxa) \leftarrow (\auxa_0, \enta_0)$: compute the new state $(\enta,\auxa)$

  \item Accept the proposed sample $\enta$ as the next step with probability given by the \emph{acceptance ratio}
  \[ \min\bigg\{1; \;
  \frac
    {\tpdf(\enta)}
    {\tpdf(\enta_0)}
  \bigg\}; \]
  otherwise
  %$\enta_0$ is repeated. %(i.e.~the next sample is still $\enta_0$).
  reject the \changed[lo]{proposal $\enta$} and repeat $\enta_0$.
\end{compactenum}
\end{minipage}
}
\caption{Metropolis-Hastings Algorithm \label{fig:mh algo}}
\end{figure}
\fi

\iffalse
\lo{Very minor, but why do you use ``;'' rather than the standard ``,'' in $\min\set{1 ; \Xi}$?}
\fi
Our sampler is built on the recently introduced \defn{involutive Markov chain Monte Carlo} method
\cite{DBLP:conf/icml/NeklyudovWEV20,cusumanotowner2020automating},
a unifying framework for MCMC algorithms.
Completely specified by a target density \changed[fz]{$\tpdf$}, an (auxiliary) kernel \changed[fz]{$\ikernel$} and an involution \changed[fz]{$\iinv$},
the iMCMC algorithm (\cref{fig:imcmc algo}) is conceptually simple.
Yet it is remarkably expressive, describing many existing MCMC samplers,
including Metropolis-Hastings (MH) \cite{metropolis1953equation,hastings1970monte}
% as shown in \cref{fig:mh algo}
with
the ``swap'' involution $\invo(\enta,\auxa) := (\auxa,\enta)$ and the proposal distribution as its auxiliary kernel $\ikernel$;
as well as Gibbs \cite{geman1984stochastic}, Hamiltonian Monte Carlo (HMC) \cite{Neal2011} and
Reversible Jump MCMC (RJMCMC) \cite{Green95}.
Thanks to its schematic nature and generality, we find iMCMC an ideal basis for constructing our nonparametric sampler, NP-iMCMC, for (arbitrary) probabilistic programs.
\changed[lo]{\emph{We stress that NP-iMCMC is applicable to any target density function that is tree representable}}.

% \citet{DBLP:conf/icml/NeklyudovWEV20} have shown that the transition kernel from iMCMC satisfies detailed balance and hence preserves the distribution described by the target density $\tpdf$.

% Many existing MCMC algorithms, such as Metropolis-Hastings (MH) \cite{metropolis1953equation,hastings1970monte}, Gibbs \cite{geman1984stochastic}, Hamiltonian Monte Carlo (HMC) \cite{Neal2011} and Reversible Jump MCMC (RJMCMC) \cite{Green95}, are instances of iMCMC algorithms.
% To see how MH is an instance of iMCMC, consider the swap involution $\invo(\enta,\auxa) = (\auxa,\enta)$ and the proposal distribution as its auxiliary kernel $\ikernel$.
% In this case, the acceptance ratio is the so-called Hastings' ratio.

% Although the iMCMC algorithms can accommodate many probabilistic models expressible in a universal PPL,
% designing trans-dimensional involutions for the iMCMC framework is still very challenging.

\subsection{Tree representable functions}

\begin{figure}[t]
\begin{code}[
  numbers=none,
  xleftmargin=.025\textwidth, xrightmargin=.025\textwidth,
  caption={Infinite Gaussian mixture model},
  label={code: iGMM}
]
  K = floor(abs(sample(normal(0, 1))))
  for i in range(K):
    xs[i] = sample(normal(0, 1))
  for d in data:
    observe d from mixture([normal(x, 1) for x in xs])
  return K
\end{code}
\end{figure}

As is standard in probabilistic programming, our sampler finds the posterior of a program $\terma$ by taking as the target density a map $\tree$, which, given an execution trace,
runs $\terma$ on the sampled values specified by the trace, and returns the weight of such a run.
Hence the support of $\tree$ is the set of traces on which $\terma$ terminates.

This density $w$ must satisfy the \defn{prefix property} \cite{DBLP:conf/icml/MakZO21}: for every trace, there is at most one prefix with strictly positive density.
Such functions are called \defn{tree representable} as they can be presented as a computation tree.
We shall see how our sampler exploits this property to jump across dimensions in \cref{sec: np-imcmc}.

% The density\footnote{We mean density in the sense of \cite{DBLP:conf/aistats/ZhouGKRYW19,DBLP:conf/icml/ZhouYTR20,cusumanotowner2020automating}, also referred to as the weight function \cite{DBLP:conf/esop/CulpepperC17,VakarKS19,MakOPW21}.} function $\tree$ of a probabilistic program that most MCMC algorithms take as input is defined to be a map that returns the weight of a given \defn{trace},
% the sequence of all sampled values in a particular run of a probabilistic program.
%a record of values drawn in the course of an execution of a probabilistic program, one for each random primitive encountered.

% For simplicity, say the only random primitive in the universal PPL is the standard normal distribution $\Gau$ on $\Real$.
Formally the \defn{trace space} $\traces$ is the disjoint union %of the real space
$\bigcup_{n\in\Nat} \Real^n$,
endowed with $\sigma$-algebra $\algebra{\traces} := \set{\bigcup_{n\in\Nat} \entseta_n \mid \entseta_n \in \Borel_n}$ and
the standard Gaussian (of varying dimensions) as measure
$\measure{\traces}(\bigcup_{n\in\Nat} \entseta_n) := \sum_{n\in\Nat} \Gau_n(\entseta_n)$.
We present traces as lists, e.g.~$[-0.2, 3.1, 2.8]$ and $\emptytrace$.
Thus the prefix property is expressible as:
for all traces $\traceb \in \traces$,
there is at most one $k \leq \len{\traceb}$ s.t.~\changed[fz]{the prefix} $\range{\traceb}{1}{k}$ is in $\support{\tree}$.
\changed[fz]{
Note that the prefix property is satisfied by any densities $\tree:\traces \to \pReal$ induced by a probabilistic program }
\changed[cm]{
(\cref{prop: all spcf terms have TR weight function}),}
\changed[fz]{
so this is a mild restriction}

\begin{example}
\label{eg: infinite gmm}
Consider the classic nonparametric infinite Gaussian mixture model (GMM),
%where given a set $\data$ of data infers the number of Gaussian mixtures.
which infers the number of Gaussian components from a data set.
It is describable as a program (\cref{code: iGMM}),
where there is a mixture of \changed[lo]{{\codeinline{K}}} Gaussian distributions such that
the \codeinline{i}-th Gaussian has mean \codeinline{xs[i]} and unit variance.
%The random variables \codeinline{K} and \codeinline{xs[i]} are normally distributed.
As \codeinline{K} is not pre-determined,
%the mixture has an unbounded number of Gaussian mixtures,
the possible number of components is unbounded, rendering the model nonparametric.
Given a trace $[3.4, -1.2, 1.0, 0.5]$,
the program describes a mixture of three Gaussians centred at $-1.2, 1.0$ and $0.5$;
and it computes the likelihood of generating the set $\data$ of data from such a mixture.
The program has density $\tree: \traces \to \pReal$ (w.r.t.~the trace measure $\measure{\traces}$) with $\tree (\traceb)$ defined as:
%\lo{In following, note the added subscript of the density $\pdfGau$.}
%\small
\begin{align*}
% \small
  %\tree: \traceb \mapsto
  \begin{cases}
    \displaystyle
    \prod_{d \in \data} \sum_{i = 1}^{\floor{\abs{\seqindex{\traceb}{1}}}}
    \frac{1}{\floor{\abs{\seqindex{\traceb}{1}}}}
    \pdfGau_{\floor{\abs{\seqindex{\traceb}{1}}}}(d \mid \seqindex{\traceb}{1+i}, 1)
    & \text{if } \len{\traceb}-1 = \floor{\abs{\seqindex{\traceb}{1}}} \\
    0 & \text{otherwise.}
  \end{cases}
\end{align*}
We can check that the density $w$ is tree representable.
%It is easy to check that any trace has at most one prefix with positive density, thereby satisfying the prefix property and making $\tree$ tree representable.
\end{example}

% Infinite GMM is an example of a \emph{nonparametric} probabilistic model, meaning that the number of samples are not determined \emph{a priori}. In programming terms, this means the number of \codeinline{sample} calls is unknown.
% Hence to sample from such a model, the inference algorithm will need to ``jump'' across states of different dimensions, giving them the name \emph{trans-dimensional}. \lo{TODO: simplify / clarify.}

\iffalse
\begin{code}[caption={Gaussian mixture model},label={code: iGMM}]
  let K = sample(normal(3, 1)) in
  let rec mixture i G = if i < K then let μ = sample(normal(0, 1)) in
                                          mixture i (G + normal(μ, 1))
                                 else G
  in observe data from mixture 1 0
\end{code}
\fi

% tree representable functions and involutive MCMC

\section{Nonparametric involutive MCMC}
\label{sec: np-imcmc}
% !TEX root = ./../icml2022.tex

\subsection{Example: infinite GMM mixture}
\label{sec:eg np-mh}
Consider how a sample for the infinite GMM (\cref{eg: infinite gmm}) can be generated using a nonparametric variant of Metropolis-Hastings (MH), an instance of iMCMC.
Suppose the current sample is $\enta_0 := [3.4, -1.2, 1.0, 0.5]$; and $[4.3, -3.4, -0.1, 1.4]$---a sample from the stock Gaussian
% $\Gau_4(\enta_0, \idmat)$
$\Gau_4$---is the value of the initial auxiliary variable $\auxa_0$.
Then, by application of the ``swap'' involution to $(\enta_0, \auxa_0)$,
% \changed[lo]{(see~\cref{fig:mh algo})}
the proposed state $(\enta, \auxa)$ becomes $([4.3, -3.4, -0.1, 1.4], [3.4, -1.2, 1.0, 0.5])$.
A problem arises if we simply propose $\enta$ as the next sample, as it describes a mixture of \emph{four} Gaussians (notice \codeinline{K} has value $4$) but only \emph{three} means are provided, viz., $-3.4, -0.1, 1.4$.
Hence, the program does not terminate on the trace specified by $\enta$, i.e., $\enta$ is not in the support of $w$, the model's density.

The key idea of NP-iMCMC is to \emph{extend} the initial state $(\enta_0, \auxa_0)$ to
$(\enta_0\concat [\entelem], \auxa_0\concat [\auxelem])$
where \changed[fz]{$\concat$ denotes trace concatenation, and} $\entelem, \auxelem$ are random draws from the stock Gaussian $\Gau$.
% respective stock distributions.
Say $-0.7$ and $-0.3$ are the values drawn; the initial state then becomes
$(\enta_0, \auxa_0) = ([3.4, -1.2, 1.0, 0.5, -0.7], [4.3, -3.4, -0.1, 1.4, -0.3])$,
and the proposed state $(\enta, \auxa)$ becomes
$([4.3, -3.4, -0.1, 1.4, -0.3], [3.4, -1.2, 1.0, 0.5, -0.7])$.
Now the program does terminate on a trace specified by the proposed sample $\enta = [4.3, -3.4, -0.1, 1.4, -0.3]$; equivalently $\enta \in \support{w}$.

Notice that if this is not the case, such a process---which extends the initial state by incrementing the dimension---can be repeated until termination happens.
For an \changed[cm]{almost surely (a.s.)} terminating program, this process a.s.~yields a proposed sample. %that is in $\support{w}$.
%\lo{Cite the relevant result here.}

Finally, we calculate the acceptance ratio for $\enta \in \support{\tree}$ from the initial sample $\range{\enta_0}{1}{4} \in \support{\tree}$ as
% \changed[lo]{
% \begin{align*}
%   \min\bigg\{1; \;
%   &\frac
%     {\tree{(\enta)}\cdot\pdfGau^{\floor{\len{\enta^1}}}(\enta, \idmat)(\auxa)}
%     {\tree{(\range{\enta_0}{1}{4})}\cdot\pdfGau^3(\range{\enta_0}{1}{4}, \idmat)(\range{\auxa_0}{1}{4})}\bigg\}.
% \end{align*}}
\changed[cm]{
\begin{align*}
  \min\bigg\{1; \;
  &\frac
    {\tree{(\enta)}\cdot\pdfGau_{5}(\enta)\cdot\pdfGau_{5}(\auxa)}
    {\tree{(\range{\enta_0}{1}{4})}\cdot\pdfGau_{5}(\enta_0)\cdot\pdfGau_{5}(\auxa_0)}\bigg\}.
\end{align*}}
\iffalse
\lo{$\extree$ not defined yet (Should unhide the definition).}
\cm{
  I am hoping not to mention
  the extension of the density $\extree$ and extension of auxiliary kernels $\exauxkernel$ in the main text.
}
\lo{OK. So far $\extree$ is not used yet.}
\fi

%Next we generalise this transformation to other iMCMC algorithms, thereby extending them to nonparametric models expressible by probabilistic programs.

\subsection{State space, target density and assumptions}

\label{sec: assumptions}

\changed[cm]{Fix an \emph{parameter (measure) space} $(\parspace, \algebra{\parspace}, \measure{\parspace})$}, which is (intuitively) the product
% disjoint union
of the respective measure space of the distribution of ${\tt X}$, with ${\tt X}$ ranging over the random variables of the model in question.
Assume an \emph{auxiliary (probability) space} $(\auxspace, \algebra{\auxspace}, \measure{\auxspace})$.
For simplicity, we assume in this paper\footnote{In \cref{sec: spaces in np-imcmc}, we consider a more general case where $\entspace$ is set to be $\Real \times \Bool$.} that both $\entspace$
and $\auxspace$ are $\Real$;
further $\measure{\parspace}$ has a derivative $\parstockpdf$ w.r.t.~the Lebesgue measure, and $\measure{\auxspace}$ also has a derivative $\auxstockpdf$ w.r.t.~the Lebesgue measure.
%We further assume that $\entspace^n \times \auxspace^n$ is a smooth manifold for each $n$.
Note that it follows from our assumption that $\entspace^n \times \auxspace^n$ is a smooth manifold for each $n$.\footnote{Notation:
For any probability space $({\tt X}, \algebra{{\tt X}}, \measure{{\tt X}})$ such that
$\measure{{\tt X}}$ has derivative $\varphi_{\tt X}$ w.r.t.~the Lebesgue measure.
${\tt X}^n$ is the Cartesian product of $n$ copies of ${\tt X}$;
$\algebra{{\tt X}^n}$ is the $\sigma$-algebra generated by subsets of the form $\prod_{i=1}^n \auxseta_i$ where $\auxseta_i \in \algebra{{\tt X}}$; and
$\measure{{\tt X}^n}$ is the product of $n$ copies of $\measure{{\tt X}}$
which has derivative $\varphi_{{\tt X}^n}$ w.r.t.~the Lebesgue measure.
Note that $({\tt X}^n, \algebra{{\tt X}^n}, \measure{{\tt X}^n})$ is a probability space.}
Now a \defn{state} is a pair of \changed[lo]{\emph{parameter}} and \emph{auxiliary} variables of equal dimension.
Formally the \defn{state space} $\states := \bigcup_{n\in\Nat} (\entspace^n \times \auxspace^n)$
is endowed with the $\sigma$-algebra
$\algebra{\states} := \sigma\set{\entseta_n \times \auxseta_n \mid \entseta_n \in \algebra{\entspace^n}, \auxseta_n \in \algebra{\auxspace^n}, n\in\Nat}$ and
measure
%$\measure{\states}(\stateset) := \sum_{i=1}^n \expandint{\auxspace^n}{\measure{\entspace^n}(\set{\enta\in\entspace^n\mid (\enta,\auxa)\in\stateset})}{\measure{\auxspace^n}}{\auxa}$.
\changed[lo]{$\measure{\states}(\stateset) := \sum_{n \in \Nat} \expandint{\auxspace^n}{\measure{\entspace^n}(\set{\enta\in\entspace^n\mid (\enta,\auxa)\in\stateset})}{\measure{\auxspace^n}}{\auxa}$.}
\lo{N.B. Same typo in thesis.}

Besides the target density function $\tree$,
our algorithm NP-iMCMC requires two additional inputs: \defn{auxiliary kernels} (as an additional source of randomness) and \defn{involutions} (to traverse the \emph{state space}).
Next we present what we assume about the three inputs and discuss some relevant properties.

\paragraph{Target density function}
We only target densities $\tree:\traces \to \pReal$ that are tree representable, where
\changed[cm]{$\traces := \bigcup_{n\in\Nat} \Real^n$.}
Moreover, we assume %\changed[fz]{the following two properties, which are satisfied by any real-world probabilistic program}:
\changed[lo]{two common features of real-world probabilistic programs}:
\begin{compactitem}
  \item[\vass\label{vass: integrable tree}]
  $\tree$ is \defn{integrable},
  i.e.~$Z := \shortint{\traces}{\tree}{\measure{\traces}} < \infty$
  (otherwise, the inference problem is undefined)

  \item[\vass\label{vass: almost surely terminating tree}]
  $\tree$ is \defn{almost surely terminating (AST)},
  i.e.~$\measure{\traces}(\set{\traceb\in \traces \mid \tree(\traceb) > 0}) = 1$
  (otherwise, the loop (\cref{np-imcmc step: extend})
  of the NP-iMCMC algorithm may not
  terminate a.s.).%
  \footnote{\changed[fz]{If a program does not terminate on a trace $\traceb$, the density $\tree(\traceb)$ is defined to be zero.}}
\end{compactitem}
\iffalse
We define an \defn{extension} $\extree:\traces \to \pReal$ of $\tree$ as $\extree:\traceb \mapsto \max\{ \tree(\range{\traceb}{1}{i}) \mid i=1, \dots,\len{\traceb}\}.$
The prefix property implies that for any $\traceb \in \traces$,
\[
  \exists k \leq \len{\traceb}\ .\ \range{\traceb}{1}{k} \in \support{\tree}
  \iff
  \traceb \in \support{\extree}
\]
\fi
%\lo{@Carol: The preceding definition of $\extree$ was commented out.}
\paragraph{Auxiliary kernel}
%Let $(\auxspace, \algebra{\auxspace}, \measure{\auxspace})$ be a probability space.
%\footnote{A probability measure space is a measure space with a probability measure.}.
 % with a $\sigma$-algebra $\algebra{\auxspace}$ a probability measure $\measure{\auxspace}$.
We assume, for each dimension $n\in\Nat$, an \defn{auxiliary (probability) kernel}
$\nauxkernel{n}: \entspace^n \kernelto \auxspace^n$
%that is probabilistic,
%\changed[lo]{which is a probability kernel},
% i.e.~$\nauxkernel{n}(\enta, \auxspace^n) = 1$ for all $\enta \in \entspace^n$,
with density function
$\nauxkernelpdf{n}:\entspace^n \times \auxspace^n \to \pReal$
(assuming a derivative w.r.t.~$\measure{\entspace^n \times \auxspace^n}$ exists).

\paragraph{Involution}

We assume, for each dimension $n \in\Nat$, a differentiable endofunction $\ninvo{n}$ on $\entspace^n \times \auxspace^n$ which is \emph{involutive}, i.e.~$\ninvo{n} = \inv{\ninvo{n}}$, and satisfies the \defn{\invoass{}}:
\iffalse
\cm{need a good name} \lo{What about ``projection commutation'' (or longer but more informative, ``projection-involution commutation'')?}:
\fi
\begin{compactitem}
  \item[\vass\label{vass: partial block diagonal inv}]
  For all $(\enta, \auxa) \in \states$ where $\len{\enta} = m$, if $\range{\enta}{1}{n} \in \support{\tree}$ for some $n$, then for all $k = n, \dots, m$, $\proj{m}{k}(\ninvo{m}(\enta, \auxa)) = \ninvo{k}(\proj{m}{k}(\enta, \auxa)) $
\end{compactitem}
\changed[cm]{where $\proj{\ell_1}{\ell}$ is the projection that takes a state $(\enta,\auxa)$
% with components of length $\ell_1$
and returns the state $(\range{\enta}{1}{\ell},\range{\auxa}{1}{\ell})$ with the first $\ell$ coordinates of each component.}
(Otherwise, the sample-component $\enta$ of
the proposal state tested in \cref{np-imcmc step: extend}
may not be an extension of the sample-component of the preceding proposal state.)

%\cm{Give a simple check for this property.}

\subsection{Algorithm}

Given a probabilistic program $\terma$ with density function $\tree$ on the trace space $\traces$,
a set $\set{\nauxkernel{n}:\entspace^n \kernelto\auxspace^n}$ of auxiliary kernels and
a set $\set{\ninvo{n}:\entspace^n \times \auxspace^n \to \entspace^n \times \auxspace^n}$ of involutions satisfying \cref{vass: partial block diagonal inv,vass: integrable tree,vass: almost surely terminating tree},
we present the %\defn{nonparametric involutive Markov chain Monte Carlo (NP-iMCMC)}
\defn{NP-iMCMC} algorithm in \cref{fig:np-imcmc algo}.
\begin{figure}[h]
\fbox{
\noindent\begin{minipage}{\columnwidth}
% Given a program $\terma$ (in a universal PPL) with density function
% $\tree$ on the trace space,
The \defn{NP-iMCMC} generates a Markov chain by proposing the next sample $\enta$ using the current sample $\enta_0$ as follows:
\begin{compactenum}[1.]
  \item{\step\label{np-imcmc step: aux sample}}
  $\auxa_0 \sim \nauxkernel{k_0}(\enta_0, \placeholder)$:
  sample a value $\auxa_0$ on the auxiliary space $\auxspace^{k_0}$
  from the auxiliary kernel $\nauxkernel{k_0}:\entspace^{k_0} \kernelto \auxspace^{k_0}$
  applied to the current sample $\enta_0$ where $k_0 = \len{\enta_0}$.

  \item{\step\label{np-imcmc step: involution}}
  $(\enta,\auxa) \leftarrow \ninvo{n}(\enta_0,\auxa_0)$:
  compute the \emph{proposal state} $(\enta,\auxa)$ by
  applying the involution $\ninvo{n}$ on $\entspace^n\times \auxspace^n$
  to the \emph{initial state} $(\enta_0,\auxa_0)$ where $n = \len{\enta_0}$.

  \item{\step\label{np-imcmc step: extend}}
  Test if for some $k$, $\range{\enta}{1}{k} \in \support{\tree}$. (\emph{Equivalently}: %\footnote{Equivalently, test if the sample-component $\enta \in \support{w}$.}
  Test if program $\terma$ terminates on the trace specified by the sample-component $\enta$ of the proposal state, or one of its prefixes.)
  % (This can be done by checking if $\enta$ is in the support of the extension $\extree$.)
  If so, proceed to the next step;
  otherwise
  \begin{compactitem}
    \item $(\enta_0, \auxa_0) \leftarrow (\enta_0 \concat [\entelem],\auxa_0 \concat [\auxelem])$:
    extend the initial state to $(\enta_0 \concat [\entelem],\auxa_0 \concat [\auxelem])$ where
    $\entelem$ and $\auxelem$ are samples drawn from
    $\measure{\entspace}$ and $\measure{\auxspace}$,
    \item Go to \cref{np-imcmc step: involution}. %\lo{I expect \LaTeX\ to generate ``Step 2'' (after the original version of the algorithm is removed).}
  \end{compactitem}

  \item{\step\label{np-imcmc step: accept/reject}}
  Accept $\range{\enta}{1}{k}$ %\in \support{\tree}$
  as the next sample with probability
  \begin{align*}
  \min\bigg\{1; \;
  &\frac
    {\tree{(\range{\enta}{1}{k})}\cdot\nauxkernelpdf{k}(\range{\enta}{1}{k}, \range{\auxa}{1}{k})}
    {\tree{(\range{\enta_0}{1}{k_0})}\cdot\nauxkernelpdf{k_0}(\range{\enta_0}{1}{k_0}, \range{\auxa_0}{1}{k_0}) } \\
  &\cdot
  % \changed[cm]{
    \frac
    {\nstockpdf{n}(\enta, \auxa)}
    {\nstockpdf{n}(\enta_0, \auxa_0)}
    \cdot
    % }
  \abs{\det(\grad{\ninvo{n}(\enta_0, \auxa_0)})}\bigg\}
  \end{align*}
  where $n = \len{\enta_0}$;
  otherwise reject the proposal and repeat $\range{\enta_0}{1}{k_0}$. %\in \support{\tree}$.
\end{compactenum}
\end{minipage}
}
\caption{NP-iMCMC Algorithm \label{fig:np-imcmc algo}}
\end{figure}

\iffalse
\changed[lo]{Let's} see how a state moves between dimensions.
Suppose the current sample $\enta_0$ has dimension $k_0$.
\cref{np-imcmc step: aux sample} of NP-iMCMC forms an initial state $(\enta_0, \auxa_0)$ by sampling a value of dimension $k_0$ for the auxiliary variable $\auxa_0$.
Then a proposal state $(\enta, \auxa)$ is computed in \cref{np-imcmc step: involution} by applying the $k_0$-dimensional involution $\ninvo{k_0}$ to $(\enta_0, \auxa_0)$.
So far, we have not moved away from dimension $k_0$.
\fi

%The goal is to move ``lazily'' across dimensions in order to generate a proposal state $(\enta, \auxa)$ where the program $\terma$ terminates on the trace specified by its sample-component $\enta$ (i.e.~$\enta \in \support{w}$).
The heart of NP-iMCMC is \cref{np-imcmc step: extend}, which can drive a state across dimensions.
\cref{np-imcmc step: extend} first checks
if $\range{\enta}{1}{k} \in \support{w}$ for some $k = 1,\dots, k_0$,
(i.e.~if the program $\terma$ terminates on the trace specified by \changed[lo]{some} prefix of $\enta$).
If so, the proposal state is set to $(\range{\enta}{1}{k}, \range{\auxa}{1}{k})$, and the state moves from dimension $k_0$ to $k$.
Otherwise, %(i.e.~the program does not terminate on any traces specified by any prefixes of $\enta$),
\cref{np-imcmc step: extend} \changed[fz]{repeatedly} % replaced ``lazily'' because it confused a reviewer
extends the initial state $(\enta_0, \auxa_0)$ to, say, $(\enta_0\concat\entb_0, \auxa_0\concat\auxb_0)$,
and computes the new proposal state $(\enta\concat\entb, \auxa\concat\auxb)$ by \cref{np-imcmc step: involution},
until the program $\terma$ terminates on the trace specified by $\enta\concat\entb$.
Then, the proposal state becomes $(\enta\concat\entb, \auxa\concat\auxb)$, and the state moves from dimension $k_0$ to dimension $k_0 + \len{\entb}$.

\begin{remark}
\label{rem: np-imcmc ast}
\begin{compactenum}[(i)]
\item The \invoass{}, \cref{vass: partial block diagonal inv}, ensures that the new proposal state computed using $(\enta_0\concat\entb_0, \auxa_0\concat\auxb_0)$ from \cref{np-imcmc step: extend} is of the form $(\enta\concat\entb, \auxa\concat\auxb)$ where $(\enta, \auxa) = \ninvo{\len{\enta_0}}(\enta_0, \auxa_0)$.

\item \cref{vass: almost surely terminating tree}, a.s.~termination of the program $\terma$, ensures that the method of computing a proposal state in \cref{np-imcmc step: extend} almost surely finds a proposal sample $\enta$ such that $M$ terminates on a trace specified by $\enta$.
\iffalse
It is important to note that this algorithm relies on the fact that in a universal PPL,
once a program terminates on some execution trace (of drawn samples), the density will not be affected by extending the execution trace.
\lo{I don't follow this: if (in Step 3) the program terminates on some trace, the algorithm would proceed to Step 4, which means that the trace will \emph{not} be extended.}
Hence the algorithm can safely extend the initial state \emph{without} affecting its weight in Step \ref{np-imcmc step: extend}. \lo{This point is amplified / clarified in \cref{rem:extend not affecting w}.}
\fi

\item \label{rem:extend not affecting w}
\changed[cm]{
The prefix property of the target density $\tree$ ensures that any proper extension of  current sample $\enta_0$ (of length $k_0$) has zero density, i.e.~$\tree(\enta_0 \concat \entb) = 0$ for all $\entb \not= \emptytrace$.
Hence only the weight of the current sample $\range{\enta_0}{1}{k_0} \in \support{\tree}$
is accounted for in \cref{np-imcmc step: accept/reject}
even when $\enta_0$ is extended.}

\lo{I agree that $\tree(\enta_0 \concat \entb) = 0$ for all $\entb \not= \emptytrace$, but I still don't see the point of saying ``Hence there is no weight to be accounted for when the initial state is extended in \cref{np-imcmc step: extend}''. Why do we need to account for the weight when the initial state is extended in \cref{np-imcmc step: extend}, which is an intermediate stage of the algorithm? What is relevant is that the values of $\tree{(\range{\enta}{1}{k})}$ and $\tree{(\range{\enta_0}{1}{k_0})}$ (in the quotient of \cref{np-imcmc step: accept/reject}) are both non-zero.}
  % appending randomly sampled values to the current sample does \emph{not} affect the weight of the sample.
  % The NP-iMCMC algorithm heavily relies on this property in \cref{np-imcmc step: extend}.
  % Indeed, if the input density does not satisfy the prefix property, the extension of the initial state made in \cref{np-imcmc step: extend} will alter the weight, which would need to be accounted for.
  % Hence NP-iMCMC algorithm is an inference algorithm \emph{for} universal PPLs as any program specified in such language satisfies this prefix property (\cref{prop: all spcf terms have TR weight function}).

\item If the program $\terma$ is parametric, thus inducing a target density $\tree$ on a fixed dimensional space, then the NP-iMCMC sampler coincides with the iMCMC sampler.
\end{compactenum}
\end{remark}

\iffalse
\cm{Finally in \cref{np-imcmc step: accept/reject},
the prefix $\range{\enta}{1}{k}$ of the sample-component of the
proposal state $(\enta, \auxa)$,
where $\terma$ terminates on a trace specified by $\range{\enta}{1}{k}$,
is accepted with the probability given by the
the ratio of
$\tree(\range{\enta}{1}{k}) \cdot \nauxkernelpdf{k}(\range{\enta}{1}{k}, \range{\auxa}{1}{k}) $
to
$\tree(\range{\enta_0}{1}{k_0}) \cdot \nauxkernelpdf{k_0}(\range{\enta_0}{1}{k_0}, \range{\auxa_0}{1}{k_0}) $
multiplied by the absolute value of the Jacobian determinant
of $\ninvo{\len{\enta_0}}$ at $(\enta_0, \auxa_0)$,
except when the product is larger than $1$,
in which case, $\range{\enta}{1}{k}$ is always accepted.} \lo{This seems to repeat Step 4 without adding any new information.}
% Now we look at the conditions on the density function $\tree$ so that NP-iMCMC preserves the posterior.
\fi

Using NP-iMCMC (\cref{fig:np-imcmc algo}),
we can formally present the
Nonparametric Metropolis-Hastings (NP-MH) sampler which was introduced in \cref{sec:eg np-mh}.
See \cref{app: np-mh} for details.
%\lo{For me, {\tt cref} fails to resolve the preceding reference {\tt app: np-mh}, and all subsequent references pointing into the appendix.}

\changed[cm]{
\subsection{Generalisations}
\label{sec: generalisations}
}

In the interest of clarity,
we have presented a version of NP-iMCMC in deliberately purified form.
Here we discuss \changed[cm]{three generalisations} of the NP-iMCMC sampler.
% An advantage of such a presentation is that we can more readily see how to modify or extend it.
% We assume all assumptions of the NP-iMCMC (\allvass{}) are satisfied, unless otherwise specified.

%\paragraph{Hybrid state space}

\mypara{Hybrid state space}
Many PPLs %such as SPCF (\cref{app: spcf})
provide continuous and discrete samplers.
The positions of discrete and continuous random variables in an execution trace may vary, because of branching.
We get around this problem by defining the \changed[cm]{parameter space} $\entspace$ to be the \emph{product space} of $\Real$ and $\Bool := \{ \false, \true \}$.
Each value $\seqindex{\traceb}{i}$ in a trace $\traceb$ is paired with a randomly drawn ``partner'' $\trace$ of the other type to make a pair $(\seqindex{\traceb}{i},\trace)$ (or $(\trace, \seqindex{\traceb}{i})$).
% For instance, the trace $[\true, -3.1]$ can be made into a \changed[lo]{(two-dimensional)} \changed[cm]{entropy trace $[(1.5,\true),(-3.1,\true)]$} with randomly drawn sample $1.5$ from the normal distribution and $\true$ from a coin toss.
Hence, the same idea of ``jumping'' across dimensions can be applied to the state space
$\bigcup_{n\in\Nat} \entspace^n \times \auxspace^{n}$.
% Importantly, each $\entspace$-valued sequence $\enta$ can only represent \emph{at most one} trace $\traceb$ that lies in the support of a tree representable function $\tree$.
% So the weight of a trace $\traceb$ will \emph{not} be affected by the additional samples we attached to it to make it an element %in the product space.
% \changed[lo]{of $\bigcup_{n\in\Nat} \entspace^n$}.
\changed[cm]{The resulting algorithm is called
the \emph{Hybrid NP-iMCMC} sampler.}
(See \cref{app: hybrid np-imcmc} for more details.)

\mypara{Computationally heavy involutions}
\cref{np-imcmc step: extend} in the NP-iMCMC sampler may seem inefficient. While it terminates almost surely (thanks to \cref{vass: almost surely terminating tree}),
the expected number of iterations may be infinite.
This is especially bad if the involution is \changed[lo]{computationally expensive}
such as the leapfrog integrator in HMC which requires gradient information of the target density function.
This can be worked around if
for each $n\in\Nat$,
there is an inexpensive
\changed[cm]{\emph{slice} function
$\slicefn{n}:\parspace^n \times\auxspace^n \to \parspace\times\auxspace $ where
$\slicefn{n}(\enta, \auxa)
= (\drop{n-1} \circ \ninvo{n})(\enta, \auxa)$
if $(\enta, \auxa)$ is a $n$-dimensional state such that
$\range{\enta}{1}{k}\in\support{\tree}$
for some $k < n$,
and
$\drop{\ell}$ is the projection that
takes a state $(\enta, \auxa)$ and returns
the state $(\range{\enta}{\ell+1}{\len{\enta}},\range{\auxa}{\ell+1}{\len{\enta}})$
with the first $\ell$ coordinates of each
component dropped.}
% endofunction
% $\chi$ on $\entspace \times \auxspace$
% where $\ninvo{n+1}(\enta \concat \entb, \auxa \concat \auxb) =
%   (\enta' \concat \entb', \auxa' \concat \auxb')$
% iff
% $\ninvo{n}(\enta, \auxa) = (\enta', \auxa')$ and $\chi(\entb, \auxb) = (\entb', \auxb')$.
% \iffalse
% \begin{align*}
%   \ninvo{n+1}(\enta \concat \entb, \auxa \concat \auxb) =
%   (\enta' \concat \entb', \auxa' \concat \auxb') \\
%   \iff
%   \ninvo{n}(\enta, \auxa) = (\enta', \auxa')
%   \text{ and }
%   \chi(\entb, \auxb) = (\entb', \auxb').
% \end{align*}
% \fi
Then the new proposal state in \cref{np-imcmc step: extend} can be computed by
applying the function $\slicefn{n}$ to
the recently extended initial state
$(\enta_0, \auxa_0)$, i.e.
$(\enta, \auxa) \leftarrow (\enta \concat [\entelem'], \auxa \concat [\auxelem']) \text{ where }
(\entelem', \auxelem') = \slicefn{n}(\enta_0, \auxa_0)$
instead.
\changed[cm]{(See \cref{sec: slice function} for more details.)}

\changed[cm]{
\mypara{Multiple step NP-iMCMC}
Suppose the involution is a composition of bijective endofunctions,
i.e.~$\ninvo{n} :=
\nelem{n}{L} \circ
\dots \circ \nelem{n}{2} \circ \nelem{n}{1}$
and each endofunction
$\set{\nelem{n}{\ell}}_n$ satisfies the
\invoass{}
and has a slice function $\slicefn{n}_{\ell}$.
A new state can then be computed by applying
the endofunctions
to the initial state \emph{one-by-one}
(instead of in one go as in \cref{np-imcmc step: involution,np-imcmc step: extend}):
For each $\ell=1,\dots,L$,
\begin{compactenum}[1.]
  \item Compute the intermediate state
  $(\enta_\ell, \auxa_\ell)$
  by applying $\nelem{n}{\ell}$
  to $(\enta_{\ell-1}, \auxa_{\ell-1})$
  where $n = \len{\enta_{\ell-1}}$.

  \item Test whether
  $\range{\enta_\ell}{1}{k}$ is in $\support{\tree}$ for some $k$.
  If so, proceed to the next $\ell$;
  otherwise
  \begin{compactitem}
    \item extend the initial state
    $(\enta_0, \auxa_0)$
    with samples drawn from
    $\measure{\entspace}$ and $\measure{\auxspace}$,

    \item
    for $i = 1,\dots,\ell$,
    extend the intermediate states
    $(\enta_i, \auxa_i)$
    with the result of
    $\slicefn{n}_{i}(\enta_{i-1}, \auxa_{i-1})$
    where $n = \len{\enta_{i-1}}$,

    \item go to 2.
  \end{compactitem}
\end{compactenum}
The resulting algorithm is called
the \emph{Multiple Step NP-iMCMC} sampler.
(See \cref{sec: multiple step npimcmc} for more details.)}
This approach was adopted in the recently proposed Nonparametric HMC \cite{DBLP:conf/icml/MakZO21}.
\changed[cm]{(See \cref{app: np-hmc} for more details.)}

\fz{Should we mention \cite{DBLP:conf/icml/MakZO21} more prominently before? The idea of extending the trace is common to both papers.}
\lo{I think this is a natural place to introduce NP-HMC. We could mention under related work the point that trace extension first appeared in NP-HMC.}

% which are compositions of elementary functions
% $\set{\elemfn{n}_{\ell}}_{\ell = 1,\dots, L}$,
% i.e.~$\ninvo{n} = \elemfn{n}_{L} \circ \dots \circ \elemfn{n}_1$.
% Assuming all assumptions are satisfied,
% we can modify the NP-iMCMC sampler t
% we can make its nonparametric variant more efficient by
% performing \cref{np-imcmc step: involution,np-imcmc step: extend} with the involution $\ninvo{n}$ replaced by the function $\elemfn{n}_{\ell}$ for $\ell = 1,\dots, k$.
% (See \cref{app: composite involution} for more details.)

% \lo{TODO. Either give some details or omit this.}

\subsection{Correctness}

%\cm{This subsection is new.}

The NP-iMCMC algorithm is correct in the sense that the invariant distribution of the Markov chain generated by iterating the algorithm in \cref{fig:np-imcmc algo} coincides with the target distribution
$\tdist: \traceset \mapsto \frac{1}{Z} \shortint{\traceset}{\tree}{\measure{\traces}} $
with the normalising constant $Z$.
We present an outline proof here. See \cref{app: proof of correctness} for a full proof
\changed[cm]{of the Hybrid NP-iMCMC algorithm, a generalisation
of NP-iMCMC.}

\anon{LAFI reviewer: Theorem 1 itself proved by reducing NP-iMCMC to iMCMC? That is, does the theorem in some sense construct an involution on the whole state space out of the provided involutions on the finite dimensional ones? This would be useful to clarify.}

Note that we \emph{cannot} reduce NP-iMCMC to iMCMC,
i.e.~the NP-iMCMC sampler cannot be formulated as an instance of the iMCMC sampler
with an involution on the whole state space $\states$.
% and make use of the correctness of the iMCMC sampler.
This is because the dimension of involution
depends on the values of the random samples drawn in \cref{np-imcmc step: extend}.
% in particular the values of the extension of the initial sample.
Instead,
we define a helper algorithm
\changed[cm]{(\cref{sec: e np-imcmc}),
which induces a Markov chain on states and}
% which induces the same Markov chain as NP-iMCMC,}
does not change the dimension of the involution.

This algorithm first extends the initial state to find the smallest $N$ such that
the program $\terma$ terminates with a trace
specified by some prefix of the sample-component of
the resulting state $(\enta, \auxa)$ after applying the involution $\ninvo{N}$.
Then, it performs the involution $\ninvo{N}$ as per the standard iMCMC sampler.
Hence all stochastic primitives are executed outside of the involution,
and the involution has a fixed dimension.
We identify the \changed[cm]{state distribution (\cref{sec: state distribution})}, and
\changed[cm]{
show that the Markov chain generated by
the auxiliary algorithm has the state distribution
as its invariant distribution (\cref{lemma: e-np-imcmc invariant}).
We then deduce
that its marginalised chain is identical to that
generated by Hybrid NP-iMCMC;
and Hybrid NP-iMCMC has the
target distribution $\tdist$ as its invariant distribution (\cref{lemma: marginalised distribution is the target distribution}).
Since Hybrid NP-iMCMC is a generalisation of NP-iMCMC (\cref{fig:np-imcmc algo}),
we have the following corollary.}

\begin{restatable}[Invariant]{corollary}{NPiMCMCinvariant}
  \label{corollary: invariant (np-imcmc)}
  If all inputs satisfy \allvass{} then
  $\tdist$ is the invariant distribution of the Markov chain generated by iterating the
  algorithm described in \cref{fig:np-imcmc algo}.
\end{restatable}

% algorithm and correctness

\section{Transforming NP-iMCMC samplers}
\label{sec: combination}
% !TEX root = ./../icml2022.tex

% The power of such a unifying framework is not only the ease with which \emph{new} nonparametric extensions of existing MCMC algorithms can be derived and proved correct,
% but also the capability of improving nonparametric algorithms by a suitable combination of algorithms.
% Formulating MCMC algorithms as instances of the unifying iMCMC framework makes it easier to understand why and how the samplers work.
% Moreover, \changed[fz]{insights} \changed[lo]{about the workings of samplers} can be generalised into ``tricks'' that can be applied to the general iMCMC framework, \changed[lo]{thus serving as tools to design new or refine existing samplers}.
% For instance, \citet{DBLP:conf/icml/NeklyudovWEV20} identified three main ``tricks'': \lo{The iMCMC paper actually presented 6 tricks. Are there three main ones? Did they say so?}
% \cm{
%   They grouped their tricks in three subsections, namely
%   smart auxiliary spaces, smart deterministic maps, and smart compositions.
%   Each subsection includes two tricks.
%   I think the two tricks under the smart auxiliary spaces is better described by state-dependent mixture.
%   }
The strength of the iMCMC framework lies in its flexibility,
which makes it a useful tool capable of expressing important ideas in existing MCMC algorithms as ``tricks'',
namely
\begin{compactitem}
  \item state-dependent mixture (Trick 1 and 2 in \cite{DBLP:conf/icml/NeklyudovWEV20}),
  % , which perform a different sampler depending on the current state;
  \item smart involutions (Trick 3 and 4), and
  % , which contains clever constructions of new involutions that better explore the parameter space.
  \item smart compositions (Trick 5 and 6).
  % , which compose iMCMC updates to make nonreversible chains which has better mixing time.
\end{compactitem}
% (See \cref{app: involutive MCMC} for details.)
% \lo{I think \cref{app: involutive MCMC} is unnecessary, since the results therein are already published.}
In each of these tricks, the auxiliary kernel and involution take special forms
to equip the resulting sampler with desirable properties such as
higher acceptance ratio and better mixing times.
This enables a ``make to order'' approach in the design of novel MCMC samplers.

A natural question is whether there are similar tricks for the NP-iMCMC framework.
In this section, we examine the tricks discussed in \cite{DBLP:conf/icml/NeklyudovWEV20},
giving requirements for and showing via examples how one can design \emph{novel} NP-iMCMC samplers with bespoke properties by suitable applications of these ``tricks'' to simple NP-iMCMC samplers.
% Giving multiple examples, we illustrate
% \begin{inparaenum}[(1)]
% \item how it is straightforward to develop nonparametric variants of existing iMCMC samplers that utilise these tricks; and
% \item how NP-iMCMC samplers can be extended easily.
% \end{inparaenum}
\changed[cm]{Similar applications can be made to the generalisations of NP-iMCMC
such as Hybrid NP-iMCMC (\cref{app: variants}) and Multiple NP-iMCMC (\cref{sec: tecniques on ms np-imcmc}).}
Throughout this section, we consider samplers for a program $\terma$ expressed in a universal PPL which has target density function $\tree$ that is integrable (\cref{vass: integrable tree}) and almost surely terminating (\cref{vass: almost surely terminating tree}).
\lo{We should also assume that $w$ satisfies (\cref{vass: partial block diagonal inv})? After all, the state-dependent mixture construction (\cref{sec:State-dependent mixture}) assumes it.}
\fz{I agree and then we can remove that condition from \cref{sec:State-dependent mixture}.}
\cm{\cref{vass: partial block diagonal inv} is not a condition on $w$, it is a condition on the involutions $\ninvo{n}$. Since the involution takes different forms in different tricks, I thought it is better to state it case-by-case. Let me know what you think.}
\fz{Good point. We can leave it as is.}

\subsection{State-dependent mixture}
\label{sec:State-dependent mixture}
Suppose we want a sampler that chooses a suitable NP-iMCMC sampler depending on the current sample.
This might be beneficial for models that are modular, and where there is already a good sampler for each module.
% Assuming that the family of involutions in each of these NP-iMCMC samplers satisfy \cref{vass: partial block diagonal inv},
We can form a \defn{state-dependent mixture} of a
family $\set{\imcmca_\mixa}_{\mixa\in M}$
of NP-iMCMC samplers\footnote{
We treat $\imcmca_\mixa$ as a piece of computer code that
changes the sample via the NP-iMCMC method described in \cref{fig:np-imcmc algo}.}
which
runs $\imcmca_\mixa$ with a weight depending on the current sample.
See \cref{sec: state-dependent mixture of np-imcmc} for details of the algorithm.

\begin{remark}
  This corresponds to Tricks 1 and 2 discussed in \cite{DBLP:conf/icml/NeklyudovWEV20}
  which generalises the
  Mixture Proposal MCMC
  and Sample-Adaptive MCMC samplers.
\end{remark}

% \begin{example}[NP-Mixed Proposal MCMC]
% \label{eg: NP-Mixed Proposal MCMC}
% The Mixture Proposal MCMC sampler \cite{habib2018auxiliary}
% can be formulated as a mixture of iMCMC samplers with
% the swap involution,
% proposal distribution conditioned on an intermediate variable $\mixa \in \mixspace$ and
% a mixture distribution on $\mixspace$ conditioned on the current sample.
% Given an appropriate mixture kernel,
% we can similarly form a nonparametric extension of Mixture Proposal MCMC sampler.
% (See \cref{app: examples} for details.)
% \end{example}
% \fi

% \begin{example}[NP-Lookahead-HMC]
% \label{ex:np-la-hmc}
% Look Ahead HMC \cite{DBLP:conf/icml/Sohl-DicksteinMD14,DBLP:journals/jcphy/CamposS15} can be seen as a mixture of HMC samplers \changed[fz]{(with persistent momentum)} that perform different numbers of leapfrog steps.
% % With an appropriate weight for each, we can formulate it as
% \changed[fz]{Similarly, we can construct a mixture of NP-HMC with Persistent Momentum (\cref{ex:np-hmc persistent}) to obtain Nonparametric Lookahead HMC (see \cref{app: np-la-hmc}).}
% \end{example}

\subsection{Auxiliary direction}
\label{sec: auxiliary direction}

Suppose we want to use sophisticated \changed[cm]{bijective} but non-involutive endofunctions
$f^{(n)}$ on $\entspace^n \times \auxspace^n$
to better explore the parameter space and return proposals with a high acceptance ratio.
% Assuming each $f^{(n)} $ is bijective and
Assuming
both families $\set{f^{(n)}}_n$ and $\set{\inv{f^{(n)}}}_n$ satisfy the
\invoass{} (\cref{vass: partial block diagonal inv}),
we can construct an NP-iMCMC sampler with \defn{auxiliary direction}, which
% samples a direction $\dira \in \dirspace := \set{+,-}$ with equal probability; and
% generates the next sample by running
% \cref{np-imcmc step: aux sample,np-imcmc step: involution,np-imcmc step: extend,np-imcmc step: accept/reject} of the NP-iMCMC sampler using
% $\set{f^{(n)}}_n$ to suggest the proposal sample if $\dira$ is sampled to be $+$;
% otherwise $\set{\inv{f^{(n)}}}_n$ is used.
\begin{compactitem}
  \item
  samples a direction $\dira \in \dirspace := \set{+,-}$ with equal probability; and
  \item
  generates the next sample by running
  \cref{np-imcmc step: aux sample,np-imcmc step: involution,np-imcmc step: extend,np-imcmc step: accept/reject}
  of the NP-iMCMC sampler using
  $\set{f^{(n)}}_n$ to suggest the proposal sample if $\dira$ is sampled to be $+$;
  otherwise $\set{\inv{f^{(n)}}}_n$ is used.
\end{compactitem}
See \cref{sec: auxiliary direction np-imcmc} for details of the algorithm.

Notice that since the distribution of the direction variable $\dira$
is the discrete uniform distribution,
we do not need to alter the acceptance ratio in \cref{np-imcmc step: accept/reject}.

\begin{example}[NP-HMC]
\label{eg: np-hmc}
% Using the method described above,
We can formulate the recently proposed
Nonparametric Hamiltonian Monte Carlo sampler in \cite{DBLP:conf/icml/MakZO21}
using the
\changed[cm]{(Multiple Step)}
NP-iMCMC framework with auxiliary direction,
in which case the sophisticated non-involutive endofunction is the
leapfrog method $\leapfrog$.
(See \cref{app: np-hmc} for details.)
\end{example}

\begin{remark}
This corresponds to Trick 3 described in \cite{DBLP:conf/icml/NeklyudovWEV20}.
% which can be found in the construction of the popular Hamiltonian Monte Carlo (HMC) sampler,
\changed[fz]{Trick 4 from \cite{DBLP:conf/icml/NeklyudovWEV20} cannot be applied
in our framework because the \invoass{} (\cref{vass: partial block diagonal inv})
is not closed under composition.}
%Hence, iMCMC samplers that follows Trick 4 cannot be made to work on universal PPL using our method.
\end{remark}

% Actually we don't know how to train the bijections in L2HMC in such a way that accomodates for the nonparametric nature of the density function.

% \paragraph{Example (NP-L2HMC)}

% The L2HMC (Learning to HMC) \cite{DBLP:conf/iclr/LevyHS18} is a generalisation of HMC where the parametric leapfrog operator is \emph{learnt} using neural networks.
% Here we are not concerned with the training stage, but rather the conditions under which a learnt parametric leapfrog operator can be extended to nonparametric models.

% Assume $\set{\ninvo{n}}_n$ are the trained leapfrog operators, and each $\ninvo{n}$ and its inverse are partially block diagonal (\cref{vass: partial block diagonal inv}).
% Then NP-L2HMC can be viewed as an instance of AD NP-iMCMC (\cref{alg: AD NP-iMCMC}).

\subsection{Persistence}

%Composition of transition kernels with the same invariant distribution is also invariant against this distribution (\cref{prop: invariant is closed under composition}).
% Composition of transition kernels preserves invariant distributions
% (\cref{prop: invariant is closed under composition}):
% if a distribution is invariant for a pair of (composable) transition kernels,
% it is also invariant for their composite.
% As MCMC samplers can be seen as transition kernels (see \cref{par: transition kernel} for details),
% this result tells us that the composition of NP-iMCMC samplers also
% preserves the invariant distribution.
% We illustrate how this is beneficial in designing efficient samplers.

Suppose we want a nonreversible sampler, so as to obtain better mixing times.
%, i.e.~more efficient convergence to the target distribution.
A typical way of achieving nonreversibility from an originally reversible MCMC sampler
is to reuse the value for a variable
(that is previously resampled in the original reversible sampler)
in the next iteration if the proposed sample is accepted.
In this way, the value of such a variable is allowed to \emph{persist}, making the sampler nonreversible.

A key observation made by \citet{DBLP:conf/icml/NeklyudovWEV20} is that
the composition of reversible iMCMC samplers can yield a nonreversible sampler.
% using a similar method.
% (See \cref{eg: composition of imcmc udpates} for more details.)
%This can be achieved by either persisting the direction of the involution (Trick 5) or an auxiliary kernel (Trick 6).
Two systematic techniques to achieve nonreversibility are persistent direction (Trick 5) and an auxiliary kernel (Trick 6).
% We present a similar application to NP-iMCMC samplers, making the resulting sampler nonreversible.
We present similar approaches for NP-iMCMC samplers.

% \subsubsection{Persistent direction}
% \label{sec: persistent direction}

Suppose there is an NP-iMCMC sampler that uses the auxiliary direction as described in \cref{sec: auxiliary direction},
i.e.~there is a non-involutive bijective endofunction $f^{(n)}$ on $\entspace^n \times \auxspace^n$ for each $n\in\Nat$ such that
$\set{f^{(n)}}_n$ and $\set{\inv{f^{(n)}}}_n$ satisfy the \invoass{}
(\cref{vass: partial block diagonal inv}).
\changed[cm]{In addition, assume there are
two distinct families of auxiliary kernels,
namely
$\set{\nauxkernel{n}_{+}}_n $ and
$\set{\nauxkernel{n}_{-}}_n $.}
The \changed[fz]{corresponding} NP-iMCMC sampler with \defn{persistence}
\begin{compactitem}
  \item
  proposes the next sample by running \cref{np-imcmc step: aux sample,np-imcmc step: involution,np-imcmc step: extend}
  of the NP-iMCMC sampler with $\set{\nauxkernel{n}_{+}}_n $ and $\set{f^{(n)}}_n$ if $\dira$ is sampled to be $+$;
  otherwise $\set{\nauxkernel{n}_{-}}_n $ and $\set{\inv{f^{(n)}}}_n$ are used;

  \item
  accepts the proposed sample with probability indicated in \cref{np-imcmc step: accept/reject}
  of the NP-iMCMC sampler;
  otherwise repeats the current sample and
  \emph{flips the direction $\dira$}.
\end{compactitem}
See \cref{sec: persistent np-imcmc} for details of the algorithm.
The family of kernels and maps indeed \emph{persist} across multiple iterations if the proposals of these iterations are accepted.
%This makes intuitive sense as it indicates that we have more faith in a family of maps to propose good samples, if it has a good record of doing so.
\changed[fz]{The intuitive idea behind this is that if a family of kernels and maps perform well (proposals are accepted) in the current part of the sample space, we should keep it, and otherwise switch to its inverse.}

\begin{remark}
  This corresponds to Tricks 5 and 6 described in \cite{DBLP:conf/icml/NeklyudovWEV20},
  which can be found in nonreversible MCMC sampler like
  the Generalised HMC algorithm \cite{HOROWITZ1991247},
  the Look Ahead HMC sampler \cite{DBLP:conf/icml/Sohl-DicksteinMD14,DBLP:journals/jcphy/CamposS15}
  and
  Lifted MH \cite{DBLP:journals/corr/abs-0809-0916}.
\end{remark}

\begin{example}[NP-HMC with Persistence]
\label{ex:np-hmc persistent}
% Let's construct an nonreversible variant of the NP-HMC sampler (as discussed in \cref{eg: np-hmc})
% with the non-involutive bijective leapfrog \changed[fz]{integrator} $\leapfrog$.
The nonreversible HMC sampler in \cite{HOROWITZ1991247} uses persistence, and, in addition, (partially) reuses the momentum vector from the previous iteration.
As shown in \cite{DBLP:conf/icml/NeklyudovWEV20}, it can be viewed as a composition of iMCMC kernels.
Using the method indicated above, we can  also add persistence to NP-HMC.
(See \cref{app: gen np-hmc} for details.)
\end{example}

\begin{example}[NP-Lookahead-HMC]
\label{ex:np-la-hmc}
\changed[cm]{
Look Ahead HMC \cite{DBLP:conf/icml/Sohl-DicksteinMD14,DBLP:journals/jcphy/CamposS15} can be seen as an HMC sampler with persistence that generates a new state with a varying number of leapfrog steps, depending on the value of the auxiliary variable.
% With an appropriate weight for each, we can formulate it as
Similarly, we can construct an NP-HMC sampler with Persistence that varies the numbers of leapfrog steps.
(See \cref{app: look ahead np-hmc} for more details.)}
\end{example}

\section{Experiments}
\label{sec: experiments}
% !TEX root = ./../icml2022.tex

%\lo{Numerous broken references in this section!}

\subsection{Nonparametric Metropolis-Hastings}

\begin{figure}[t]
  \includegraphics[width=0.5\textwidth]{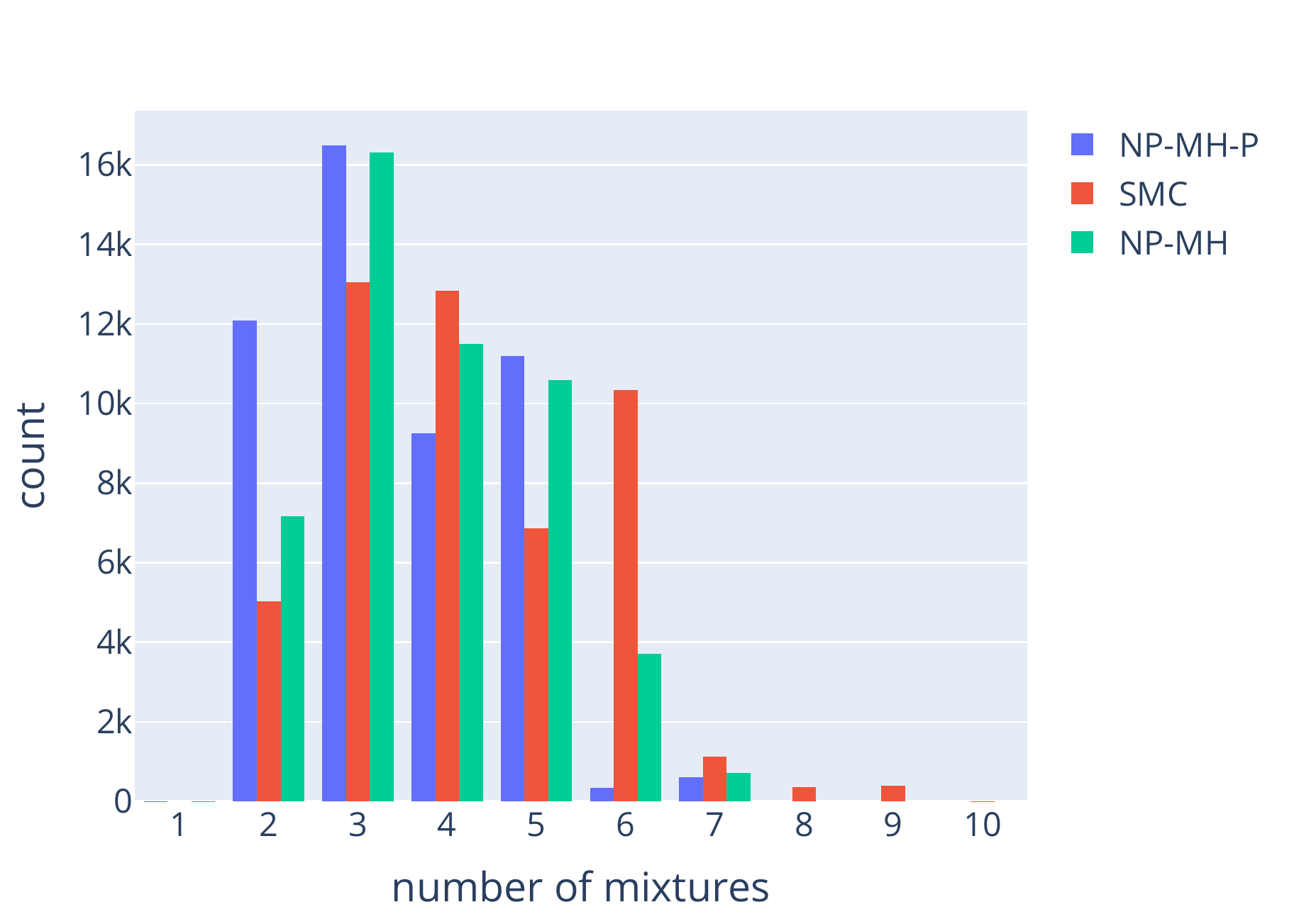}
  \caption{Histogram of the number of components for the infinite GMM; correct posterior is 3.}
  \label{fig: gmm plot}
\end{figure}

We first implemented two simple instances of the NP-iMCMC sampler, namely
NP-MH (\cref{app: np-mh}) and
NP-MH with Persistence (\cref{app: lifted np-mh})
in the Turing language \cite{DBLP:conf/aistats/GeXG18}.\footnote{
\changed[cm]{
The code to reproduce the Turing experiments is available in
\url{https://github.com/cmaarkol/nonparametric-mh}.}
% \fz{@Carol TODO: add repo link for Julia code}
% the supplementary material.%
Turing's SMC implementation is nondeterministic (even with a fixed random seed), so its results may vary somewhat, but everything else is exactly reproducible.}
We compared them with Turing's built-in Sequential Monte Carlo (SMC) algorithm on
an infinite Gaussian mixture model (GMM) where the number of mixture components is drawn from a normal distribution.
Posterior inference is performed on 30 data points generated from a ground truth with three components.
The results of ten runs with 5000 iterations each (\cref{fig: gmm plot}) suggest that
\changed[cm]{the NP-iMCMC samplers work pretty well.}
% persistence helps with inferring the correct number of components.
% Moreover it suggests that nonparametric extensions of efficient algorithms remains to be efficient.
% NP-Lifted-MH retains its performance gain against others.
\iffalse
\lo{The preceding sentence implies the existence of a prior comparison of these algorithms with NP-Lifted-MH being the lead performer. What is this comparison?}
\cm{I meant the following. \cite{DBLP:journals/corr/abs-0809-0916} shows that Lifted MH works better than MH. In \cref{fig: gmm plot}, it seems that their nonparametric extensions show a similar result. One experiment is too weak to draw any correlations; perhaps it is best to leave it for this abstract.}
This suggests that the NP-iMCMC method is a reasonable extension to existing iMCMC algorithms.
\fi

\subsection{Nonparametric Hamiltonian Monte Carlo}

Secondly, we consider Nonparametric HMC \cite{DBLP:conf/icml/MakZO21}, mentioned in \cref{eg: np-hmc} before.
We have seen how the techniques from \cref{sec: combination} can yield nonreversible versions of NP-iMCMC inference algorithms.
Here, we look at \changed[fz]{nonparametric versions of} two extensions \changed[fz]{described in \cite{DBLP:conf/icml/NeklyudovWEV20}}:
persistence (\cref{ex:np-hmc persistent}) and lookahead (\cref{ex:np-la-hmc}).
Persistence means that the previous momentum vector is reused in the next iteration.
It is parametrised by $\alpha \in [0,1]$ where $\alpha = 1$ means no persistence (standard HMC) and $\alpha = 0$ means full persistence (no randomness added to the momentum vector).
Lookahead HMC is parametrised by $K \geq 0$, which is the number of extra iterations (``look ahead'') to try before rejecting a proposed sample (so $K = 0$ corresponds to standard HMC).
\changed[fz]{Detailed descriptions of these algorithms and how they fit into the}
\changed[cm]{(Multiple Step)}
\changed[fz]{NP-iMCMC framework can be found in \cref{app: gen np-hmc,app: look ahead np-hmc}.} %\lo{Broken reference}}

We evaluate these extensions of NP-HMC on the benchmarks from \cite{DBLP:conf/icml/MakZO21}: a model for the geometric distribution, a model involving a random walk, and an unbounded Gaussian mixture model.
Note that similarly to \cite{DBLP:conf/icml/MakZO21}, we actually work with a discontinuous version of NP-HMC, called \emph{NP-DHMC}, which is a nonparametric extension of discontinuous HMC \cite{NishimuraDL20}.%
\footnote{\changed[fz]{The source code is available at \url{https://github.com/fzaiser/nonparametric-hmc}.}}
The discontinuous version can handle the discontinuities arising from the jumps between dimensions more efficiently.
We don't discuss it in this paper due to lack of space.
However, the modifications necessary to this discontinuous version are the same as for the standard NP-HMC.
\changed[fz]{\citet{DBLP:conf/icml/MakZO21} demonstrated the usefulness of NP-DHMC and how it can obtain better results than other general-purpose inference algorithms like Lightweight Metropolis-Hastings and Random-walk Lightweight Metropolis-Hastings.
Here, we focus on the benefits of nonreversible versions of NP-DHMC, which were derived using the}
\changed[cm]{(Multiple Step)}
\changed[fz]{NP-iMCMC framework.}

\begin{table}[t]
\caption{Geometric distribution example: total variation difference from the ground truth, averaged over 10 runs, and standard deviation. Each run: $10^3$ samples, $L$ leapfrog steps, step size $\epsilon = 0.1$, persistence parameter $\alpha \in \{0.5\}$.}
\label{table:geometric-results}

\begin{tabular}{ccc}
&persistence &TVD from ground truth \\
\hline
$L = 5$ & --- & $0.0524 \pm 0.0069$ \\
$L = 5$ & $\alpha = 0.5$ & $0.0464 \pm 0.0074$ \\
$L = 5$ & $\alpha = 0.1$ & $0.0461 \pm 0.0083$ \\
\hline
$L = 2$ & --- & $0.0768 \pm 0.0181$ \\
$L = 2$ & $\alpha = 0.5$ & $0.0570 \pm 0.0115$ \\
$L = 2$ & $\alpha = 0.1$ & $0.0534 \pm 0.0058$
\end{tabular}
\end{table}

\mypara{Geometric distribution}
The geometric distribution benchmark from \cite{DBLP:conf/icml/MakZO21} illustrates the usefulness of persistence:
we ran NP-DHMC for a step count $L \in \{ 2, 5 \}$ with and without persistence.
As can be seen in \cref{table:geometric-results}, persistence usually decreases the distance from the ground truth.
In fact, the configuration $L = 2, \alpha = 0.1$ is almost as good as $L = 5$ without persistence, despite taking 2.5 times less computing time.

\begin{figure}
\includegraphics[width=\columnwidth]{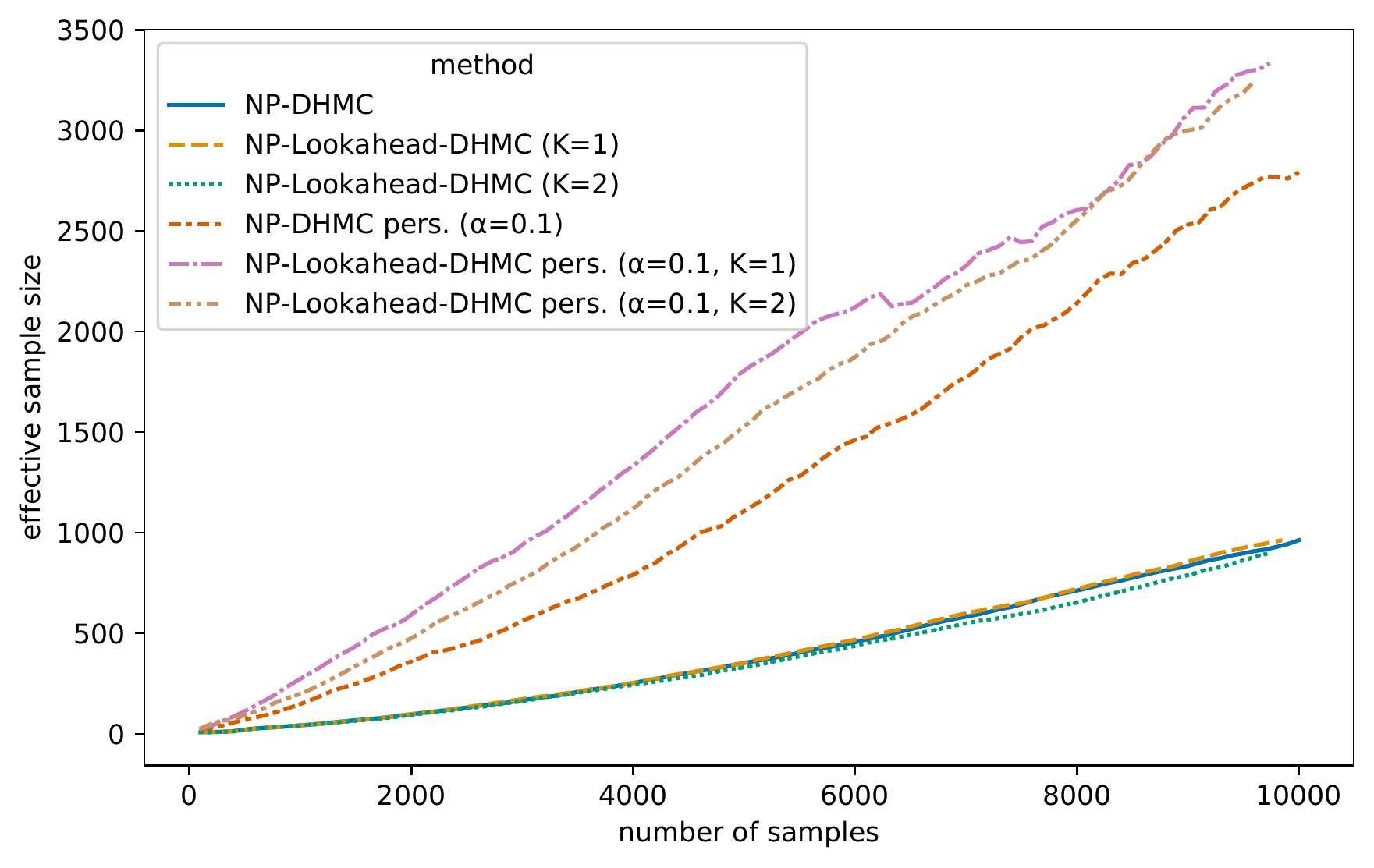}
\caption{ESS for the random walk example in terms of number of samples, computed from 10 runs.
Each run: $10^3$ samples with $L = 5$ leapfrog steps of size $\epsilon = 0.1$, persistence parameter $\alpha \in \{0.5, 0.1\}$, and look-ahead $K \in \{1, 2\}$.}
\label{fig:walk-ess}
\end{figure}

\mypara{Random walk}
The next benchmark from \cite{DBLP:conf/icml/MakZO21} models a random walk and observes the distance travelled.
\Cref{fig:walk-ess} shows the effective sample size (ESS) in terms of the number of samples drawn, comparing versions of NP-DHMC with persistence ($\alpha = 0.5$) and look-ahead ($K \in \{1, 2\}$).
We can see again that persistence is clearly advantageous.
Look-ahead ($K \in \{1,2\}$) seems to give an additional boost on top.
We ran all these versions with the same computation time budget, which is why the the lines for $K = 1,2$ are cut off before the others.

\begin{figure}[t]
\includegraphics[width=\columnwidth]{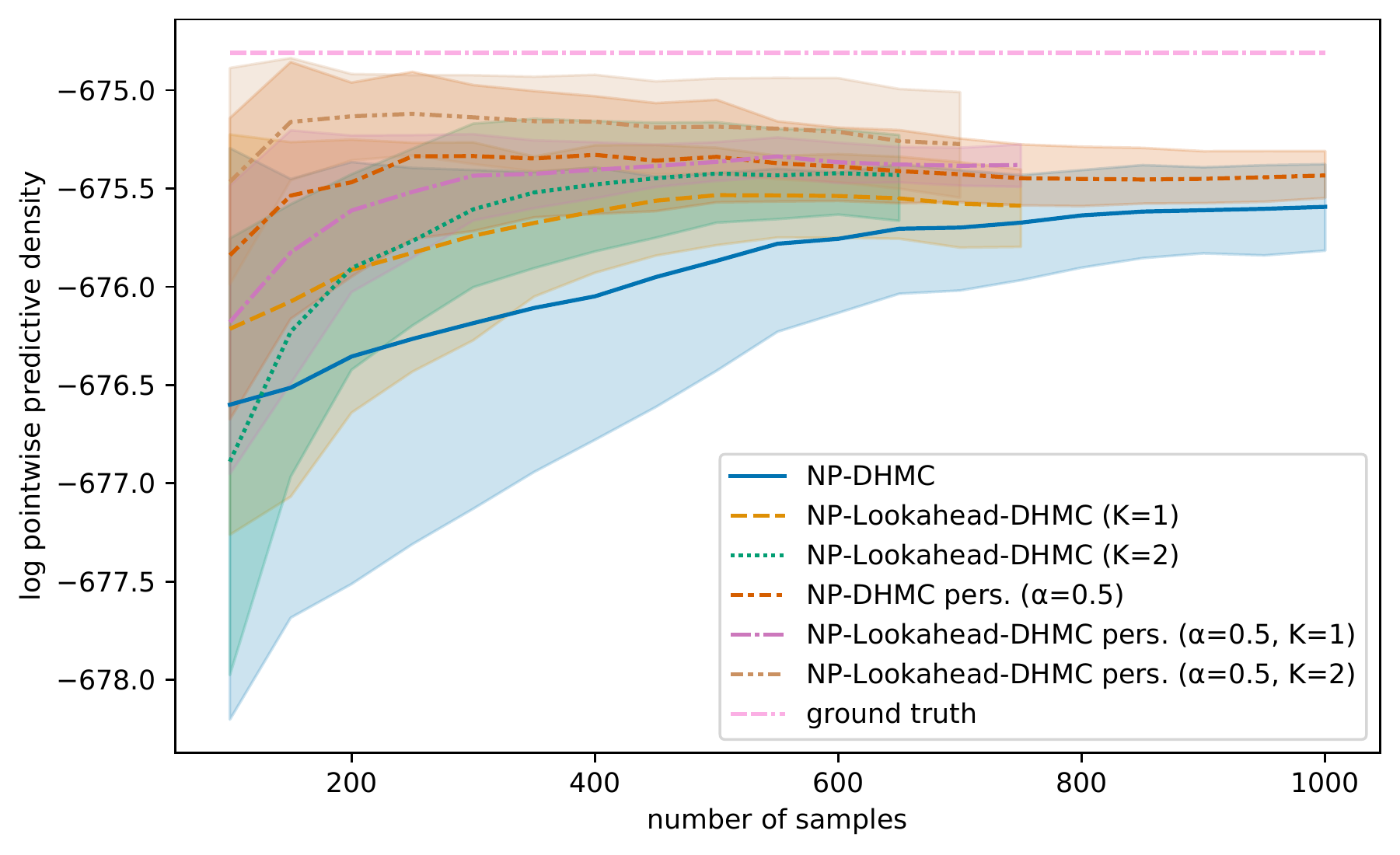}
\caption{Gaussian mixture with Poisson prior: LPPD in terms of number of samples, averaged over 10 runs.
The shaded area is one standard deviation.
Each run: $10^3$ samples with $L = 25$ leapfrog steps of size $\epsilon = 0.05$, persistence parameter $\alpha = 0.5$, and look-ahead $K \in \{1, 2\}$.}
\label{fig:gmm-lppd}
\end{figure}

\mypara{Unbounded Gaussian mixture model}
Next, we consider a Gaussian mixture model where the number of mixture components is drawn from a Poisson prior.
Inference is performed on a training data set generated from a mixture of 9 components (the ground truth).
We then compute the log pointwise predictive density (LPPD) on a test data set drawn from the same distribution as the training data.
The LPPD is shown in \cref{fig:gmm-lppd} in terms of the number of samples.
Note that again, all versions were run with the same computation budget, which is why some of the lines are cut off early.
Despite this, we can see that the versions with lookahead ($K \in \{1,2\}$) converge more quickly than the versions without lookahead.
Persistent direction ($\alpha = 0.5$) also seems to have a (smaller) benefit.

\begin{figure}
\includegraphics[width=\columnwidth]{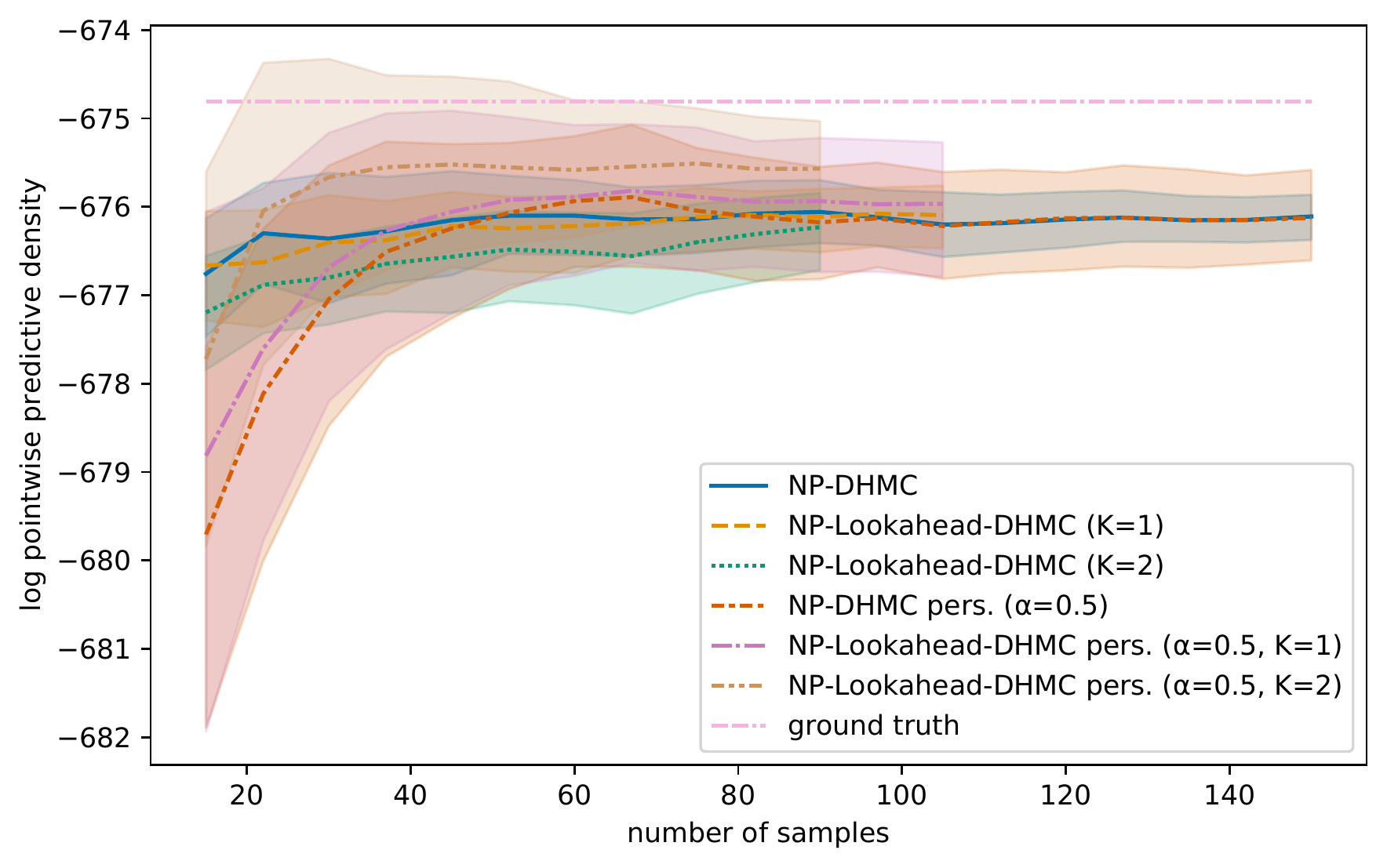}
\caption{Dirichlet process mixture: LPPD in terms of number of samples, averaged over 10 runs.
The shaded area is one standard deviation.
Each run: 150 samples with $L = 20$ leapfrog steps of size $\epsilon = 0.05$, persistence parameter $\alpha = 0.5$, and look-ahead $K \in \{1, 2\}$.}
\label{fig:dpmm-lppd}
\end{figure}

\mypara{Dirichlet process mixture model}
Finally, we consider a Gaussian mixture whose weights are drawn from a Dirichlet process.
The rest of the setup is the same as for the Poisson prior, and the results are shown in \cref{fig:dpmm-lppd}.
The version with persistence is worse at the start but obtains a better LPPD at the end.
Look-ahead ($K \in \{1, 2\}$) yields a small additional boost in the LPPD.
\changed[fz]{It should be noted that the variance over the 10 runs is larger in this example than in the previous benchmarks, so the conclusion of this benchmark is less clear-cut.}

\section{Related work and Conclusion}
\label{sec: conclusion}
% !TEX root = ./../icml2022.tex

\mypara{Involutive MCMC and its instances}

The involutive MCMC \changed[fz]{framework} %sampler
\cite{DBLP:conf/icml/NeklyudovWEV20,DBLP:conf/pldi/Cusumano-Towner19,Matheos2020}
can in principle be used for nonparametric models
by setting $X := \bigcup_{n\in\Nat}\parspace^n$ and
$Y := \bigcup_{n\in\Nat}\auxspace^n$ in \cref{fig:imcmc algo} and
defining an auxiliary kernel on
$X \kernelto Y := \bigcup_{n\in\Nat}\parspace^n \kernelto \bigcup_{n\in\Nat}\auxspace^n$
an involution on
$X \times Y := \bigcup_{n\in\Nat}\parspace^n \times \auxspace^n$.
For instance, {Reversible Jump MCMC}
%(RJMCMC)
\cite{Green95}
is an instance of iMCMC that works for the infinite GMM model,
with the split-merge proposal \cite{https://doi.org/10.1111/1467-9868.00095} specifying when and how states can ``jump'' across dimensions.
However, designing appropriate auxiliary kernels and involutions
\changed[cm]{
that enable the extension of an iMCMC sampler to nonparametric models}
remains challenging and model specific.
By contrast, NP-iMCMC only requires the specification of involutions on the finite-dimensional space $\parspace^n \times \auxspace^n$;
moreover, it provides a \emph{general} procedure (via \cref{np-imcmc step: extend}) that drives state movement between dimensions.
For designers of nonparametric samplers who do not care to custom build trans-dimensional methods, we contend that NP-iMCMC is their method of choice.

\changed[cm]{
The performance of NP-iMCMC and iMCMC depends on the complexity of the respective auxiliary kernels, involutions and the model in question.
Take iGMM for example.
RJMCMC with
the split-merge proposal which computes the weight, mean, and variance of the new component(s)
would be slower than
NP-MH, an instance of NP-iMCMC with a computationally light involution (a swap),
but more efficient than
NP-HMC, an instance of (Multiple Step) NP-iMCMC with the computationally heavy leapfrog integrator as involution.}

\mypara{Trans-dimensional samplers}
A standard MCMC algorithm for universal PPLs is the Lightweight Metropolis-Hastings algorithm (LMH) \cite{DBLP:conf/aistats/YangHG14,DBLP:conf/pkdd/TolpinMPW15,DBLP:conf/aistats/RitchieSG16}.
Widely implemented in several universal PPLs (Anglican, Venture, Gen, and Web PPL), LMH performs single-site updates on the current sample and re-executes the program from the resampling point.

% Variational inference (VI) \cite{BleiKM17} solves the inference problem by treating it as an optimisation problem.
% When adapted to probabilistic programs, the Blackbox VI \cite{RanganathGB14}, which uses the score function gradient estimator \cite{DBLP:conf/icml/MnihG14,DBLP:journals/corr/abs-1301-1299}, can in principle be applied to a large class of branching and recursive programs because only the variational density functions need to be differentiable.
% In practice, however, VI of probabilistic programs are far from automatic: the guide programs (that express variational distributions) still need to be hand-coded in the main.

Divide, Conquer, and Combine (DCC) \cite{DBLP:conf/icml/ZhouYTR20} is an inference algorithm that is applicable to probabilistic programs that use branching and recursion.
A hybrid algorithm, DCC solves the problem of designing a proposal that can efficiently transition between configurations by performing local inferences on submodels, and returning an appropriately weighted combination of the respective samples.

\citet{DBLP:conf/icml/MakZO21} have recently introduced Nonparametric Hamiltonian Monte Carlo (NP-HMC), which generalises HMC to nonparametric models.
%Inputs to NP-HMC are a new class of density functions called \emph{tree representable}, which serve as a language-independent representation of the density functions of programs in a universal PPL.
\fz{Since we use tree representable functions in the main text, I don't think we need to mention it here. I suggest the following instead:}
\fz{As we have seen, NP-HMC is an instantiation of NP-iMCMC.}
\changed[cm]{As we've seen, NP-HMC is an instantiation of (Multiple Step) NP-iMCMC.}
\lo{But this is not true.}
% \subsection*{Conclusion}
\paragraph{Conclusion}
We have introduced the \emph{nonparametric involutive MCMC algorithm} %(\cref{sec: np-imcmc})
as a general framework for designing MCMC algorithms for models expressible in a universal PPL, and provided a correctness proof. %(\cref{app: proof of correctness}).
%\changed[lo]{
To demonstrate the relative ease of make-to-order design of nonparametric extensions of existing MCMC algorithms, %(with inherited correctness proof),
we have constructed several \emph{new} algorithms, and demonstrated empirically that the expected features and statistical properties are preserved.

\section*{Acknowledgements}

We thank the reviewers for their
insightful feedback and pointing out important related work.
We are grateful to Maria Craciun who gave detailed comments on an early draft, and
to Hugo Paquet and Dominik Wagner for their helpful comments and advice. We gratefully acknowledge support from
the EPSRC and the Croucher Foundation.

\clearpage

\bibliography{front-matter/database.bib}
\bibliographystyle{icml2022}

%%%%%%%%%%%%%%%%%%%%%%%%%%%%%%%%%%%%%%%%%%%%%%%%%%%%%%%%%%%%%%%%%%%%%%%%%%%%%%%
%%%%%%%%%%%%%%%%%%%%%%%%%%%%%%%%%%%%%%%%%%%%%%%%%%%%%%%%%%%%%%%%%%%%%%%%%%%%%%%
% APPENDIX
%%%%%%%%%%%%%%%%%%%%%%%%%%%%%%%%%%%%%%%%%%%%%%%%%%%%%%%%%%%%%%%%%%%%%%%%%%%%%%%
%%%%%%%%%%%%%%%%%%%%%%%%%%%%%%%%%%%%%%%%%%%%%%%%%%%%%%%%%%%%%%%%%%%%%%%%%%%%%%%
\newpage
\appendix
\onecolumn
\icmltitle{Nonparametric Hamiltonian Monte Carlo (Appendix)}

\vskip 0.3in

\startcontents[sections]
\printcontents[sections]{l}{1}{\setcounter{tocdepth}{2}}
% !TEX root = ./../icml2022.tex

\lo{TODO: Give an outline of the appendix.}

% \section{Involutive MCMC}
% \label{app: involutive MCMC}
% \lo{This section gives an account of some of the contents of \cite{DBLP:conf/icml/NeklyudovWEV20}.
% Is it really necessary to give an account of published materials?}
% \input{section/imcmc}
% \clearpage

\section{Statistical PCF}
\label{app: spcf}
% !TEX root = ./../icml2022.tex

\iffalse
\cm{TODO:
\begin{compactenum}[(1)]
  \item Define a suitable Bayesian model to showcase the language.
  \item Draw the program tree of a term.
  \item Read all papers that introduces a PPL and compare the presentation and language.
\end{compactenum}
}
\fi
% outline
%In this section, we present a simply-typed functional Turing-complete probabilistic programming language (PPL) with (stochastic) branching and recursion, and its operational semantics.
In this section, we present a functional probabilistic programming language (PPL) with (stochastic) branching and recursion, and its operational semantics.
We also define what it means for a program to be almost surely terminating and integrable.
We conclude the section by showing that a broad class of programs satisfies the assumptions for the NP-iMCMC inference algorithm described in \cref{sec: np-imcmc}.

% This language serves two purposes for the NP-HMC algorithm.
% This language is a purified universal probabilistic programming language (PPL) widely considered \cite{DBLP:conf/icfp/BorgstromLGS16,VakarKS19,MakOPW20} which specifies tree-representable functions that satisfies assumptions and hence NP-HMC can be applied.
% Second, its (operational) semantics is used to prove correctness of NP-HMC in \cref{appendix: correctness}.

\subsection{Syntax}
\label{subsec:stat PCF}

%Statistical PCF (SPCF) is a probabilistic version of the call-by-value simply-typed functional Turing-complete PCF \cite{DBLP:journals/tcs/Scott93,DBLP:conf/fsttcs/Sieber90} with Real and Boolean as the ground types.
Statistical PCF (SPCF) is a statistical probabilistic extension of the call-by-value PCF \cite{DBLP:journals/tcs/Scott93,DBLP:conf/fsttcs/Sieber90} with the reals and Booleans as the ground types.
The terms and part of the typing system of SPCF are presented in \cref{fig:SPCF syntax}.

SPCF has three probabilistic constructs:
\begin{compactenum}[(1)]

  \item The continuous sampler $\normal$ draws from the standard Gaussian distribution $\Gau$ with mean $0$ and variance $1$.

  \item The discrete sampler $\coin$ is a fair coin (formally $\coin$ draws from the Bernoulli distribution $\Bern(0.5)$ with probability $0.5$).

  \item The scoring construct $\score{\terma}$ enables conditioning on observed data by multiplying the weight of the current execution with the real number denoted by $\terma$.

\end{compactenum}

% no space
\begin{remark}[Continuous Sampler]
  The continuous sampler in most PPLs \cite{DBLP:conf/esop/CulpepperC17,DBLP:journals/pacmpl/WandCGC18,DBLP:journals/pacmpl/EhrhardPT18,VakarKS19,MakOPW21}
  draw from the standard uniform distribution $\Uni$ with endpoints $0$ and $1$.
  However, we decided against $\Uni$ since its support is not the whole of $\Real$, which is a common target space for inference algorithms (e.g.~Hamiltonian Monte Carlo (HMC) inference algorithm).
  Instead our continuous sampler draws from the \emph{standard normal distribution} $\Gau$ which has the whole of $\Real$ as its support.
  This design choice does not restrict nor extend our language as we will see in \cref{example: not restrictive}.
  \lo{This can be shortened, if space becomes an issue.}
\end{remark}

\begin{remark}[Discrete Sampler]
\iffalse
  Contrast to the uniformity of the continuous sampler in the literature,
  the discrete sampler varies from languages to languages.
  For instance,
  \cite{DBLP:journals/iandc/DanosE11,scibior2017denotational} sample from the fair coin,
  whereas \cite{DBLP:conf/popl/EhrhardTP14} samples from the discrete uniform distribution.
\fi
  Like \cite{DBLP:journals/iandc/DanosE11,scibior2017denotational},
  we choose the fair coin as our discrete sampler for its simplicity.
  However, as shown in \cref{example: discrete from coin}, this is not limiting.
  \cite{DBLP:conf/popl/EhrhardTP14}, for example, samples from the discrete uniform distribution.
\end{remark}

\iffalse
\begin{remark}[Multiple Samplers] \label{remark:samplers}
  Though it is expected that practical PPLs support both continuous and discrete distributions,
  most purified languages only study one or the other,
  \cite{DBLP:conf/lics/StatonYWHK16,DBLP:conf/esop/Staton17,scibior2017denotational} being the exceptions.
  As shown in \cref{example: not restrictive}, one \emph{can} describe discrete distributions using a continuous sampler.
  However, some inference algorithms (e.g.~Mixed HMC \cite{DBLP:conf/nips/Zhou20}) apply special treatment to discrete variables. For this reason, we include both samplers in our language.
\end{remark}
\fi

\begin{figure}[t]
  \defn{Types} (typically denoted $\typea,\typeb$) and
  \defn{terms} (typically $\terma,\termb,\termc $):
  \begin{align*}
    % types
    \typea,\typeb & ::=
    \tyreal \mid \tybool \mid \typea \tyarrow \typeb \\
    % terms
    \terma,\termb,\termc & ::=
    \pcf{\reala} \mid
    \pcf{\boola} \mid
    \pcf{\funca}(\terma_1,\dots,\terma_\ell) \tag{Constants and functions} \\
    & \mid
    y \mid
    \lambda y.\terma \mid
    \terma\,\termb \tag{Higher-order} \\
    & \mid
    \pcfif{\termc}{\terma}{\termb}\mid
    \Y{\terma} \tag{Branching and recursion} \\
    & \mid
    \normal \mid
    \coin \mid
    \score{\terma} \tag{Probabilistic}
  \end{align*}
  \defn{Typing system}:
  $$
    \AxiomC{$\boola \in \Bool$}
    \UnaryInfC{$\Gamma \vdash \pcf{\boola} :\tybool$}
    \DisplayProof
    \qquad
    \AxiomC{$
      \set{\Gamma \vdash \terma_i : \tyreal}_{i=1}^{n}
    $}
    \AxiomC{$
      \set{\Gamma \vdash \termb_j : \tybool}_{j=1}^{m}
    $}
    \AxiomC{$\funca: \Real^n \times \Bool^{m} \partialto \groundspace$}
    \TrinaryInfC{$
      \Gamma \vdash \pcf{\funca}{(\terma_1,\dots, \terma_n, \termb_{1}, \dots, \termb_{m})} :
      \begin{cases}
        \tyreal & \text{if }G = \Real \\
        \tybool & \text{if }G = \Bool
      \end{cases}
    $}
    \DisplayProof
  $$\\
  $$
    \AxiomC{$\Gamma \vdash \termc : \tybool$}
    \AxiomC{$\Gamma \vdash \terma : \typea$}
    \AxiomC{$\Gamma \vdash \termb : \typea$}
    \TrinaryInfC{$\Gamma \vdash \pcfif{\termc}{\terma}{\termb} : \typea$}
    \DisplayProof
    \qquad
    \AxiomC{$
      \Gamma \vdash \terma : (\typea \tyarrow \typeb) \tyarrow (\typea \tyarrow \typeb)
    $}
    \UnaryInfC{$
      \Gamma \vdash \Y{\terma} : \typea \tyarrow \typeb
    $}
    \DisplayProof
  $$ \\
  $$
    \AxiomC{\vphantom{$4$}}
    \UnaryInfC{$\Gamma \vdash \normal:\tyreal$}
    \DisplayProof
    \qquad
    \AxiomC{\vphantom{$4$}}
    \UnaryInfC{$\Gamma \vdash \coin:\tybool$}
    \DisplayProof
    \qquad
    \AxiomC{$\Gamma \vdash \terma:\tyreal$}
    \UnaryInfC{$\Gamma \vdash \score{\terma}:\tyreal$}
    \DisplayProof
    \qquad
  $$
  \caption{
    Syntax of SPCF, where
    $\reala, \realb, \realc \in \Real$,
    $\boola, \boolb \in \Bool$,
    $x,y,z$ are variables, and
    $\funca, \funcb, \funcc$ ranges over a set $\primitives$ of primitive functions.
  }
  \label{fig:SPCF syntax}
\end{figure}

Following the convention,
the set of all terms is denoted as $\terms$ with meta-variables $\terma,\termb,\termc$,
the set of free variables of a term $\terma$ is denoted as $\freevar{\terma}$ and
the set of all closed terms is denoted as $\closedterms$.
In the interest of readability, we sometimes use pseudocode in the style of ML (e.g.~\cref{example: spcf terms}) to express SPCF terms.

\begin{example}
  \label{example: spcf terms}
  \codeinline{let rec f x = if coin then f(x+normal) else x in f 0} is a simple program which keeps tossing a coin and sampling from the normal distribution until the first coin failure, upon which it returns the sum of samples from the normal distribution.
\end{example}

\subsection{Primitive Functions}
\label{app: primitive functions}

Primitive functions play an important role in the expressiveness of SPCF. To be concise, we only consider partial, measurable functions of types $\Real^n \times \Bool^{m} \partialto \Real$ or $\Real^n \times \Bool^{m} \partialto \Bool$ for some $n, m \in \Nat$.
Examples of these primitives include addition $+$, division $/$, comparison $<$ and equality $=$.
As we will see in \cref{example: not restrictive,example: discrete from coin}, it is important that the cumulative distribution functions (cdf) and probability density functions (pdf) of distributions are amongst the primitives in $\primitives$.
However, we do not require all measurable functions to be primitives, unlike \cite{DBLP:conf/lics/StatonYWHK16,DBLP:conf/esop/Staton17}.

\begin{example}
  \label{example: not restrictive}
  \begin{compactenum}[(1)]

    \item Let \codeinline{cdfnormal} be the cdf of the standard normal distribution.
    Then, the standard uniform distribution with endpoints $0$ and $1$ can be described as \codeinline{uniform = cdfnormal(normal)}.

    \item Any distribution with an inverse cdf \codeinline{f} in the set of primitives can be described as \codeinline{f(uniform)}.
    For instance, the inverse cdf of the exponential distribution (with rate $1$) is $f(p) := -\ln(1-p)$ and hence \codeinline{-ln(1-uniform)} describes the distribution.

    \item The Poisson distribution can be specified using the uniform distribution (\cite{Devroye1986}) as follows.
\begin{code}
Poi(rate) = let p = exp(-rate) in
            let rec f x p s = if s < uniform then f (x+1) (p*rate/x) (s+p) else x
            in f 0 p p
\end{code}
    %Note that Poisson is a discrete distribution but can be specified using continuous distribution.
  \end{compactenum}
\end{example}

\begin{example}
  \label{example: discrete from coin}
  %As discussed in \cref{remark:samplers}, 
  It might be beneficial for some inference algorithm if discrete distributions are specified using discrete random variables. Hence, we show how different discrete distributions can be specified by our discrete sampler \codeinline{coin}.
  \begin{compactenum}[(1)]
    \item The Bernoulli distribution with probability $p \in [0,1] \cap \Dyadic$, where $\Dyadic := \set{\frac{n}{2^{m}}\mid n,m \in \Nat} $ is the set of all Dyadic numbers,
    can be specified by
    \begin{code}
let rec bern(p) = if p = 0 then False else
                  if p = 1 then True else
                  if p < 0.5 then
                    if coin then bern(2*p) else False
                  else
                    if coin then True else bern(2*(p-0.5))
    \end{code}

    \item The geometric distribution with rate $p \in [0,1] \cap \Dyadic$ can be specified by 
    %\codeinline{geo(p) = count = 0; while bern(1-p): count += 1; return count}.
    \begin{code}
geo(p) = count = 0; while bern(1-p): count += 1; return count
    \end{code}

    \item The binomial distribution with $n \in \Nat$ trails and probability $p \in [0,1] \cap \Dyadic$ can be specified by \codeinline{bin(n,p) = sum([1 for i in range(n) if bern(p)])}.

    \item Let \codeinline{pdfPio} and \codeinline{pdfgeo} be the pdfs of the Poisson and geometric distributions respectively. Then, the Poisson distribution can be described by
    \begin{code}
Poi(rate) = n = geo(0.5);
            score(pdfPoi(rate,n)/pdfgeo(0.5,n));
            return n
    \end{code}
  \end{compactenum}
\end{example}

\subsection{Church Encodings}
\label{sec: church encodings}

We can represent pairs and lists in SPCF using Church encoding as follows:
\begin{align*}
  \typair{\typea}{\typeb} & := \typea \to \typeb \to (\typea \to \typeb \to \tyreal) \to \tyreal &
  \tylist{\typea} & := (\typea \to \tyreal \to \tyreal) \to (\tyreal \to \tyreal) \\
  \anbr{\terma,\termb} & \equiv \lambda z .z\,\terma\,\termb &
  [\terma_1,\dots,\terma_\ell] & \equiv \lambda f x .f\,\terma_1(f\,\terma_2 \dots (f\,\terma_\ell\,\pcf{0}))
\end{align*}
Moreover standard primitives on pairs and lists, such as \codeinline{projection}, \codeinline{len}, \codeinline{append} and \codeinline{sum}, can be defined easily.

\subsection{Operational Semantics}
\label{subsec:operational semantics}

\subsubsection{Trace Space}
\label{sec:trace space}

Since $\normal$ samples from the standard normal distribution $\Gau$ and $\coin$ from the Bernoulli distribution $\Bern(0.5)$, the \defn{sample space} of SPCF is the union of the measurable spaces of $\Real$ and $\Bool$. Formally it is the measurable space with set $\samspace := \Real \cup \Bool$, $\sigma$-algebra $\algebra{\samspace} := \set{V \cup W \mid V \in \Borel, W \in \algebra{\Bool}}$ and measure $\measure{\samspace}(V \cup W) := \Gau(V) + \Bern(0.5)(W)$.
We denote the product of $n$ copies of the sample space as $(\samspace^n, \algebra{\samspace^n}, \measure{\samspace^n})$ and call it the $n$-dimensional sample space.

A \defn{trace} is a record of the values sampled in the course of an execution of a SPCF term.
Hence, the \defn{trace space} is the union of sample spaces of varying dimension.
Formally it is the measurable space with set $\traces := \bigcup_{n\in\Nat} \samspace^n$,
$\sigma$-algebra $\algebra{\traces} := \set{\bigcup_{n\in\Nat} U_n \mid U_n \in \algebra{\samspace^n}}$ and
measure $\measure{\traces}(\bigcup_{n\in\Nat} U_n) = \sum_{n\in\Nat} \measure{\samspace^n}(U_n)$.
We present traces as lists, e.g.~$[-0.2,\true, \true, 3.1, \false]$ and $\emptytrace$.

\begin{remark}
  Another way of recording the sampled value in a run of a SPCF term is to have separate records for the values of the continuous and discrete samples.
  In this case, the trace space will be the set $\bigcup_{n\in\Nat} \Real^n \times \bigcup_{m\in\Nat} \Bool^m$.
  We find separating the continuous and discrete samples unnecessarily complex for our purposes and hence follow the more conventional definition of trace space.
\end{remark}

\subsubsection{Small-step Reduction}

\begin{figure}[t!]
  \defn{Values} (typically denoted $\valuea$),
  \defn{redexes} (typically $\redexa$) and
  \defn{evaluation contexts} (typically $\evalcon$):
  \begin{align*}
    \valuea & ::=
    \pcf{\reala} \mid \pcf{\boola} \mid \lambda y.\terma \\
    \redexa & ::=
              \pcf{f}(\pcf{\constant_1},\dots,\pcf{\constant_\ell}) \mid
              (\lambda y.\terma)\,\valuea \mid
              \pcfif{\pcf{\boola}}{\terma}{\termb} \mid
              \Y{(\lambda y.\terma)}\\
            & \mid
              \normal \mid \coin \mid
              \score{\pcf{\reala}}\\
    \evalcon & ::=
               [] \mid
               \evalcon\,\terma \mid
               (\lambda y.\terma)\,\evalcon \mid
               \pcfif{\evalcon}{\terma}{\termb}\mid
               \pcf{f}(\pcf{\constant_1}, \dots, \pcf{\constant_{i-1}},\evalcon,\terma_{i+1},\dots,\terma_\ell) \mid
               \Y{\evalcon} \\
            & \mid
              \score{\evalcon}
  \end{align*}
  \noindent\defn{Redex contractions}:
  \begin{align*}
    \config{\pcf{f}(\pcf{\constant_1},\dots,\pcf{\constant_\ell})}{w}{\traceb} & \red
      \begin{cases}
        \config{\pcf{f(\constant_1,\dots,\constant_{\ell})}}{w}{\traceb}
        & \text{if } (\constant_1,\dots,\constant_{\ell}) \in \domain{f},\\
        \Fail & \text{otherwise}
      \end{cases}
      \\
    \config{(\lambda y.\terma)\,\valuea}{w}{\traceb} & \red
      \config{\terma[\valuea/y]}{w}{\traceb} \\
    \config{\pcfif{\pcf{\boola}}{\terma}{\termb}}{w}{\traceb} & \red
      \begin{cases}
        \config{\terma}{w}{\traceb}
        & \text{if } \boola,\\
        \config{\termb}{w}{\traceb} & \text{otherwise}
      \end{cases}\\
    \config{\Y{(\lambda y.\terma)}}{w}{\traceb} & \red
      \config{\lambda z.\terma[\Y{(\lambda y.\terma)}/y]\,z}{w}{\traceb}
      \tag{for fresh variable $z$} \\
    \config{\normal}{w}{\traceb} & \red
      \config{\pcf{\reala}}{w}{\trace\concat [\reala]} \tag{for some $\reala \in \Real$} \\
    \config{\coin}{w}{\traceb} & \red
      \config{\pcf{\boola}}{w}{\trace\concat [\boola]} \tag{for some $\boola \in \Bool$} \\
    \config{\score{\pcf{\reala}}}{w}{\traceb} & \red
      \begin{cases}
        \config{\pcf{\reala}}{\reala\cdot w}{\traceb}
        & \text{if } \reala > 0,\\
        \Fail & \text{otherwise.}
      \end{cases}
  \end{align*}
  \noindent\defn{Evaluation contexts}:
  $$
    \AxiomC{$\config{\redexa}{w}{\traceb} \red \config{\contra}{w'}{\trace'}$}
    \UnaryInfC{$\config{E[\redexa]}{w}{\traceb} \red \config{E[\contra]}{w'}{\trace'}$}
    \DisplayProof
    \qquad
    \AxiomC{$\config{\redexa}{w}{\traceb} \red\Fail$}
    \UnaryInfC{$\config{E[\redexa]}{w}{\traceb} \red\Fail$}
    \DisplayProof
  $$
  \caption{
    Small-step reduction of SPCF, where
    $\reala, \realb, \realc \in \Real$,
    $\boola, \boolb \in \Bool$,
    $\constant \in \Real \cup \Bool$,
    $x,y,z$ are variables, and
    $\funca, \funcb, \funcc$ ranges over the set $\primitives$ of primitive functions.
  }
  \label{fig:operational small-step}
\end{figure}

The small-step reduction of SPCF terms can be seen as a rewrite system of \defn{configurations}, which are triples of the form $\config{\terma}{w}{\traceb}$ where $\terma$ is a closed SPCF term, $w > 0$ is a \defn{weight}, and $\traceb \in \traces$ a trace, as defined in \cref{fig:operational small-step}.

In the rule for $\normal$, a random value $\reala \in \Real$ is generated and recorded in the trace, while the weight remains unchanged:
even though the program samples from a normal distribution, the weight does not factor in Gaussian densities as they are already accounted for by $\measure{\traces}$.
Similarly, in the rule for $\coin$, a random boolean $\boola \in \Bool$ is sampled and recorded in the trace with an unchanged weight.
In the rule for $\score{\pcf{\reala}}$, the current weight is multiplied by $\reala\in\Real$: typically this reflects the likelihood of the current execution given some observed data.
Similar to \cite{DBLP:conf/icfp/BorgstromLGS16} we reduce terms which cannot be reduced in a reasonable way (i.e.~scoring with nonpositive constants or evaluating functions outside their domain) to $\Fail$.

We write $\redplus$ for the transitive closure and
$\redstar$ for the reflexive and transitive closure of the small-step reduction.

\subsubsection{Value and Weight Functions}
\label{appendix: value and weight functions}

Following \cite{DBLP:conf/icfp/BorgstromLGS16},
we view the set $\terms$ of all SPCF terms as $\bigcup_{n,m\in\Nat} (\sk_{n,m} \times \Real^n \times \Bool^m)$
where $\sk_{n,m}$ is the set of SPCF terms with exactly $n$ real-valued and $m$ boolean-valued place-holders.
The measurable space of terms
is equipped with the $\sigma$-algebra $\algebra{\terms}$ that is the Borel algebra of the
countable disjoint union topology of the product topology of
the discrete topology on $\sk_{n,m}$, the standard topology on $\Real^n$ and the discrete topology on $\Bool^m$.
Similarly the subspace $\closedvalues$ of closed values inherits the Borel algebra on $\terms$.

% Let $\terma$ be an SPCF term with free variables amongst $x_1,\dots,x_m$ of type $\tyreal$. Its \defn{value function} $\valuefn_{\terma} : \Real^m \times \traces \to \closedvalues \cup \{\bot\}$ returns, given values for each free variable and a trace, the output value of the program, if the program terminates in a value. The \defn{weight function} $\weightfn_{\terma} : \Real^m \times \traces \to \pReal$ returns the final weight of the corresponding execution. Formally:
Let $\terma$ be a closed SPCF term.
Its \defn{value function} $\valuefn{\terma} : \traces \to \closedvalues \cup \{\bot\}$ returns, given a trace, the output value of the program, if the program terminates in a value. Its \defn{weight function} $\weightfn{\terma} : \traces \to \pReal$ returns the final weight of the corresponding execution. Formally:
%\iffalse
\[
\begin{array}{c}
  \valuefn{\terma} (\trace) :=
  \begin{cases}
  V & \hbox{if $\config{\terma}{1}{\emptytrace} \redstar
      \config{\valuea}{w}{\traceb}$}\\
  \bot & \text{otherwise.}
  \end{cases}
\end{array}
\quad
\begin{array}{c}
  \weightfn{\terma} (\trace) :=
  \begin{cases}
  w & \hbox{if
    $\config{\terma}{1}{\emptytrace} \redstar \config{\valuea}{w}{\traceb}$}\\
  0 & \text{otherwise.}
  \end{cases}
\end{array}
\]
%\fi
\iffalse
\begin{align*}
  \valuefn{\terma} (\trace) & :=
  \begin{cases}
  V & \hbox{if $\config{\terma}{1}{\emptytrace} \redstar
      \config{\valuea}{w}{\traceb}$}\\
  \bot & \text{otherwise.}
  \end{cases}
\end{align*}
\begin{align*}
  \weightfn{\terma} (\trace) & :=
  \begin{cases}
  w & \hbox{if
    $\config{\terma}{1}{\emptytrace} \redstar \config{\valuea}{w}{\traceb}$}\\
  0 & \text{otherwise.}
  \end{cases}
\end{align*}
\fi
It follows readily from \cite{DBLP:conf/icfp/BorgstromLGS16} that the functions $\valuefn{\terma}$ and $\weightfn{\terma}$ are measurable.

Finally, every closed SPCF term $\terma$ has an associated \defn{value measure} $\oper{\terma}$ on $\closedvalues$ given by
\begin{align*}
  {\oper{\terma}}: {\Sigma_{\closedvalues}} &\longrightarrow{\pReal}\\
   U & \longmapsto
    \shortint{\inv{\valuefn{\terma}}(U)}
    {\weightfn{\terma}}
    {\measure{\traces}}
\end{align*}

\begin{remark}
  \label{remark: support of weight and value}
  A trace is in the support of the weight function if and only if the value function returns a (closed) value when given this trace.
  i.e.~$\support{\weightfn{\terma}} = \inv{\valuefn{\terma}}(\closedterms)$ for all closed SPCF term $\terma$.
\end{remark}

\begin{remark}
  The weight function defined here is the density of the target distribution from which an inference algorithm typically samples.
  In this work, we call it the \emph{weight function} when considering semantics following \cite{DBLP:conf/esop/CulpepperC17,VakarKS19,MakOPW21},
  and call it \defn{density function} when discussing inference algorithms following \cite{DBLP:conf/aistats/ZhouGKRYW19,DBLP:conf/icml/ZhouYTR20,cusumanotowner2020automating}.
\end{remark}

\subsection{Tree Representable Functions}

We consider a necessary condition for the weight function of closed SPCF terms which would help us in designing inference algorithms for them.
Note that not every function of type $\traces \to \pReal$ makes sense as a weight function.
Consider the program \codeinline{let rec f x = if coin then f(x+normal) else x in f 0} in \cref{example: spcf terms}.
This program executes successfully with the trace $[\true, 0.5, \false]$.
This immediately tells us that upon sampling $\true$ and $0.5$, there must be a sample following them, and this third sample must be a boolean.
In other words, the program does \emph{not} terminate with any proper prefix of $[\true, 0.5, \false]$ such as $[\true, 0.5]$, nor any traces of the form $[\true, 0.5, \reala]$ for $\reala\in\Real$.

Hence, we consider measurable functions $\tree: \traces \to \pReal$ satisfying
\begin{itemize}
  \item \defn{prefix property}:
  whenever $\traceb \in \nsupport{\tree}{n}$\footnote{$\nsupport{\tree}{n} := \support{\tree} \cap \samspace^n$ for all $n\in\Nat$.} then for all $k < n$,
  we have $\trace^{1\dots k} \not\in \nsupport{\tree}{k}$; 
  %\lo{Have we defined $\nsupport{\tree}{n}$ (with superscript $n$)? $\nsupport{\tree}{n} := \support{\tree} \cap \samspace^n$.}
  and

  \item \defn{type property}:
  whenever $\traceb \in \nsupport{\tree}{n}$ then for all $k < n$ and for all $\trace \in \samspace\setminus \type{\seqindex{\traceb}{k+1}}$\footnote{The \defn{type} $\type{\trace}$ of a sample $\trace \in \samspace$ is $\Real$ if $\trace \in \Real$ and is $\Bool$ if $\trace \in \Bool$. } we have $\traceb^{1\dots k} \concat [\trace] \not\in\nsupport{\tree}{k+1}$.
\end{itemize}
They are called \defn{tree representable (TR) functions} \cite{DBLP:conf/icml/MakZO21} because any such function $\tree$ can be represented as a (possibly) infinite but finitely branching tree, which we call \emph{program tree}.

This is exemplified in \cref{fig:binary-tree} (left), where
a hexagon node denotes an element of the input of type $\samspace$;
a triangular node gives the condition for $\traceb \in \nsupport{\tree}{n}$
(with the left, but not the right, child satisfying the condition); and
a leaf node gives the result of the function on that branch.
Any \defn{branch} (i.e.~path from root to leaf) in a program tree of $\tree$ represents a set of finite sequences $[\trace_1,\dots,\trace_n]$ in $\support{\tree}$.
In fact, every program tree of a TR function $\tree$ specifies a countable partition of $\support{\tree}$ via its branches.
The prefix property guarantees that for each TR function $\tree$, there are program trees of the form in \cref{fig:binary-tree} representing $\tree$.

The program tree of $\terma$ is depicted in \cref{fig:binary-tree} (right), where
a circular node denotes a real-valued input and
a squared node denotes a boolean-valued input.
\cm{Draw a better tree}

\tikzset{position/.style={blue,font=\footnotesize}}
\tikzset{path/.style={draw=blue, line width=3pt}}
\tikzset{not path/.style={draw=black, line width=0.5pt}}
\tikzset{input/.style={regular polygon,regular polygon sides=6,draw=black, line width=0.5pt,minimum size=12pt,font=\footnotesize},inner sep=2pt}
\tikzset{normal/.style={circle,draw=black, line width=0.5pt,minimum size=12pt,font=\footnotesize},inner sep=2pt}
\tikzset{coin/.style={regular polygon,regular polygon sides=4,draw=black, line width=0.5pt,minimum size=12pt,font=\footnotesize},inner sep=2pt}
\tikzset{condition/.style={isosceles triangle,isosceles triangle apex angle=60,rotate=90, draw=black, line width=0.5pt,minimum size=12pt,font=\footnotesize}}
\tikzset{leaf/.style={draw=none,font=\footnotesize}}
\tikzset{yes-no/.style={black!50!green}}

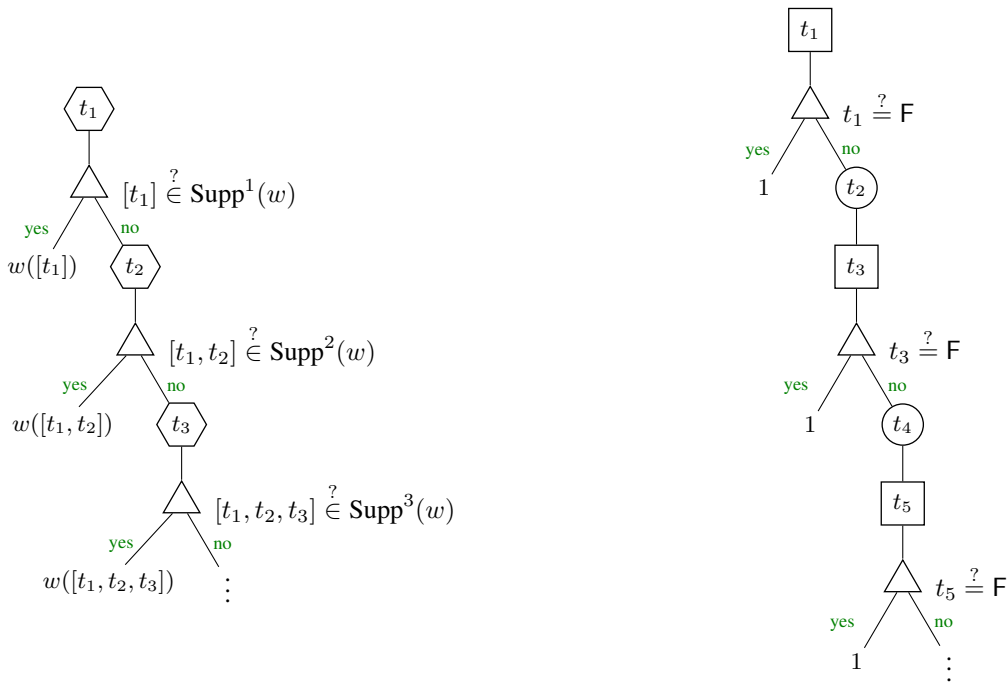
\begin{figure}[t]
  \parbox{.5\textwidth}{
    \centering
    \begin{tikzpicture}[scale=0.7,sibling distance = 5em]
      \node [input] (q1) {$\trace_1$}
        child {
          node [condition,label={[label distance=0.5em]270:$[\trace_1] \overset{?}{\in} \nsupport{\tree}{1}$}] (c1) {}
          child {node [leaf] (r1) {$\tree([\trace_1])$}}
          child {
            node [input] (q2) {$\trace_2$}
            child {
              node [condition,label={[label distance=0.5em]270:$[\trace_1,\trace_2] \overset{?}{\in} \nsupport{\tree}{2}$}] (c2) {}
              child {node [leaf,xshift=-1em] (r2) {$\tree([\trace_1,\trace_2])$}}
              child {
                node [input] (q3) {$\trace_3$}
                child {
                  node [condition,label={[label distance=0.5em]270:$[\trace_1,\trace_2,\trace_3] \overset{?}{\in} \nsupport{\tree}{3}$}] (c3) {}
                  child {node [leaf,xshift=-1em] (r3) {$\tree([\trace_1,\trace_2,\trace_3])$}}
                  child {node {$\vdots$}}
                }
              }
            }
          }
        };
        \node at ($(c1) + (220:1.3cm) $) [yes-no] {\scriptsize yes};
        \node at ($(c1) + (315:1.1cm) $) [yes-no] {\scriptsize no};
        \node at ($(c2) + (215:1.4cm) $) [yes-no] {\scriptsize yes};
        \node at ($(c2) + (315:1.1cm) $) [yes-no] {\scriptsize no};
        \node at ($(c3) + (215:1.4cm) $) [yes-no] {\scriptsize yes};
        \node at ($(c3) + (315:1.1cm) $) [yes-no] {\scriptsize no};
    \end{tikzpicture}
  }%
  \parbox{.5\textwidth}{
    \centering
    \begin{tikzpicture}[scale=0.7,sibling distance = 5em, node distance=0.5em]
      \node [coin] (q1) {$\trace_1$}
        child {
          node [condition,label={[label distance=0.5em]270:{$\trace_1 \overset{?}{=} \false$}}] (c1) {}
          child {node [leaf] (r1) {$1$}}
          child {
            node [normal] (q2) {$\trace_2$}
            child {node [coin] (q3) {$\trace_3$}
            child {
              node [condition,label={[label distance=0.5em]270:{$\trace_3 \overset{?}{=} \false$}}] (c2) {}
              child {node [leaf] (r2) {$1$}}
              child {
                node [normal] (q4) {$\trace_4$}
                child {node [coin] (q5) {$\trace_5$}
                child {
                  node [condition,label={[label distance=0.5em]270:{$\trace_5 \overset{?}{=} \false$}}] (c3) {}
                  child {node [leaf] (r3) {$1$}}
                  child {node {$\vdots$}}
                }}
              }
            }}
          }
        };
        \node at ($(c1) + (220:1.3cm) $) [yes-no] {\scriptsize yes};
        \node at ($(c1) + (315:1.1cm) $) [yes-no] {\scriptsize no};
        \node at ($(c2) + (215:1.4cm) $) [yes-no] {\scriptsize yes};
        \node at ($(c2) + (315:1.1cm) $) [yes-no] {\scriptsize no};
        \node at ($(c3) + (215:1.4cm) $) [yes-no] {\scriptsize yes};
        \node at ($(c3) + (315:1.1cm) $) [yes-no] {\scriptsize no};
    \end{tikzpicture}
  }%
  \caption{Program tree of a tree representable function $\tree$}
  \label{fig:binary-tree}
\end{figure}

The following proposition ties SPCF terms and TR functions together.

\begin{restatable}{proposition}{spcfTR}
  \label{prop: all spcf terms have TR weight function}
  Every closed SPCF term has a tree representable weight function.
\end{restatable}

% \begin{proposition}
%   \label{prop: all spcf terms have TR weight function}
%   Every closed SPCF term has a tree representable weight function.
% \end{proposition}

We will see in \cref{sec: np-imcmc} how the TR functions, in particular the prefix property, is instrumental in the design of the inference algorithm.

\subsection{Almost Sure Termination and Integrability}
\label{sec: spcf ast and integrability}

\begin{definition}\rm
  \label{def:ast}
  We say a SPCF term $\terma$ \defn{terminates almost surely} if
  $\terma$ is closed and
  $\measure{\traces}(\set{\traceb \in \traces \mid \exists \valuea, w \,.\,
  \config{\terma}{1}{\emptytrace} \redstar \config{\valuea}{w}{\traceb}}) =1$.

  We denote the set of terminating traces as $\tertraces := \set{\traceb \in \traces \mid \exists \valuea, w \,.\, \config{\terma}{1}{\emptytrace} \redstar \config{\valuea}{w}{\traceb}}$.
\end{definition}

\begin{remark}
  \label{remark: AST}
  The set of traces on which a closed SPCF term $\terma$ terminates,
  i.e.~$\set{\traceb \in \traces \mid \exists V, w \,.\,
  \config{\terma}{1}{\emptytrace} \redstar \config{\valuea}{w}{\traceb}}$, can be understood as the support of its weight function $\support{\weightfn{\terma}}$, or as discussed in \cref{remark: support of weight and value}, the traces on which the value function returns a value, i.e.~$\inv{\valuefn{\terma}}(\closedvalues)$.
  Hence, $\terma$ almost surely terminates if and only if $\measure{\traces}(\support{\weightfn{\terma}}) =\measure{\traces}(\inv{\valuefn{\terma}}(\closedvalues)) = 1$.
\end{remark}

\begin{definition}
  \label{def: term and max}
  Following \cite{MakOPW21}, we say a trace $\traceb\in\traces$ is \defn{maximal} w.r.t.~a closed term $\terma$ if
  there exists a term $\termb$, weight $w$ where
  $\config{\terma}{1}{\emptytrace} \redstar \config{\termb}{w}{\traceb}$ and
  for all $\traceb' \in \traces \setminus\set{\emptytrace} $ and all terms $\termb'$,
  $\config{\termb}{w}{\traceb} \not\redstar \config{\termb'}{w'}{\traceb\concat\traceb'}$.

  We denote the set of maximal traces as $\maxtraces$.
\end{definition}

\begin{proposition}[\cite{MakOPW21}, Lemma 9]
  \label{prop: AST}
  A closed term $\terma$ is almost surely terminating if $\measure{\traces}(\maxtraces \setminus \tertraces) = 0 $.
\end{proposition}

%(The following proposition is used to prove \cref{prop: step 1 is probabilistic}, which is used in the correctness proof of the NP-iMCMC algorithm (\cref{sec: correctness}).)

\begin{restatable}{proposition}{asttermprob}
  \label{prop: AST SPCF term gives probability measure}
  The value measure $\oper{\terma}$ of
  a closed almost surely terminating SPCF term $\terma$ which does not contain $\score{\placeholder}$ as a subterm
  is probabilistic.
\end{restatable}
% \begin{proposition}
%   \label{prop: AST SPCF term gives probability measure}
%   The value measure $\oper{\terma}$ of
%   a closed almost surely terminating SPCF term $\terma$ which does not contain $\score{\placeholder}$ as a subterm
%   is probabilistic.
% \end{proposition}

\iffalse
\begin{remark}
  This work relies on the verification works \cm{cite} done by the community in ensuring almost sure termination.
\end{remark}
\fi

\begin{definition}\rm
  \label{def:integrable}
  We say a SPCF term $\terma$ is \defn{integrable} if
  $\terma$ is closed and
  its value measure is finite,
  i.e.~$\oper{M}(\closedvalues) < \infty$;
\end{definition}

\begin{restatable}{proposition}{integrable}
  \label{prop: integrable term}
  An integrable term has an integrable weight function.
\end{restatable}

% \begin{proposition}
%   \label{prop: integrable term}
%   An integrable term has an integrable weight function.
% \end{proposition}

\begin{example}
  Now we look at a few examples in which we show that almost surely termination and integrability identify two distinct sets of SPCF terms.
  \begin{compactenum}[(1)]

    \item The term $\terma_1$ defined as
    % \codeinline{count = 0; while coin: count += 1; return score(2**count)}
    \codeinline{let rec f x = if coin then f (x+1) else x in score(2**(f 0))}
    almost surely terminates since it only diverges on the infinite trace $[\false, \false, \dots]$ which has zero probability.
    However, it is \emph{not} integrable as the value measure applied to all closed values
    $\oper{\terma_1}(\closedvalues)
    % = \shortint{\inv{\valuefn{\terma_1}}(\closedvalues)}{\weightfn{\terma_1}}{\measure{\traces}}
    = \shortint{\set{[\true, \dots, \true, \false]}}{\weightfn{\terma_1}}{\measure{\traces}}
    = \sum_{n=0}^\infty \shortint{\set{[T]^n\concat [\false]}}{\weightfn{\terma_1}}{\measure{\Bool^n}}\footnote{We write $[x]^n$ to be the list that contains $n$ copies of $x$.}
    = \sum_{n=0}^\infty (\frac{1}{2})^{n+1} \cdot 2^n = \sum_{n=0}^\infty \frac{1}{2}$
    is infinite.

    \item Consider the term $\terma_2$ defined as \codeinline{if coin then Y (lambda x:x) 0 else 1}. Since it reduces to a diverging term, namely \codeinline{Y (lambda x:x) 0}, with non-zero probability, it does not terminate almost surely. However, it is integrable, since
    $\oper{\terma_2}(\closedvalues)
    = \shortint{\set{[\false]}}{\weightfn{\terma_2}}{\measure{\traces}} = \frac{1}{2} < \infty$.

    \item The term $\terma_3$ defined as $\pcfif{\coin}{\terma_1}{\terma_2}$ is neither almost surely terminating nor integrable, since $\terma_1$ is not integrable and $\terma_2$ is not almost surely terminating.

    \item All terms considered previously in \cref{example: spcf terms,example: not restrictive,example: discrete from coin} are both almost surely terminating and integrable.

  \end{compactenum}
\end{example}
\clearpage

\section{Hybrid Nonparametric Involutive MCMC and its Correctness}
\label{app: hybrid np-imcmc}
% !TEX root = ./../icml2022.tex

In this section, we present the
\defn{Hybrid Nonparametric Involutive Markov chain Monte Carlo} (Hybrid NP-iMCMC),
an inference algorithm
that simulates the probabilistic model
specified by a given SPCF program
that may contains both discrete and
continuous samplers.
% For ease of reference, we
% write NP-iMCMC to mean
% the Hybrid NP-iMCMC sampler.

To start,
we detail the Hybrid NP-iMCMC inference algorithm:
its state space, conditions on the inputs and
steps to generate
the next sample; and
study how the sampler
moves between states of varying dimensions
and
returns new samples of a nonparametric
probabilistic program.
We then give an implementation of Hybrid NP-iMCMC in SPCF
and demonstrate how
the Hybrid NP-iMCMC method extends the MH sampler. Last but not least,
we conclude with a discussion on
the correctness of Hybrid NP-iMCMC.

\subsection{State Spaces}
\label{sec: spaces in np-imcmc}

A \emph{state} in the
Hybrid NP-iMCMC algorithm
is a pair $(\enta,\auxa)$ of equal \emph{dimension}
(but not necessarily equal length)
\emph{parameter} and \emph{auxiliary} variables.
The parameter variable $\enta$ is used to
store traces and
the auxiliary variable $\auxa$ is used to
record randomness.
Both variables are
vectors of \emph{entropies},
i.e.~Real-Boolean pairs.
This section gives the formal definitions of
the entropy, parameter and auxiliary variables
and the state,
in preparation for the discussion of the Hybrid NP-iMCMC
sampler.

\subsubsection{Entropy Space}
\label{sec: entropy space}

As shown in \cref{app: spcf},
the reduction of a SPCF program is determined by
the input trace $\traceb \in
\traces := \bigcup_{n\in\Nat} (\Real \cup \Bool)^n$,
a record of drawn values in a particular run of the program.
Hence in order
to simulate a probabilistic model described by a SPCF program,
the Hybrid NP-iMCMC sampler
should generate Markov chains on the trace space.
However traversing through the trace space
is a delicate business because
the positions and numbers of discrete and continuous values in
a trace given by a SPCF program
may vary.
(Consider \codein{if coin: normal else: coin}.)

Instead,
we pair each value $\seqindex{\traceb}{i}$ in a trace $\traceb$
with a random value $\trace$ of the other type
to make a Real-Boolean pair
$(\seqindex{\traceb}{i},\trace)$ (or $(\trace, \seqindex{\traceb}{i})$).
For instance, the trace $[\true, -3.1]$ can be made into a Real-Boolean vector $[(1.5,\true),(-3.1,\true)]$
with randomly drawn values $1.5$ and $\true$.
Now, the position of discrete and continuous random variables
does not matter and
the number of discrete and continuous random variables are fixed in each vector.

We call a Real-Boolean pair
an \defn{entropy} and define
the \defn{entropy space} $\entsp$ to be
the product space
$\Real \times \Bool$ of
the Borel measurable space and
the Boolean measurable space,
equipped with the $\sigma$-algebra
$\algebra{\entsp} := \sigma\big(
\set{\realset \times \boolset \mid
\realset \in \Borel, \boolset \in \algebra{\Bool}}
\big)$,
and the product measure
$\measure{\entsp} := {\Gau} \times {\measure{\Bool}}$
where ${\measure{\Bool}} := \Bern(0.5)$.
Note the Radon-Nikodym derivative
$\entpdf{}$ of $\measure{\entsp}$
can be defined as
$\entpdf{}(\reala, \boola) := \frac{1}{2}\pdfGau(\reala)$.
A $n$-length \defn{entropy vector} is then
a vector of $n$ entropies, formally
an element in the product measurable space
$({\entsp^n}, \algebra{\entsp^n})$.
We write $\len{\enta}$ to mean the
length of the entropy vector $\enta$.
% The product of $n$ copies of entropy space is
% the \defn{$n$-dimensional entropy space} and denoted as
% $(\nparsp{n}, \algebra{\nparsp{n}}, \measure{\nparsp{n}})$.

As mentioned earlier, the parameter variable of
a state
is an entropy vector that stores traces.
Hence, it would be useless if
a unique trace cannot be restored
from an entropy vector.
We found that such a recovery is possible if
the trace is in the support of a
tree representable function.

Say we would like to recover the trace $\hat{\traceb}$
that is used to form the entropy vector $\enta$
by pairing each value in the trace with a random value
of the other type.
First we realise that
traces can be made
by selecting either the Real or Boolean component
of each pair in a prefix of $\enta$.
For example,
traces like $\emptytrace$, $[\true]$, $[-0.2]$,
$[\true, 2.9]$ and $[-0.2,\true, \false]$ can be made
from the entropy vector $[(-0.2,\true), (2.9, \true), (1.3, \false)]$.
We call these traces \defn{instances} of the entropy vector.
Formally, a trace $\traceb \in \traces$ is an instance of an entropy vector
$\enta \in \entsp^{n}$ if
$\len{\traceb} \leq n$ and
$\seqindex{\traceb}{i} \in \set{\reala,\boola \mid (\reala,\boola) = \seqindex{\enta}{i}}$ for all $i = 1,\dots, \len{\traceb}$.
We denote the set of all instances of $\enta$ as
$\instances{\enta} \subseteq \traces$.
Then, the trace $\hat{\traceb}$ must be an instance
of $\enta$.
Moreover, if we can further assume that
$\hat{\traceb}$ is in the support of a tree representable function,
then \cref{prop: instances and tree representable}
says
we can uniquely identify $\hat{\traceb}$ amongst
all instances of $\enta$.

\begin{proposition}
  \label{prop: instances and tree representable}
  There is at most one (unique) trace
  that is both
  an instance of an entropy vector
  and
  in the support of a tree representable function.
\end{proposition}

% The trace we started with must be an instance of $\enta$.
% Since $\tree$ is tree representable,
% by \cref{prop: instances and tree representable},
% \emph{at most one} trace $\traceb \in \instances{\enta}$
% is in the support of a tree representable function $\tree$.
% Hence, even though the transformation from trace to entropy vector
% can have infinite possibilities,
% we can always recover the trace that is in the support of
% a tree representable function.

Finally, we consider differentiability on the
multi-dimensional entropy space.
We say a function $f: \entsp^{k_1} \to \entsp^{k_2}$
is \defn{differentiable almost everywhere} if
for all $\boolveca \in \Bool^{k_1}$, $\boolvecb \in \Bool^{k_2} $,
the partial function
$\fij{\boolveca}{\boolvecb}:\Real^{k_1} \to \Real^{k_2} $
where
\[
  \fij{\boolveca}{\boolvecb}(\realveca) = \realvecb
  \qquad\iff\qquad
  f(\zip(\realveca, \boolveca)) = (\zip(\realvecb, \boolvecb))
  \footnote{
  We write $\zip(\ell_1, \ell_2) $ to be the $n$-length vector
  $[(\seqindex{\ell_1}{1}, \seqindex{\ell_2}{1}),
  (\seqindex{\ell_1}{2}, \seqindex{\ell_2}{2}), \dots,
  (\seqindex{\ell_1}{n}, \seqindex{\ell_2}{n})] \in (L_1\times L_2)^n $
  for any vectors $\ell_1 \in L_1^{n_1}$ and
  $\ell_2\in L_2^{n_2}$ with $n := \min\{n_1, n_2\}$.
  }
\]
is differentiable \emph{almost everywhere} on its domain
$\domain{\fij{\boolveca}{\boolvecb}} :=
\set{\realveca \in \Real^{k_1}\mid
\exists \realvecb \in \Real^{k_2} \ .\
f(\zip(\realveca, \boolveca)) = (\zip(\realvecb, \boolvecb))} $.
%  on the set
% $\set{ (\realveca, \auxa) \in \Real^n \times \nauxsp{n} \mid
% \ninvo{n}(\zip(\realveca, \boolveca), \auxa) = (\zip(\realvecb, \boolvecb), \auxb)}$.
The Jacobian
of $f$ on
$(\zip(\realveca, \boolveca))$
is given by
$\grad{\fij{\boolveca}{\boolvecb}}(\realveca)$,
if it exists.

\subsubsection{Parameter Space}

A parameter variable $\enta$ of dimension $n$
is an entropy vector of length $\iparsp(n)$
where $\iparsp:\Nat\to\Nat$ is a strictly monotone map.
For instance,
the parameter variable
$\enta := [(-0.2,\true), (2.9, \true), (1.3, \false)]$
is of dimension two if $\iparsp(n):=n+1$,
and dimension three if $\iparsp(n):=n$.
We write $\dim(\enta)$ to mean the dimension of $\enta$ and
$\len{\enta}$ to mean the length of $\enta$.
Hence,
$\dim(\enta) \leq \len{\enta}$ and
$\iparsp(\dim(\enta)) = \len{\enta} $.
We extend the notion of dimension to traces and
say a trace $\traceb \in \traces$ has dimension
$n$ if $\len{\traceb} = \iparsp(n)$.
Importantly, we assume that
every trace in the support of $\tree$
has a dimension (w.r.t.~$\iparsp$),
i.e.~$\support{\tree} = \bigcup_{n\in\Nat}
\nsupport{\tree}{\iparsp{(n)}}$.

% We define the \defn{parameter space}
% $(\parsp, \algebra{\parsp})$ to be the entropy space
% with the entropy measure as
% the parameter space base measure
% $\measure{\parsp} := \measure{\entsp}$.
% Then, the Radon-Nikodym derivative $\parpdf$ of the parameter space
% base measure is simply $\entpdf{}$.
% Moreover, $\parsp$ is a
% $\spdim{\parsp}$-dimensional
% differential manifold,
% equipped with an atlas.
Formally,
the $n$-dimensional parameter space
$(\nparsp{n}, \algebra{\nparsp{n}})$
is the product of $\iparsp(n)$ copies of the entropy space $(\entsp, \algebra{\entsp})$
and
the base measure $\measure{\nparsp{n}}$ on $\nparsp{n}$
is the product of $\iparsp(n)$ copies of the entropy measure $\measure{\entsp}$
with the Radon-Nikodym derivative $\nparpdf{n}$.
% $\nparsp{n}$ is a
% $(n\cdot\spdim{\parsp})$-dimensional
% differential manifold,
% equipped with the product atlas.
For ease of reference, we write
$(\parsp, \algebra{\parsp}, \measure{\parsp})$
for the one-dimensional parameter space.

\subsubsection{Auxiliary Space}

Similarly, an auxiliary variable $\auxa$ of dimension $n$
is an entropy vector of length $\iauxsp(n)$
where $\iauxsp:\Nat\to\Nat$ is a strictly monotone map.
% we define the \defn{auxiliary space}
% $(\auxsp, \algebra{\auxsp})$ to be the entropy space
% with the entropy measure as
% the auxiliary space base measure
% $\measure{\auxsp} := \measure{\entsp}$.
% Then, the Radon-Nikodym derivative $\auxpdf$ of the auxiliary space
% base measure is simply $\entpdf{}$.
% Moreover, $\auxsp$ is a
% $\spdim{\auxsp}$-dimensional
% differential manifold,
% equipped with an atlas.
The $n$-dimensional auxiliary space
$(\nauxsp{n}, \algebra{\nauxsp{n}})$
is the product of $\iauxsp(n)$ copies of the entropy space $(\entsp, \algebra{\entsp})$
and
the base measure $\measure{\nauxsp{n}}$ on $\nauxsp{n}$
is the product of $\iauxsp(n)$ copies of the entropy measure $\measure{\entsp}$
with the Radon-Nikodym derivative $\nauxpdf{n}$.
For ease of reference, we write
$(\auxsp, \algebra{\auxsp}, \measure{\auxsp})$
for the one-dimensional auxiliary space.
% $\nauxsp{n}$ is a
% $(n\cdot\spdim{\auxsp})$-dimensional
% differential manifold,
% equipped with the product atlas.

\subsubsection{State Space}

A \defn{state} is a pair of
\emph{equal dimension
but not necessarily equal length}
parameter and auxiliary variable.
For instance
with $\iparsp(n):=n+1$ and $\iauxsp(n):=n$,
the parameter variable
$\enta := [(-0.2,\true), (2.9, \true), (1.3, \false)]$
and the auxiliary variable
$\auxa := [(1.5,\true), (-2.1, \false)]$
are both of dimension two and
$(\enta, \auxa)$ is a two-dimensional state.

Formally, the \defn{state space} $\states$ is
the list measurable space of
the product of
parameter and auxiliary spaces of equal dimension,
i.e.~$\states := \bigcup_{n\in\Nat} (\nparsp{n} \times \nauxsp{n})$,
equipped with the $\sigma$-algebra
$\algebra{\states} := \sigma
\set{\entseta_n \times \auxseta_n \mid
\entseta_n \in \algebra{\nparsp{n}}, \auxseta_n \in \algebra{\nauxsp{n}}, n\in\Nat}$ and
measure
$\measure{\states}(\stateset) := \sum_{n\in\Nat} \expandint{\nauxsp{n}}{\measure{\nparsp{n}}(\set{\enta\in\nparsp{n}\mid (\enta,\auxa)\in\stateset})}{\measure{\nauxsp{n}}}{\auxa}$.
We write $\nstates{n}$ for the set
consisting of all $n$-dimensional states.

We extend the notion of instances to states and
say a trace $\traceb$ is an instance of a state
$(\enta, \auxa)$ if it is an instance of
the parameter component $\enta$.

The distinction between dimension and length
in parameter and auxiliary variables
gives us the necessary pliancy to
discuss techniques for further extension of the
Hybrid NP-iMCMC sampler in \cref{app: variants}.
Before that,
we present the inputs to the Hybrid NP-iMCMC sampler.

\subsection{Inputs of Hybrid NP-iMCMC Algorithm}
\label{sec: inputs (np-imcmc)}

Besides the target density function,
the Hybrid NP-iMCMC sampler, like iMCMC,
introduces randomness via \defn{auxiliary kernels} and
moves around the state space via \defn{involutions}
in order to propose the next sample.
We now examine each of these inputs closely.

\subsubsection{Target Density Function}

Similar to other inference algorithms for
probabilistic programming,
the Hybrid NP-iMCMC sampler takes
the \emph{weight function}
$\tree:\traces \to\pReal$
as the target density function.
Recall $\tree(\traceb)$
gives the weight of a particular run of
the given probabilistic program
indicated by the trace $\traceb$.
By \cref{prop: all spcf terms have TR weight function},
the weight function $\tree$ is tree representable.
For the sampler to work properly,
we also require weight function $\tree$ to
satisfy the following assumptions.
\begin{itemize}
  \item[\hass\label{hass: integrable tree}]
  $\tree$ is \defn{integrable},
  i.e.~$\shortint{\traces}{\tree}{\measure{\traces}} =: Z < \infty$
  (otherwise, the inference problem is undefined).

  \item[\hass\label{hass: almost surely terminating tree}]
  $\tree$ is \defn{almost surely terminating (AST)},
  i.e.~$\measure{\traces}(\set{\traceb\in \traces \mid \tree(\traceb) > 0}) = 1$
  (otherwise, the loop in the Hybrid NP-iMCMC algorithm
  may not terminate almost surely).
\end{itemize}

Virtually all useful probabilistic models can be
specified by SPCF programs
with densities satisfying
\cref{hass: integrable tree,hass: almost surely terminating tree}.
Exceptions are models that are not normalizable
or diverge with non-zero probability.
(See \cref{sec: spcf ast and integrability}
for more details.)

\iffalse
As discussed in \cref{sec: entropy space},
the Hybrid NP-iMCMC sampler traverses
the list entropy space
$\bigcup_{n\in\Nat}\nparsp{n}$
instead of the trace space.
Hence, instead of directly using
the weight function $\tree$ on the
\emph{trace space} in the Hybrid NP-iMCMC algorithm,
we define an \defn{extension} $\extree$
of a weight function $\tree$,
which is on the list entropy space,
has type $\bigcup_{n\in\Nat}\nparsp{n} \to \pReal $
and
$\extree(\enta)$
returns the weight of
the instance $\traceb \in \instances{\enta}$
in the support of $\tree$,
if it exists; and returns $0$ otherwise.
Formally,
\begin{align*}
  \extree: \bigcup_{n\in\Nat}\nparsp{n} & \to \pReal \\
  \enta & \mapsto \max\{ \tree(\traceb) \mid \traceb \in \instances{\enta} \}.
\end{align*}
\fi

\subsubsection{Auxiliary kernels}

To introduce randomness,
the Hybrid NP-iMCMC sampler takes, for each $n\in\Nat$,
a probability \defn{auxiliary kernel}
$\nkernel{n}: \nparsp{n} \kernelto \nauxsp{n}$
which gives a probability distribution
$\nkernel{n}(\enta, \placeholder) $
on $\nauxsp{n}$ for each $n$-dimensional
parameter variable $\enta$.
We assume
each auxiliary kernel $\nkernel{n}$ has a
probability density function (pdf)
$\nkernelpdf{n}:\nparsp{n} \times \nauxsp{n} \to \pReal$
w.r.t.~$\measure{\nauxsp{n}}$.

\iffalse
For ease of reference, we define the \defn{extended auxiliary kernel} $\exauxkernel:\bigcup_{n\in\Nat} \nparsp{n} \kernelto \bigcup_{n\in\Nat} \nauxsp{n}$ which given a trace $\enta$ with $\range{\enta}{1}{k} \in \support{\tree}$, returns an auxiliary variable $\auxa$ where the first $k$-th components of $\auxa$ has distribution $\nkernel{k}(\range{\enta}{1}{k}, \placeholder)$ and the rest of the components has the stock distribution $\measure{\auxsp}$.
Formally, $\exauxkernel(\enta, \auxseta)$ is set to be
$\expandint{\auxseta}{\nkernelpdf{k}(\range{\enta}{1}{k}, \range{\auxa}{1}{k})}{\measure{\nauxsp{n}}}{\auxa}$ for $\range{\enta}{1}{k} \in \support{\tree}$ and
$\measure{\nauxsp{n}}(\auxseta)$ otherwise.
\fi

\subsubsection{Involutions}

To move around the state space $\states$,
the Hybrid NP-iMCMC sampler takes, for each $n\in\Nat$,
an endofunction
$\ninvo{n}$ on $\nparsp{n} \times \nauxsp{n}$
that is both involutive and differentiable
almost everywhere.
We require the set $\set{\ninvo{n}}_n$ of involutions
to satisfy the \emph{\invoass{}}:
\begin{itemize}
  \item[\hass\label{hass: partial block diagonal inv}]
  For all $(\enta, \auxa) \in \states$ where $\dim(\enta) = m$,
  if $\nsupport{\tree}{\iparsp{(n)}} \cap \instances{\enta} \not= \varnothing$ for some $n$,
  then for all $k = n, \dots, m$,
  $\take{k}(\ninvo{m}(\enta, \auxa)) = \ninvo{k}(\take{k}(\enta, \auxa)) $
\end{itemize}
where $\take{k}$ is the projection that
given a state $(\enta,\auxa)$,
takes the first $\iparsp(k)$ coordinates of $\enta$ and
the first $\iauxsp(k)$ coordinates of $\auxa$
and forms a $k$-dimensional state.

The \invoass{} ensures that the order of applying
a projection and an involution
to a state (which has an instance in the support
of the target density function)
does not matter.

\subsection{The Hybrid NP-iMCMC Algorithm}
\label{sec: np-imcmc algorithm}

After identifying the state space
and the necessary conditions on the inputs
of the Hybrid NP-iMCMC sampler,
we have enough foundation to describe the
algorithm.

Given a SPCF program $\terma$
with weight function
$\tree$ on the trace space $\traces$,
the \defn{Hybrid Nonparametric Involutive Markov chain Monte Carlo (Hybrid NP-iMCMC)} algorithm
% a set $\set{\nkernel{n}:\nparsp{n} \kernelto\nauxsp{n}}$ of auxiliary kernels and
% a set $\set{\ninvo{n}:\nparsp{n} \times \nauxsp{n} \to \nparsp{n} \times \nauxsp{n}}$ of involutions satisfying \cref{hass: partial block diagonal inv,hass: integrable tree,hass: almost surely terminating tree},
generates a Markov chain on $\traces$
as follows.
Given a current sample $\traceb_0$ of
dimension $k_0$
(i.e.~$\len{\traceb_0} = \iparsp(k_0)$),
\begin{enumerate}
  \item{\hstep\label{hnp-imcmc step: initialisation}}
  (Initialisation Step)
  Form a ${k_0}$-dimensional parameter variable
  $\enta_0 \in \nparsp{k_0}$ by
  pairing each value
  $\seqindex{\traceb_0}{i}$ in
  $\traceb_0$ with
  a randomly drawn value $\trace$
  of the other type to make
  a pair
  $(\seqindex{\traceb_0}{i}, \trace)$ or $(\trace, \seqindex{\traceb_0}{i})$
  in the entropy space $\entsp$.
  Note that $\traceb_0$ is the unique instance of
  $\enta_0$ that is in the support of $\tree$.

  \item{\hstep\label{hnp-imcmc step: stochastic}}
  (Stochastic Step)
  Introduce randomness to the sampler by
  drawing a ${k_0}$-dimensional value
  $\auxa_0 \in \nauxsp{k_0}$
  from the probability measure
  $\nkernel{{k_0}}(\enta_0, \placeholder)$.

  \item{\hstep\label{hnp-imcmc step: deterministic}}
  (Deterministic Step)
  Move around the $n$-dimensional state space $\nparsp{n}\times\nauxsp{n}$ and
  compute the new state $(\enta,\auxa)$
  by applying the involution $\ninvo{n}$
  to the \emph{initial state} $(\enta_0,\auxa_0)$
  where $n = \dim{(\enta_0)} = \dim{(\auxa_0)}$.

  \item{\hstep\label{hnp-imcmc step: extend}}
  (Extend Step)
  Test whether
  any instance $\traceb$ of $\enta$ is in
  the support of $\tree$.
  If so, proceed to the next step
  with $\traceb$ as the proposed sample;
  otherwise
  \begin{enumerate}[(i)]
    \item
    Extend the $n$-dimensional initial state to
    a state
    $(\enta_0 \concat \entb_0, \auxa_0 \concat \auxb_0)$
    of dimension $n+1$
    where $\entb_0$ and $\auxb_0$ are values
    drawn randomly
    from $\measure{\entsp^{\iparsp{(n+1)}-\iparsp{(n)}}}$ and
    $\measure{\entsp^{\iauxsp{(n+1)}-\iauxsp{(n)}}}$
    respectively,

    \item Go to \cref{hnp-imcmc step: deterministic}
    with the initial state
    $(\enta_0, \auxa_0)$ replaced by
    $(\enta_0 \concat \entb_0, \auxa_0 \concat \auxb_0)$.
  \end{enumerate}

  \item{\hstep\label{hnp-imcmc step: accept/reject}}
  (Accept/reject Step)
  Accept the proposed sample $\traceb$
  as the next sample with probability
  \begin{align}
    \label{eq: np-imcmc acceptance ratio}
    \min\bigg\{1; \;
    \frac
      {
        \tree{(\traceb)}\cdot
        \nkernelpdf{k}(\proj{n}{k}(\enta, \auxa)) \cdot
        \nparpdf{n}(\enta)\cdot\nauxpdf{n}(\auxa)}
      {
        \tree{(\traceb_0)}\cdot
        \nkernelpdf{k_0}(\proj{n}{k_0}(\enta_0, \auxa_0)) \cdot
        % \nkernelpdf{k_0}(
        %   \range{\enta_0}{1}{k_0},
        %   \range{\auxa_0}{1}{k_0})\cdot
        \nparpdf{n}(\enta_0)\cdot\nauxpdf{n}(\auxa_0)}
      \cdot
    % \\
    % & \qquad\qquad\qquad\qquad\qquad\qquad\qquad\qquad\qquad
    \abs{\det(
        \grad{\ninvo{n}(
          \enta_0, \auxa_0)})}
    \bigg\}
  \end{align}
  where $n = \dim{(\enta_0)} = \dim{(\auxa_0)}$,
  $k$ is the dimension of $\traceb$ and
  $k_0$ is the dimension of ${\traceb_0}$;
  otherwise reject the proposal and repeat $\traceb_0$.
\end{enumerate}

\begin{remark} \label{rm: conditions on inputs}
  The integrable assumption on the target density
  (\cref{hass: integrable tree})
  ensures the inference problem is well-defined.
  The almost surely terminating assumption
  on the target density
  (\cref{hass: almost surely terminating tree})
  guarantees that the Hybrid NP-iMCMC sampler
  almost surely terminates.
  (See \cref{sec: Almost sure termination (np-imcmc)}
  for a concrete proof.)
  The \invoass{} on the involutions
  (\cref{hass: partial block diagonal inv})
  allows us to define the invariant distribution
\end{remark}

\subsubsection{Movement Between Samples of Varying Dimensions}

All MCMC samplers that simulate a nonparametric model
must decide how to move between samples of varying dimensions.
We now discuss how the Hybrid NP-iMCMC sampler
as given in \cref{sec: np-imcmc algorithm}
achieves this.

\paragraph{Form initial and new states in the same dimension}
Say the current sample $\traceb_0$ has a dimension of $k_0$.
\cref{hnp-imcmc step: initialisation,hnp-imcmc step: stochastic,hnp-imcmc step: deterministic}
form a $k_0$-dimensional initial state
$(\enta_0, \auxa_0)$ and a new
$k_0$-dimensional state $(\enta,\auxa)$.

\paragraph{Move between dimensions}
The novelty of Hybrid NP-iMCMC is its ability to generate
a proposed sample $\traceb$ in the support of the target density $\tree$
which may not be of same dimension as $\traceb_0$.
This is achieved by \cref{hnp-imcmc step: extend}.

\paragraph{Propose a sample of a lower dimension}
\cref{hnp-imcmc step: extend} first checks whether
any instance of the parameter-component $\enta \in \nparsp{k_0}$
of the new state
(computed in \cref{hnp-imcmc step: deterministic})
is in the support of $\tree$.
If so, we proceed to \cref{hnp-imcmc step: accept/reject} with that instance, say $\traceb$,
as the proposed sample.

Say the dimension of $\traceb$ is $k$.
Then, we must have $k \leq k_0$ as
the instance $\traceb\in\traces$ of a
$k_0$-dimensional parameter $\enta\in\nparsp{k_0}$
must have a dimension that is lower than or equals to $k_0$.
Hence, the dimension of the proposed sample $\traceb$
is lower than or equals to
the current sample $\traceb_0$.

\paragraph{Propose a sample of a higher dimension}
Otherwise (i.e.~none of the instances
of $\enta \in \nparsp{k_0}$ is in the support of $\tree$)
\cref{hnp-imcmc step: extend} extends the initial state
$(\enta_0, \auxa_0) \in \nparsp{k_0} \times \nauxsp{k_0} $ to
$(\enta_0\concat\entb_0, \auxa_0\concat\auxb_0)\in \nparsp{k_0+1} \times \nauxsp{k_0+1}$;
and computes a new $(k_0+1)$-dimensional state
$(\enta\concat\entb, \auxa\concat\auxb)\in \nparsp{k_0+1} \times \nauxsp{k_0+1}$
(via \cref{hnp-imcmc step: deterministic}).
This process of incrementing the dimensions
of both the initial and new states is repeated
until an instance $\traceb$ of the new state,
say of dimension $n$,
is in the support of $\tree$.
At which point, the proposed sample is set to be
$\traceb$.

Say the dimension of $\traceb$ is $k$.
Then, we must have $k > k_0$ as
$\traceb$ is \emph{not} an instance of the
$k_0$-dimensional parameter $\enta\in\nparsp{k_0}$
but one of $\enta \concat \entb \in \nparsp{n}$.
Hence, the dimension of the proposed sample $\traceb$
is higher than
the current sample $\traceb_0$.

\paragraph{Accept or reject the proposed sample}
Say the proposed sample $\traceb$ is of dimension $k$.
With the probability given in \cref{eq: np-imcmc acceptance ratio},
\cref{hnp-imcmc step: accept/reject} accepts $\traceb$
as the next sample
and Hybrid NP-iMCMC updates
the current sample $\traceb_0$ of dimension $k_0$ to
a sample $\traceb$ of dimension $k$.
Otherwise, the current sample $\traceb_0$ is repeated
and the dimension remains unchanged.

\subsubsection{Hybrid NP-iMCMC is a Generalisation of NP-iMCMC}
\label{sec: Hybrid NP-iMCMC is a Generalisation of NP-iMCMC}

Given a target density $\tree$ on
$\bigcup_{n\in\Nat}\Real^n$,
we can set the entropy space $\entsp$
to be $\Real$ and the index maps
$\iparsp$ and $\iauxsp$ to be identities.
Then,
the $n$-dimensional parameter space
$\nparsp{n} := \Real^n$,
the $n$-dimensional auxiliary space
$\nauxsp{n} := \Real^n$ and
the state space
$\states
:= \bigcup_{n\in\Nat} (\nparsp{n}\times\nauxsp{n})
= \bigcup_{n\in\Nat} (\Real^{n}\times\Real^{n})$ of the Hybrid NP-iMCMC sampler
matches with those given in \cref{sec: assumptions}
for the NP-iMCMC sampler.
An instance $\traceb$
is then a prefix
$\range{\enta}{1}{k}$
of a parameter variable $\enta$.
Moreover, the assumptions \allhass{}
on the inputs of Hybrid NP-iMCMC are identical to those \allvass{} on the inputs of NP-iMCMC.
Hence the Hybrid NP-iMCMC algorithm
(\cref{sec: np-imcmc algorithm})
is a generalisation of the NP-iMCMC sampler (\cref{fig:np-imcmc algo}).

\subsubsection{Pseudocode of Hybrid NP-iMCMC Algorithm}
\label{sec: pesudocode of np-imcmc}

We implement the Hybrid NP-iMCMC algorithm
in the flexible and expressive SPCF language explored in \cref{app: spcf}.

\begin{figure}[t]
\begin{code}[label={code: np-imcmc}, caption={Pseudocode of the Hybrid NP-iMCMC algorithm}]
def NPiMCMC(t0):
  k0 = dim(t0)                                                 # initialisation step
  x0 = [(e, coin) if Type(e) in R else (normal, e) for e in t0]
  v0 = auxkernel[k0](x0)                                       # stochastic step
  (x,v) = involution[k0](x0,v0)                                # deterministic step
  n = k0                                                       # extend step
  while not intersect(instance(x),support(w)):
    x0 = x0 + [(normal, coin)]*(indexX(n+1)-indexX(n))
    v0 = v0 + [(normal, coin)]*(indexY(n+1)-indexY(n))
    n = n + 1
    (x,v) = involution[n](x0,v0)
  t = intersect(instance(x),support(w))[0]                     # accept/reject step
  k = dim(t)
  return t if uniform < min{1, w(t)/w(t0) * pdfauxkernel[k](proj((x,v),k))/
                                              pdfauxkernel[k0](proj((x0,v0),k0)) *
                                            pdfpar[n](x)/pdfpar[n](x0) *
                                            pdfaux[n](v)/pdfaux[n](v0) *
                                            absdetjacinv[n](x0,v0)}
           else t0
\end{code}
\end{figure}

  % # return a list of instances of x that is in the support of w
  % def instanceANDsupport(x,w):
  %   instances = [t[:k+1]
  %     for t in [[x[i][p[i]] for i in range(len(x))] for p in product([0,1],len(x))]
  %     for k in range(len(t))]
  %   return filter(lambda t': w(t') > 0, instances)

The \codein{NPiMCMC} function in \cref{code: np-imcmc}
is an implementation of the Hybrid NP-iMCMC algorithm in SPCF.
We assume that
the following SPCF types and terms exist.
For each $n\in\Nat$, the SPCF types
\codein{T}, \codein{X[n]} and \codein{Y[n]}
implements $\traces$, $\nparsp{n}$ and $\nauxsp{n}$ respectively;
the SPCF term
\codein{w} of type \codein{T -> R}
implements the target density $\tree$;
for each $n\in\Nat$,
the SPCF terms
\codein{auxkernel[n]} of type \codein{X[n] -> Y[n]}
implements the auxiliary kernel $\nkernel{n}:\nparsp{n} \kernelto \nauxsp{n}$;
\codein{pdfauxkernel[n]} of type \codein{X[n]*Y[n] -> R}
implements the probability density function
$\nkernelpdf{n} : \nparsp{n}\times\nauxsp{n} \to \Real$ of the auxiliary kernel;
\codein{involution[n]} of type \codein{X[n]*Y[n] -> X[n]*Y[n]}
implements the involution $\ninvo{n}$ on $\nparsp{n} \times \nauxsp{n}$; and
\codein{absdetjacinv[n]} of type \codein{X[n]*Y[n] -> R}
implements the absolute value of the Jacobian determinant of $\ninvo{n}$.

We further assume that the following primitives
are implemented:
\codein{dim} returns the dimension of a given trace;
\codein{indexX} and \codein{indexY}
implements the maps $\iauxsp$ and $\iparsp$ respectively;
\codein{pdfpar[n]} implements the
derivative $\nparpdf{n}$ of the $n$-dimensional parameter space $\nparsp{n}$;
\codein{pdfaux[n]} implements the
derivative $\nauxpdf{n}$ of the $n$-dimensional auxiliary space $\nauxsp{n}$;
\codein{instance} returns a list of all instances of a given entropy vector;
\codein{support} returns a list of traces in the support of a given function; and
\codein{proj} implements the projection function
where \codein{proj((x,v),k)=(x[:indexX(k)],v[:indexY(k)])}.
% \codein{intersect} returns the intersection of the given lists;
% \codein{min} implements the minimum function.

\subsection{Correctness}
\label{app: proof of correctness}
% !TEX root = ./../icml2022.tex

The Hyrbid Nonparametric Involutive
Markov chain Monte Carlo (Hyrbid NP-iMCMC)
algorithm is presented in \cref{sec: np-imcmc algorithm}
for the simulation of probabilistic
models specified by probabilistic programs.

We justify this by
proving that
the Markov chain generated by iterating the
Hybrid NP-iMCMC algorithm
preserves the target distribution, specified by
\begin{align*}
  {\tdist:}\
  {\Sigma_{\traces}}& \longrightarrow {\pReal} \\
  {U} & \longmapsto {\frac{1}{Z}\shortint{U}{\tree}{\measure{\traces}}}
  \qquad\text{where }
  Z := \shortint{\traces}{\tree}{\measure{\traces}},
\end{align*}
as long as
the target density function $\tree$
(given by the weight function of the
probabilistic program)
is integrable (\cref{hass: integrable tree}) and
almost surely terminating (\cref{hass: almost surely terminating tree});
with a probability kernel
$\nkernel{n}:\nparsp{n} \kernelto \nauxsp{n}$
and
an endofunction $\ninvo{n}$ on $\nparsp{n}\times\nauxsp{n}$
that is involutive and
differentiable almost everywhere
for each $n\in\Nat$ such that
$\set{\ninvo{n}}_n$ satisfies the \invoass{}
(\cref{hass: partial block diagonal inv}).

Throughout this chapter, we assume the
assumptions stated above, and prove
the followings.
\begin{enumerate}
  \item The Hybrid NP-iMCMC sampler
  almost surely returns a sample
  for the simulation
  (\cref{lemma: np-imcmc is ast}).

  \item The state movement in the Hybrid NP-iMCMC
  sampler
  preserves a distribution on the states
  (\cref{lemma: e-np-imcmc invariant}).

  \item The marginalisation of the
  state distribution which
  the state movement of Hybrid NP-iMCMC preserves
  coincides with
  the target distribution
  (\cref{lemma: marginalised distribution is the target distribution}).
\end{enumerate}

\subsubsection{Almost Sure Termination}
\label{sec: Almost sure termination (np-imcmc)}
% !TEX root = ./../icml2022.tex

In \cref{rm: conditions on inputs},
we asserted that
the almost surely terminating assumption
(\cref{hass: almost surely terminating tree})
on the target density
guarantees that
the Hybrid NP-iMCMC algorithm
(\cref{sec: np-imcmc algorithm})
almost surely terminates.
We justify this claim here.

\cref{np-imcmc step: extend}
in the Hybrid NP-iMCMC algorithm
(\cref{sec: np-imcmc algorithm})
repeats itself if
the sample-component ${\enta}$ of the
new state ${(\enta,\auxa)}$
(computed by
applying the involution $\ninvo{n}$
on the extended initial state
$(\enta_0,\auxa_0)$)
does not have an instance
in the support of ${\tree}$.
This loop halts almost surely if
the measure of
\[
  \set{(\enta_0, \auxa_0) \in \states \mid
  (\enta, \auxa) = \ninvo{n}(\enta_0, \auxa_0)
  \text{ and }
  \instances{\enta} \cap \support{\tree} = \varnothing
  }
\]
tends to zero as the dimension $n$ tends to infinity.
Since $\ninvo{n}$ is invertible and
$\abs{\det \grad{\ninvo{n}}(\enta_0, \auxa_0)} > 0$
for all $n\in\Nat$ and
$(\enta_0, \auxa_0) \in \states$,
\begin{align*}
  & \measure{\states}(\set{(\enta_0, \auxa_0) \in \states \mid
  (\enta, \auxa) = \ninvo{n}(\enta_0, \auxa_0)
  \text{ and }
  \instances{\enta} \cap \support{\tree} = \varnothing
  }) \\
  & = \ninvo{n}_*\measure{\states}(\set{(\enta, \auxa) \in \states \mid
  \instances{\enta} \cap \support{\tree} = \varnothing
  }) \\
  & < \measure{\states}(\set{(\enta, \auxa) \in \states \mid
  \instances{\enta} \cap \support{\tree} = \varnothing
  }) \\
  & = \measure{\nparsp{n}}(\set{\enta \in \nparsp{n}
  \mid
  \instances{\enta} \cap \support{\tree} = \varnothing
  }).
\end{align*}
Thus
it is enough to show that
the measure of a $n$-dimensional
parameter variable
not having any instances
in the support of $\tree$ tends to zero
as the dimension $n$ tends to infinity,
i.e.
\[
  \measure{\nparsp{n}}(\set{\enta \in \nparsp{n}
  \mid
  \instances{\enta} \cap \support{\tree} = \varnothing
  }) \to 0
  \qquad\text{as}\qquad n\to\infty.
\]

We start with the following proposition which
shows that
the chance of a $n$-dimensional
parameter variable
having some instances in the support of $\tree$
is the same as
the chance of $\tree$ terminating
before $n$ reduction steps.

\begin{proposition}
  \label{prop: tree and extree}
  $\measure{\nparsp{n}}(\set{\enta \in \nparsp{n}
  \mid
  \instances{\enta} \cap \support{\tree} \not= \varnothing
  })
  = \measure{\traces}(\bigcup_{i=1}^{n} \nsupport{\tree}{{\iparsp{(i)}}})$
  for all $n\in\Nat$ and
  all tree representable function $\tree$.
\end{proposition}

\begin{proof}
  Let $n\in\Nat$ and $\tree$ be a tree representable
  function.

  For each $i \leq n$,
  we unpack the set
  $\set{\enta \in \nparsp{i}
  \mid
  \instances{\enta} \cap \traceset \not= \varnothing
  }$
  of $i$-dimensional parameter variables
  that has an instance in the set
  $\traceset \in \algebra{\samsp^{\iparsp(i)}}$
  of traces of length $\iparsp(i)$
  where $\samsp := \Real \cup \Bool$.
  Write $\pi:\set{1, \dots, {\iparsp{(i)}}} \to \set{\Real, \Bool}$ for the measurable space
  $\pi(1) \times \pi(2) \times \dots \times \pi({\iparsp{(i)}})$
  with a probability measure $\measure{\pi} := \measure{\samsp^{\iparsp{(i)}}}$ on $\pi$;
  $\inv{\pi}$ for the ``inverse'' measurable space of $\pi$,
  i.e.~$\inv{\pi}(j) :=
  \samsp \setminus \pi(j)$ for all $j \leq \iparsp(i)$; and
  $S$ for the set of all such measurable spaces.
  Then,
  for any $i$-dimensional parameter variable
  $\enta$, $\traceb \in \instances{\enta}\cap \traceset$
  if and only if
  there is some $\pi \in S$ where
  $\traceb \in \traceset \cap \pi$
  and
  $\enta \in \zip(\traceset \cap \pi, \inv{\pi})$.
  Hence,
  $\set{\enta \in \nparsp{i}
  \mid
  \instances{\enta} \cap \traceset \not= \varnothing
  }$
  can be written as
  $\bigcup_{\pi \in S} \zip(\traceset \cap \pi, \inv{\pi})$.
  Moreover
  $\measure{\nparsp{i}}(\zip(\traceset \cap \pi, \inv{\pi}))
  = \measure{\pi}(\traceset \cap \pi) \cdot
  \measure{\inv{\pi}}(\inv{\pi})
  = \measure{\pi}(\traceset \cap \pi)$.

  % is the set containing all $i$-dimensional
  % parameter variables which has
  % at least one instance in $\traceset$.

  Consider the case where $\traceset := \nsupport{\tree}{\iparsp(i)}$.
  Then, we have
  \[
  \set{\enta \in \nparsp{i} \mid
  \instances{\enta} \cap \nsupport{\tree}{\iparsp(i)} \not= \varnothing}
  = \bigcup_{\pi \in S} \zip(\nsupport{\tree}{\iparsp(i)} \cap \pi, \inv{\pi}).
  \]
  We first show that this is a disjoint union,
  i.e.~for all $\pi \in S$,
  $\zip(\nsupport{\tree}{\iparsp(i)} \cap \pi, \inv{\pi})$ are disjoint.
  Let
  $\enta \in \zip(\nsupport{\tree}{\iparsp(i)} \cap \pi_1, \inv{\pi_1}) \cap
  \zip(\nsupport{\tree}{\iparsp(i)} \cap \pi_2, \inv{\pi_2})$
  where $\pi_1, \pi_2 \in S$.
  Then, at least one instance $\traceb_1$ of $\enta$ is in $\nsupport{\tree}{\iparsp(i)} \cap \pi_1$ and
  similarly at least one instance $\traceb_2$ of $\enta$ is in $\nsupport{\tree}{\iparsp(i)} \cap \pi_2$.
  By \cref{prop: instances and tree representable}, $\traceb_1 = \traceb_2$
  and hence $\pi_1 = \pi_2$.

  Since {$\zip(\nsupport{\tree}{{\iparsp{(i)}}} \cap \pi, \inv{\pi})$ are disjoint} for all $\pi \in S$,
  we have
  \begin{align*}
    & \measure{\nparsp{i}}(\set{\enta \in \nparsp{i} \mid
    \instances{\enta} \cap \nsupport{\tree}{\iparsp(i)} \not= \varnothing})
    =
    \measure{\nparsp{i}}(\bigcup_{\pi \in S} \zip(\nsupport{\tree}{{\iparsp{(i)}}} \cap \pi, \inv{\pi})) \\
    & =
    \sum_{\pi \in S} \measure{\nparsp{i}}(\zip(\nsupport{\tree}{{\iparsp{(i)}}} \cap \pi, \inv{\pi}))
    % =
    % \sum_{\pi \in S} \measure{\pi}(\nsupport{\tree}{{\iparsp{(i)}}} \cap \pi) \cdot \measure{\inv{\pi}}(\inv{\pi}) \\
    =
    \sum_{\pi \in S} \measure{\pi}(\nsupport{\tree}{{\iparsp{(i)}}} \cap \pi) \\
    & =
    \sum_{\pi \in S} \measure{\samsp^{\iparsp{(i)}}}(\nsupport{\tree}{{\iparsp{(i)}}} \cap \pi)
    =
    \measure{\samsp^{\iparsp{(i)}}}(\nsupport{\tree}{{\iparsp{(i)}}})
    =
    \measure{\traces}(\nsupport{\tree}{{\iparsp{(i)}}}).
  \end{align*}

  Finally,
  $\set{\enta \in \nparsp{n}
  \mid
  \instances{\enta} \cap \support{\tree} \not= \varnothing
  }$ is equal to
  $\bigcup_{i=1}^n
  \set{\enta \in \nparsp{i} \mid
    \instances{\enta} \cap \nsupport{\tree}{\iparsp(i)} \not= \varnothing}
  \times \entsp^{\iparsp(n)-\iparsp(i)}$
  and hence
  \begin{align*}
    &\measure{\nparsp{n}}(\set{\enta \in \nparsp{n}
    \mid
    \instances{\enta} \cap \support{\tree} \not= \varnothing}) \\
    & =
    \measure{\nparsp{n}}(\bigcup_{i=1}^n
    \set{\enta \in \nparsp{i} \mid
      \instances{\enta} \cap \nsupport{\tree}{\iparsp(i)} \not= \varnothing}
    \times \entsp^{\iparsp(n)-\iparsp(i)})
    \\
    & =
    \sum_{i=1}^n
    \measure{\nparsp{i}}(
    \set{\enta \in \nparsp{i} \mid
      \instances{\enta} \cap \nsupport{\tree}{\iparsp(i)} \not= \varnothing}) \\
    & =
    \sum_{i=1}^n
    \measure{\traces}(\nsupport{\tree}{{\iparsp{(i)}}}) \\
    & =
    \measure{\traces}(\bigcup_{i=1}^n \nsupport{\tree}{{\iparsp{(i)}}})
  \end{align*}
\end{proof}

\cref{prop: tree and extree} links the
termination of the Hybrid NP-iMCMC sampler
with that of the target density function $\tree$.
Hence by assuming that
$\tree$ terminates almost surely
(\cref{hass: almost surely terminating tree}),
we can deduce that the Hybrid NP-iMCMC algorithm (\cref{sec: np-imcmc algorithm}) almost surely terminates.

\begin{restatable}[Almost Sure Termination]{lemma}{ast}
  \label{lemma: np-imcmc is ast}
  Assuming \cref{hass: almost surely terminating tree},
  the Hybrid NP-iMCMC algorithm (\cref{sec: np-imcmc algorithm})
  almost surely terminates.
\end{restatable}
% \ast

\begin{proof}
  Since $\ninvo{n}$ is invertible for all $n\in\Nat$, and
  $\tree$ almost surely terminates
  (\cref{hass: almost surely terminating tree}),
  i.e.~$\lim_{m\to\infty} \measure{\traces}(\bigcup_{j=1}^m \nsupport{\tree}{j}) = 1$,
  we deduce from \cref{prop: tree and extree} that
  \begin{align*}
    &\measure{\states}(
    \set{(\enta_0, \auxa_0) \in \states
    \mid
    (\enta, \auxa) = \ninvo{n}(\enta_0, \auxa_0)
    \text{ and }
    \instances{\enta} \cap \support{\tree} = \varnothing
    }) \\
    & < \measure{\nparsp{n}}(
    \set{\enta \in \nparsp{n}
    \mid
    \instances{\enta} \cap \support{\tree} = \varnothing
    }) \\
    & =
    \measure{\nparsp{n}}( \nparsp{n} \setminus
    \set{\enta \in \nparsp{n}
    \mid
    \instances{\enta} \cap \support{\tree} \not= \varnothing
    }) \\
    & =
    1-\measure{\nparsp{n}}(
    \set{\enta \in \nparsp{n}
    \mid
    \instances{\enta} \cap \support{\tree} \not= \varnothing
    }) \\
    & =
    1-\measure{\traces}(
    \bigcup_{i=1}^n \nsupport{\tree}{\iparsp(i)})
    \tag{\cref{prop: tree and extree}}
    \\
    & \to 1-1 = 0 \quad \text{as }n\to\infty.
    \tag{\cref{hass: almost surely terminating tree}}
  \end{align*}
  So the probability of satisfying the condition
  of the loop in \cref{np-imcmc step: extend}
  of Hybrid NP-iMCMC sampler tends to zero as
  the dimension $n$ tends to infinity,
  making the Hybrid NP-iMCMC sampler (\cref{sec: np-imcmc algorithm})
  almost surely terminating.
\end{proof}

\subsubsection{Invariant State Distribution}
\label{sec: Invariant State Distribution (np-imcmc)}
% !TEX root = ./../icml2022.tex

After ensuring the Hybrid NP-iMCMC sampler
(\cref{sec: np-imcmc algorithm})
almost always returns a sample (\cref{lemma: np-imcmc is ast}),
we identify the distribution on the states
and show that it is invariant against
the movement between states of varying dimensions
in Hybrid NP-iMCMC.

\paragraph{State Distribution}
\label{sec: state distribution}

Recall a state is an equal dimension
parameter-auxiliary pair.
We define the \defn{state distribution} $\sdist$ on the
state space $\states := \bigcup_{n\in\Nat} (\nparsp{n} \times \nauxsp{n})$
to be a distribution with density $\spdf$ (with respect to $\smeasure$) given by
\begin{align} \label{eq: defn of state distribution}
  \spdf(\enta,\auxa) :=
  \begin{cases}
    \multicolumn{2}{@{}l@{\quad}}{
    \displaystyle\frac{1}{Z}\cdot
    \tree(\traceb) \cdot
    \nkernelpdf{k}(\proj{n}{k}(\enta, \auxa))
    } \\
    &\qquad \text{if }(\enta,\auxa)\in \validstates \text{ and }
    \traceb \in \instances{\enta} \cap \support{\tree}
    \text{ has dimension } k \nonumber\\
    {0} &\qquad \text{otherwise}
  \end{cases}
\end{align}
where
$Z := \shortint{\traces}{\tree}{\measure{\traces}}$
(which exists by \cref{hass: integrable tree}) and
% $\extree:\bigcup_{n\in\Nat} \nparsp{n}\to \pReal$ is
% the \emph{extension} of $\tree$,
% such that $\extree(\enta) := \max\{\tree(\traceb)
% \mid \traceb \in \instances{\enta} \}$, and
$\validstates$ is the subset of $\states$ consisting of all \emph{valid states}.

\begin{remark}
  If there is some trace in
  $\instances{\enta} \cap \support{\tree}$
  for a parameter variable $\enta$,
  by \cref{prop: instances and tree representable}
  this trace $\traceb$ is unique and hence
  $\enta$ represents a sample of
  the target distribution.
\end{remark}

We say a $n$-dimensional state
$(\enta,\auxa)$ is \defn{valid} if
\begin{enumerate}[(i)]
  \item $\instances{\enta} \cap \support{\tree} \not=\varnothing$, and

  \item $(\entb,\auxb) = \ninvo{n}(\enta,\auxa)$ implies
  $\instances{\entb} \cap \support{\tree} \not=\varnothing$, and

  \item $\proj{n}{k}(\enta, \auxa)
  \not\in \validstates$ for all $k <n$.
\end{enumerate}
Intuitively, valid states are the states which,
when transformed by the involution $\ninvo{n}$,
the instance of the parameter-component of
which does not
``fall beyond'' the support of $\tree$.

We write $\validstates_n := \validstates \cap (\nparsp{n} \times \nauxsp{n})$ to denote the the set of all $n$-dimensional valid states.
The following proposition shows that involutions preserve the validity of states.

\begin{proposition}
  \label{prop: valid states are preserved by involutions}
  Assuming \cref{hass: partial block diagonal inv},
  the involution $\ninvo{n}$ sends $\validstates_n$
  to $\validstates_n$ for all $n\in\Nat$.
  i.e.~If $(\enta,\auxa) \in \validstates_n$,
  then $(\entb,\auxb) = \ninvo{n}(\enta,\auxa) \in \validstates_n$.
\end{proposition}

\begin{proof}
  Let $(\enta,\auxa) \in \validstates_n$
  and $(\entb,\auxb) = \ninvo{n}(\enta,\auxa)$.
  We prove $(\entb,\auxb) \in \validstates_n$
  by induction on $n\in\Nat$.
  \begin{itemize}
    \item Let $n=1$.
    As $\ninvo{1}$ is involutive and
    $(\enta,\auxa)$ is a valid state,
    \begin{enumerate}[(i)]
      \item $\instances{\entb} \cap \support{\tree} \not=\varnothing$ and

      \item $(\enta,\auxa) = \ninvo{1}(\entb,\auxb)$ and
      $\instances{\enta} \cap \support{\tree}
      \not=\varnothing$.

      \item holds trivially
    \end{enumerate}
    and hence $(\entb,\auxb) \in \validstates_1$.

    \item Assume for all $m < n$,
    $(\entc,\auxc) \in \validstates_m$
    implies
    $(\entc',\auxc') = \ninvo{m}(\entc,\auxc)
    \in \validstates_m$.
    Similar to the base case, (i) and (ii)
    hold
    as $\ninvo{n}$ is involutive and
    $(\enta,\auxa)$ is a valid state.
    Assume for contradiction that (iii)
    does not hold, i.e.~there is $k < n$ where
    $\proj{n}{k}(\entb, \auxb)\in\validstates_k $.
    As $\instances{\proj{n}{k}(\entb)} \cap \support{\tree} \not=\varnothing$,
    by \cref{hass: partial block diagonal inv}
    and the inductive hypothesis,
    \[
      \proj{n}{k}(\enta, \auxa)
      =
      \proj{n}{k}(\ninvo{n}(\entb, \auxb))
      =
      \ninvo{k}(\proj{n}{k}(\entb, \auxb))
      \in \validstates_k
    \]
    which contradicts with the fact that
    $(\enta, \auxa)$ is a valid state.
  \end{itemize}
\end{proof}

We can partition the set
$\validstates$ of valid states.
Let $(\enta, \auxa)$ be a $m$-dimensional
valid state.
The parameter variable $\enta$ can be
written as
$\zip(\traceb_1, \traceb_2) \concat \entb$
where
$\traceb_1 \in \instances{\enta} \cap \support{\tree}$
is of dimension $k_0$,
$\traceb_2$ is a trace where
$\zip(\traceb_1, \traceb_2) =
\proj{m}{k_0}(\enta)$, and
$\entb := \drop{k_0}(\enta)$
where
$\drop{k}$ drops the first $\iparsp(k)$ components of the input parameter.
Similarly, the auxiliary variable $\auxa$
can be written as $\auxa_1 \concat \auxa_2$
where
$\auxa_1 := \proj{m}{k_0}(\auxa)$ and
$\auxa_2 := \drop{k_0}(\auxa)$
where
$\drop{k}$ drops the first $\iauxsp(k)$ components of the input parameter.
Hence, we have
\begin{align*}
  \validstates =
  \bigcup_{k_0=1}^{\infty}
  \bigcup_{m=1}^{\infty}
  \{ &
  (\zip(\traceb_1,\traceb_2)\concat\entb,
  \auxa_1\concat\auxa_2) \in \validstates_m
  \mid \\
  &
  \traceb_1 \in \nsupport{\tree}{\iparsp{(k_0)}},
  \traceb_2 \in \traces,
  \entb \in \entsp^{\iparsp(m)-\iparsp(k_0)},
  \auxa_1 \in \nauxsp{k_0},
  \auxa_2 \in \entsp^{\iauxsp(m)-\iauxsp(k_0)}
  \}
\end{align*}
and the state distribution $\sdist$ on
the measurable set $\stateset \in \algebra{\states}$ can be written as
\begin{align*} \label{eq: state distribution}
  \sdist(\stateset)
  & =
  \displaystyle
  \sum_{k_0=1}^{\infty}
  \sum_{m=1}^{\infty}
  \int_{\entsp^{\iauxsp(m)-\iauxsp(k_0)}}
  \int_{\nauxsp{k_0}}
  \int_{\entsp^{\iparsp(m)-\iparsp(k_0)}}
  \int_{\traces}
  \int_{\nsupport{\tree}{\iparsp{(k_0)}}} \\
  &
  \quad
  [(\zip(\traceb_1,\traceb_2)\concat\entb,
  \auxa_1\concat\auxa_2) \in \stateset \cap \validstates_m]
  \cdot
  \frac{1}{Z} \tree(\traceb_1) \cdot
  \nkernelpdf{k_0}
  (\zip(\traceb_1, \traceb_2),\auxa_1)\
  \\
  & \displaystyle
  \quad
  {\measure{\traces}}{(\dif\traceb_1)}\
  {\measure{\traces}}{(\dif\traceb_2)}\
  {\measure{\entsp^{\iparsp(m)-\iparsp(k_0)}}}{(\dif\entb)}\
  {\measure{\nauxsp{k_0}}}{(\dif\auxa_1)}\
  {\measure{\entsp^{\iauxsp(m)-\iauxsp(k_0)}}}{(\dif\auxa_2)}
\end{align*}

We can now show that the state distribution $\sdist$
is indeed a probability measure and
the set of valid states almost surely covers
all states w.r.t.~the state distribution.

\begin{proposition}
  \label{prop: valid states tends to 1}
  Assuming \cref{hass: integrable tree},
  \begin{enumerate}
    \item $\sdist(\states) = 1$; and

    \item $\sdist(\states \setminus
    \bigcup_{k=1}^n \validstates_k) \to 0$ as $n\to \infty$.
  \end{enumerate}
\end{proposition}

\begin{proof}
  \begin{enumerate}
    \item Consider the set $\validstates$ with
    the partition discussed above.
    \begin{align*}
      & \sdist(\states) \\
      & = \sdist(\validstates) \\
      & =
      \displaystyle
      \sum_{k_0=1}^{\infty}
      \sum_{m=1}^{\infty}
      \int_{\entsp^{\iauxsp(m)-\iauxsp(k_0)}}
      \int_{\nauxsp{k_0}}
      \int_{\entsp^{\iparsp(m)-\iparsp(k_0)}}
      \int_{\traces}
      \int_{\nsupport{\tree}{\iparsp{(k_0)}}} \\
      &
      \quad
      [(\zip(\traceb_1,\traceb_2)\concat\entb,
      \auxa_1\concat\auxa_2) \in \validstates_m]
      \cdot
      \frac{1}{Z} \tree(\traceb_1) \cdot
      \nkernelpdf{k_0}
      (\zip(\traceb_1, \traceb_2),\auxa_1)\
      \\
      & \displaystyle
      \quad
      {\measure{\traces}}{(\dif\traceb_1)}\
      {\measure{\traces}}{(\dif\traceb_2)}\
      {\measure{\entsp^{\iparsp(m)-\iparsp(k_0)}}}{(\dif\entb)}\
      {\measure{\nauxsp{k_0}}}{(\dif\auxa_1)}\
      {\measure{\entsp^{\iauxsp(m)-\iauxsp(k_0)}}}{(\dif\auxa_2)} \\
      & =
      \displaystyle
      \sum_{k_0=1}^{\infty}
      \int_{\nauxsp{k_0}}
      \int_{\traces}
      \int_{\nsupport{\tree}{\iparsp{(k_0)}}}
      \frac{1}{Z} \tree(\traceb_1) \cdot
      \nkernelpdf{k_0}
      (\zip(\traceb_1, \traceb_2),\auxa_1)\cdot
      \\
      &
      \quad
      \Big(
      \bigcup_{\ell_1=1}^{\infty}
      \bigcup_{\ell_2=1}^{\infty}
      \int_{\entsp^{\ell_1}}
      \int_{\entsp^{\ell_2}}
      [(\zip(\traceb_1,\traceb_2)\concat\entb,
      \auxa_1\concat\auxa_2) \in \validstates]\
      {\measure{\entsp^{\ell_2}}}{(\dif\entb)}\
      {\measure{\entsp^{\ell_1}}}{(\dif\auxa_2)}
      \Big)\\
      & \displaystyle
      \quad
      {\measure{\traces}}{(\dif\traceb_1)}\
      {\measure{\traces}}{(\dif\traceb_2)}\
      {\measure{\nauxsp{k_0}}}{(\dif\auxa_1)}\
      \\
      & =
      \displaystyle
      \sum_{k_0=1}^{\infty}
      \int_{\traces}
      \int_{\nsupport{\tree}{\iparsp{(k_0)}}}
      \frac{1}{Z} \tree(\traceb_1) \cdot
      \Big(
      \int_{\nauxsp{k_0}}
      \nkernelpdf{k_0}
      (\zip(\traceb_1, \traceb_2),\auxa_1)\
      {\measure{\nauxsp{k_0}}}{(\dif\auxa_1)}\
      \Big)
      \\
      & \displaystyle
      \quad
      {\measure{\traces}}{(\dif\traceb_1)}\
      {\measure{\traces}}{(\dif\traceb_2)}\
      \\
      & =
      \displaystyle
      \sum_{k_0=1}^{\infty}
      \int_{\traces}
      \int_{\nsupport{\tree}{\iparsp{(k_0)}}}
      \frac{1}{Z} \tree(\traceb_1)\
      {\measure{\traces}}{(\dif\traceb_1)}\
      {\measure{\traces}}{(\dif\traceb_2)}\
      \\
      & =
      \displaystyle
      \int_{\support{\tree}}
      \frac{1}{Z} \tree(\traceb_1)\
      {\measure{\traces}}{(\dif\traceb_1)}\
      = 1
    \end{align*}

    \item
    Since $\sdist$ is a probability distribution
    and $\sdist(\states \setminus \validstates) = 0$,
    the series
    $\sum_{n=1}^{\infty} \sdist(\validstates_n)$
    which equals $\sdist(\bigcup_{n=1}^{\infty} \validstates_n) = \sdist(\validstates) = 1$
    must converge.
    Hence
    $
    \sdist(\states \setminus {\bigcup_{k=1}^n \validstates_k})
    = \sdist(\validstates \setminus {\bigcup_{k=1}^n \validstates_k})
    = \sum_{i=n+1}^{\infty} \sdist(\validstates_i)
    \to 0
    $ as $n\to \infty$.
  \end{enumerate}
\end{proof}

\iffalse
Since the the parameter component of
all valid states must be supported by
the extension $\extree$ of the target density, the set of valid states can be written as
\[
  \validstates =
  \bigcup_{m=1}^{\infty}
  \{
    (\enta,\auxa) \in \validstates_m \mid
    \enta \in \support{\trunc{\len{\enta}}},
    \auxa \in \nauxsp{m}
  \}.
\]
\fi

\paragraph{Equivalent Program}
\label{sec: e np-imcmc}

Though the Hybrid NP-iMCMC algorithm (\cref{sec: np-imcmc algorithm}) traverses state,
it takes and returns \emph{samples}
on the trace space $\traces$.
Hence instead of asking
whether the state distribution $\sdist$ is
invariant against the Hybrid NP-iMCMC sampler directly,
we consider a similar program
which takes and returns \emph{states} and
prove the state distribution $\sdist$
is \emph{invariant} w.r.t.~this program.

% !TEX root = ./../icml2022.tex

\begin{figure}[t]
\begin{code}[caption={Code for the equivalent program},label={code: e np-imcmc}]
def eNPiMCMC(x*,v*):
  t0 = intersect(instance(x*),support(w))[0]    # find a valid state
  k0 = dim(t0)
  x0 = [(e, coin) if Type(e) in R else (normal, e) for e in t0]
  v0 = auxkernel[k0](x0)
  (x,v) = involution[k0](x0,v0)
  n = k0
  while not intersect(instance(x),support(w)):
    x0 = x0 + [(normal, coin)]*(indexX(n+1)-indexX(n))
    v0 = v0 + [(normal, coin)]*(indexY(n+1)-indexY(n))
    n = n + 1
    (x,v) = involution[n](x0,v0)
  (x,v) = involution[n](x0,v0)                 # accept/reject proposed state
  t = intersect(instance(x),support(w))[0]
  k = dim(t)
  return (x,v) if uniform < min{1, w(t)/w(t0) *
                               pdfauxkernel[k](proj((x,v),k))/
                                 pdfauxkernel[k0](proj((x0,v0),k0)) *
                               pdfpar[n](x)/pdfpar[n](x0) *
                               pdfaux[n](v)/pdfaux[n](v0) *
                               absdetjacinv[n](x0,v0)}
           else (x0,v0)
\end{code}
\end{figure}

Consider the program \codein{eNPiMCMC}
in \cref{code: e np-imcmc}.
It is similar to
\codein{NPiMCMC} (\cref{code: np-imcmc})
syntactically except
it takes and returns states instead of traces, and
has two additional lines (Lines 2 and 13).
Hence, it is easy to deduce from \cref{lemma: np-imcmc is ast}
that \codein{eNPiMCMC} almost surely terminates.

In \codein{eNPiMCMC}, we group the commands differently
and into two groups:
\begin{itemize}
  \item[Line 2-12] An initial valid state \codein{(x0,v0)}
  is constructed so that \codein{x0} and \codein{x*}
  have the same instance in the support of \codein{w}.

  \item[Line 14-22] A proposed state \codein{(x,v)}
  is computed and accepted/rejected.
\end{itemize}

\paragraph{Invariant Distribution}

Take a SPCF program
\codein{M} of type \codein{List(X*Y) -> List(X*Y)}
where the SPCF types \codein{X} and \codein{Y}
implements the parameter space $\parsp$ and
auxiliary space $\auxsp$ respectively.
We define the \defn{transition kernel} of
\codein{M} to be the kernel
$\transkernel{\codein{M}}:
\states \kernelto \states$ where
\begin{align*}
  \transkernel{\codein{M}}({\state},\stateset)
  :=
  \shortint{\inv{\valuefn{\codein{M(s)}}}
  (\stateset')}
  {\weightfn{\codein{M(s)}}}{\tmeasure}
  =
  \oper{\codein{M(s)}}(\stateset')
\end{align*}
where \codein{s} implements the state $\state$ and
$\stateset'$ is the set consisting of SPCF terms that
implements states in $\stateset$.
Intuitively,
$\transkernel{\codein{M}}({\state},\stateset)$
gives the probability that
the term $\codein{M}$ returns a state
in $\stateset$ given the current state $\state$.

\begin{proposition}
  \label{prop: ast transition kernel}
  Let \codein{M} be a SPCF term of type
  \codein{List(X*Y) -> List(X*Y)}.
  If \codein{M(s)} does not contain any scoring
  subterm and almost surely terminates
  for all SPCF terms \codein{s},
  then
  its transition kernel $\transkernel{\codein{M}}$ is probabilistic.
\end{proposition}

\begin{proof}
  Since the term \codein{M(s)}
  does not contain $\score{\placeholder}$ and
  terminates almost surely,
  by \cref{prop: AST SPCF term gives probability measure}
  its value measure must be probabilistic.
  Hence
  $
  \transkernel{\codein{M}}(\codein{s},\states)
  = \oper{\codein{M(s)}}(\states')
  = \oper{\codein{M(s)}}(\closedvalues)
  = 1
  $.
\end{proof}

We say a distribution $\mu$ on states $\states$ is \defn{invariant}
w.r.t.~a almost surely terminating
SPCF program \codein{M} of
type \codein{List(X*Y) -> List(X*Y)}
if
$\mu$ is not altered after applying $\codein{M}$, formally
$\expandint{\states}
{\transkernel{\codein{M}}(\state,\stateset)}
{\mu}{\state} = \mu(\stateset)$.

We now prove that
\codein{eNPiMCMC} preserves the state distribution
$\sdist$ stated in \cref{sec: state distribution} by
considering the transition kernels
given by the two steps in \codein{eNPiMCMC}
given in \cref{sec: e np-imcmc}:
find a valid state (Lines 2-12) and
accept/reject the computed proposed state (Lines 13-22).

\paragraph{Finding a Valid State}

Assuming the initial state \codein{(x*,v*)} is valid,
\codein{eNPiMCMC} (Lines 2-12) aims to construct a \emph{valid} state \codein{(x0,v0)}
where \codein{x*} and \codein{x0}
share the same instance \codein{t0}
that is in the support of the density \codein{w}.

To do this, it first finds the instance
\codein{t0} of \codein{x*}
which is in the support of \codein{w} (Line 2).
Say the dimension of \codein{t0} is \codein{k0} (Line 3).
It then forms a \codein{k0}-dimensional
state \codein{(x0,v0)} by
sampling partners \codein{t}
for each value in the trace \codein{t0}
to form a \codein{k0}-dimensional
parameter variable \codein{x0} (Line 4); and
drawing a \codein{k0}-dimensional
auxiliary variable from
\codein{auxkernel[k0](x0)} (Line 5).
Say \codein{v} is the auxiliary value
drawn.
Then, the \codein{k0}-dimensional state
can be written as
\codein{(zip(t0,t),v)}.

Note that the \codein{k0}-dimensional state
\codein{(zip(t0,t),v)} might
\emph{not} be valid.
In which case, it repeatedly appends
\codein{zip(t0,t)} and \codein{v}
with entropies \codein{(normal, coin)}
until the resulting state is valid (Lines 6-12).
Say \codein{y} and \codein{u} are the entropy
vectors drawn
for the parameter and auxiliary variables
respectively.
Then the resulting state can be written as
\codein{(zip(t0,t)+y, v+u)}.

The transition kernel of Lines 2-12
can be expressed
\begin{align*}
  & \transkernel{1}((\enta^*,\auxa^*),\stateset) \\
  & :=
  \sum_{n=1}^{\infty}
  \int_{\entsp^{\iauxsp(n)-\iauxsp(k_0)}}
  \int_{\entsp^{\iparsp(n)-\iparsp(k_0)}}
  \int_{\nauxsp{k_0}}
  \int_{\traces}\
  [(\zip(\traceb_0, \traceb) \concat\entb ,\auxa \concat \auxb) \in \stateset \cap \validstates_n] \cdot
  \nkernelpdf{k_0}(\zip(\traceb_0, \traceb),\auxa)\\
  & \quad
  {\measure{\traces}}{(\dif\traceb)}\
  {\measure{\nauxsp{k_0}}}{(\dif\auxa)}\
  {\measure{\entsp^{\iparsp(n)-\iparsp(k_0)}}}{(\dif\entb)}\
  {\measure{\entsp^{\iauxsp(n)-\iauxsp(k_0)}}}{(\dif\auxb)}\
\end{align*}
if $(\enta^*,\auxa^*) \in \validstates$ and
$\traceb_0 \in \instances{\enta^*} \cap \support{\tree}$
has some dimension $k_0\in\Nat$; and $0$ otherwise.

\begin{remark}
  Recall
  $\zip(\ell_1, \ell_2) :=
  [(\seqindex{\ell_1}{1}, \seqindex{\ell_2}{1}),
  (\seqindex{\ell_1}{2}, \seqindex{\ell_2}{2}), \dots,
  (\seqindex{\ell_1}{n}, \seqindex{\ell_2}{n})] \in (L_1\times L_2)^n $
  for any vectors $\ell_1 \in L_1^{n_1}$ and
  $\ell_2\in L_2^{n_2}$ with
  $n := \min\{n_1, n_2\}$.
  Here we extend the definition to lists
  $\ell_1, \ell_2 \in (L_1 \cup L_2)^n$ such that
  either $(\seqindex{\ell_1}{i}, \seqindex{\ell_2}{i})$ or $(\seqindex{\ell_2}{i}, \seqindex{\ell_1}{i})$ is in $L_1 \times L_2$
  for all $i = 1, \dots, n$.
  Then, we write
  $\zip(\ell_1, \ell_2) $ for the list of pairs in $L_1 \times L_2$.
\end{remark}

\begin{proposition}
  \label{prop: step 1 is probabilistic}
  Assuming \cref{hass: almost surely terminating tree},
  $\transkernel{1}((\enta_0,\auxa_0),\validstates) = 1$
  for all $(\enta_0,\auxa_0) \in \validstates$.
\end{proposition}

\begin{proof}
  Since Lines 2-12 in \codein{eNPiMCMC} can be
  described by a closed SPCF term
  that
  does not contain $\score{\placeholder}$ and
  terminates almost surely.
  By \cref{prop: ast transition kernel},
  its transition kernel is probabilistic.
  Moreover, as this term always return
  a valid state, we have
  $\transkernel{1}((\enta_0,\auxa_0),\validstates)
  = \transkernel{1}((\enta_0,\auxa_0),\states)
  = 1$.
\end{proof}

\begin{proposition}
  \label{prop: step 1 is invariant}
  Assuming \cref{hass: integrable tree,hass: almost surely terminating tree},
  the state distribution $\sdist$ is invariant against
  Lines 2-12 in \codein{eNPiMCMC}.
\end{proposition}

\begin{proof}
  We aim to show:
  $\expandint{\states}
  {\transkernel{1}((\enta^*,\auxa^*),\stateset)}
  {\sdist}{(\enta^*,\auxa^*)}
  = \sdist(\stateset)$ for any measurable set
  $\stateset \in \algebra{\states}$.
  (Changes are highlighted for readability.)
  \begin{calculation}
    \displaystyle
    \expandint{\mathhl{\states}}{\transkernel{1}((\enta^*,\auxa^*),\stateset)}{\sdist}{(\enta^*,\auxa^*)}
    \step[=]{
      $\transkernel{1}((\enta^*,\auxa^*),\stateset) = 0$ for all $(\enta^*,\auxa^*) \not\in \validstates$
    }
    \displaystyle
    \mathhl{\int_{\validstates}}\ {\transkernel{1}(\mathhl{(\enta^*,\auxa^*)},\stateset)}\
    \mathhl{{\sdist}{(\dif(\enta^*,\auxa^*))}}
    \step[=]{
      Writing
      $(\enta^*, \auxa^*)$ as
      $(\zip(\traceb_1,\traceb_2)\concat\entb,
      \auxa_1\concat\auxa_2) \in \validstates_m$
      where \\
      $\traceb_1 \in \nsupport{\tree}{\iparsp{(k_0)}},
      \traceb_2 \in \traces,
      \entb \in \entsp^{\iparsp(m)-\iparsp(k_0)},
      \auxa_1 \in \nauxsp{k_0},
      \auxa_2 \in \entsp^{\iauxsp(m)-\iauxsp(k_0)}, m, k_0 \in \Nat$
    }
    \displaystyle
    \sum_{k_0=1}^{\infty}
    \sum_{m=1}^{\infty}
    \int_{\entsp^{\iauxsp(m)-\iauxsp(k_0)}}
    \int_{\nauxsp{k_0}}
    \int_{\entsp^{\iparsp(m)-\iparsp(k_0)}}
    \int_{\traces}
    \int_{\nsupport{\tree}{\iparsp{(k_0)}}}
    \mathhl{\transkernel{1}((\zip(\traceb_1,\traceb_2)\concat\entb,
    \auxa_1\concat\auxa_2),\stateset)}\
    \\
    [(\zip(\traceb_1,\traceb_2)\concat\entb,
    \auxa_1\concat\auxa_2) \in \validstates_m]
    \cdot
    \displaystyle\frac{1}{Z} \tree(\traceb_1) \cdot
    \nkernelpdf{k_0}
    (\zip(\traceb_1, \traceb_2),\auxa_1)\
    \\
    {\measure{\traces}}{(\dif\traceb_1)}\
    {\measure{\traces}}{(\dif\traceb_2)}\
    {\measure{\entsp^{\iparsp(m)-\iparsp(k_0)}}}{(\dif\entb)}\
    {\measure{\nauxsp{k_0}}}{(\dif\auxa_1)}\
    {\measure{\entsp^{\iauxsp(m)-\iauxsp(k_0)}}}{(\dif\auxa_2)}
    \step[=]{
      Definition of $\transkernel{1}$
      on $(\zip(\traceb_1,\traceb_2)\concat\entb,
      \auxa_1\concat\auxa_2) \in \validstates$
      where $\traceb_1 \in \nsupport{\tree}{\iparsp{(k_0)}}$
    }
    \displaystyle
    \sum_{k_0=1}^{\infty}
    \sum_{m=1}^{\infty}
    \int_{\entsp^{\iauxsp(m)-\iauxsp(k_0)}}
    \int_{\nauxsp{k_0}}
    \int_{\entsp^{\iparsp(m)-\iparsp(k_0)}}
    \int_{\traces}
    \int_{\nsupport{\tree}{\iparsp{(k_0)}}}
    \\
    \quad\bigg(
      \displaystyle\sum_{n=1}^{\infty}
      \int_{\entsp^{\iauxsp(n)-\iauxsp(k_0)}}
      \int_{\entsp^{\iparsp(n)-\iparsp(k_0)}}
      \int_{\nauxsp{k_0}}
      \int_{\traces}\\
      \qquad
      [(\zip(\traceb_1, \traceb') \concat\entb' ,\auxa' \concat \auxb') \in \stateset \cap \validstates_n] \cdot
      \nkernelpdf{k_0}(\zip(\traceb_1, \traceb'),\auxa')\\
      \qquad
      {\measure{\traces}}{(\dif\traceb')}\
      {\measure{\nauxsp{k_0}}}{(\dif\auxa')}\
      {\measure{\entsp^{\iparsp(n)-\iparsp(k_0)}}}{(\dif\entb')}\
      {\measure{\entsp^{\iauxsp(n)-\iauxsp(k_0)}}}{(\dif\auxb')}\
    \bigg)
    \\
    [(\zip(\traceb_1,\traceb_2)\concat\entb,
    \auxa_1\concat\auxa_2) \in \validstates_m]
    \cdot
    \displaystyle\frac{1}{Z} \tree(\traceb_1) \cdot
    \nkernelpdf{k_0}
    (\zip(\traceb_1, \traceb_2),\auxa_1)\
    \\
    {\measure{\traces}}{(\dif\traceb_1)}\
    {\measure{\traces}}{(\dif\traceb_2)}\
    {\measure{\entsp^{\iparsp(m)-\iparsp(k_0)}}}{(\dif\entb)}\
    {\measure{\nauxsp{k_0}}}{(\dif\auxa_1)}\
    {\measure{\entsp^{\iauxsp(m)-\iauxsp(k_0)}}}{(\dif\auxa_2)}
    \step[=]{
      Tonelli's theorem
      as all measurable functions
      are non-negative
    }
    \sum_{k_0=1}^{\infty}
    \sum_{n=1}^{\infty}
    \int_{\entsp^{\iauxsp(n)-\iauxsp(k_0)}}
    \int_{\nauxsp{k_0}}
    \int_{\entsp^{\iparsp(n)-\iparsp(k_0)}}
    \int_{\traces}
    \int_{\nsupport{\tree}{\iparsp{(k_0)}}}
    \\
    \quad\Big(\mathhl{
      \sum_{m=1}^{\infty}
      \int_{\entsp^{\iauxsp(m)-\iauxsp(k_0)}}
      \int_{\entsp^{\iparsp(m)-\iparsp(k_0)}}
      \int_{\nauxsp{k_0}}
      \int_{\traces}}\\
      \qquad\mathhl{
      [(\zip(\traceb_1,\traceb_2)\concat\entb,\auxa_1\concat\auxa_2) \in \validstates_m]
      \cdot
      \nkernelpdf{k_0}(\zip(\traceb_1, \traceb_2),\auxa_1)}\
      \\
      \qquad\mathhl{
      {\measure{\traces}}{(\dif\traceb_2)}\
      {\measure{\nauxsp{k_0}}}{(\dif\auxa_1)}\
      {\measure{\entsp^{\iparsp(m)-\iparsp(k_0)}}}{(\dif\entb)}\
      {\measure{\entsp^{\iauxsp(m)-\iauxsp(k_0)}}}{(\dif\auxa_2)}}
    \Big)
    \\
    [(\zip(\traceb_1, \traceb') \concat\entb' ,\auxa' \concat \auxb') \in \stateset \cap \validstates_n] \cdot
    \displaystyle\frac{1}{Z} \tree(\traceb_1) \cdot
    \nkernelpdf{k_0}(\zip(\traceb_1, \traceb'),\auxa')
    \\
    {\measure{\traces}}{(\dif\traceb_1)}\
    {\measure{\traces}}{(\dif\traceb')}\
    {\measure{\entsp^{\iparsp(n)-\iparsp(k_0)}}}{(\dif\entb')}\
    {\measure{\nauxsp{k_0}}}{(\dif\auxa')}\
    {\measure{\entsp^{\iauxsp(n)-\iauxsp(k_0)}}}{(\dif\auxb')}\
    \step[=]{
      Definition of $\transkernel{1}$
      on $(\zip(\traceb_1,\traceb_2)\concat\entb,\auxa_1\concat\auxa_2) \in \validstates$
      where $\traceb_1 \in \nsupport{\tree}{\iparsp{(k_0)}}$
    }
    \displaystyle
    \sum_{k_0=1}^{\infty}
    \sum_{n=1}^{\infty}
    \int_{\entsp^{\iauxsp(n)-\iauxsp(k_0)}}
    \int_{\nauxsp{k_0}}
    \int_{\entsp^{\iparsp(n)-\iparsp(k_0)}}
    \int_{\traces}
    \int_{\nsupport{\tree}{\iparsp{(k_0)}}}\
    \mathhl{\transkernel{1}((\zip(\traceb_1,\traceb_2)\concat\entb,\auxa_1\concat\auxa_2), \validstates)}
    \\
    [(\zip(\traceb_1, \traceb') \concat\entb' ,\auxa' \concat \auxb') \in \stateset \cap \validstates_n] \cdot
    \displaystyle\frac{1}{Z} \tree(\traceb_1) \cdot
    \nkernelpdf{k_0}(\zip(\traceb_1, \traceb'),\auxa')
    \\
    {\measure{\traces}}{(\dif\traceb_1)}\
    {\measure{\traces}}{(\dif\traceb')}\
    {\measure{\entsp^{\iparsp(n)-\iparsp(k_0)}}}{(\dif\entb')}\
    {\measure{\nauxsp{k_0}}}{(\dif\auxa')}\
    {\measure{\entsp^{\iauxsp(n)-\iauxsp(k_0)}}}{(\dif\auxb')}\
    \step[=]{
      By \cref{prop: step 1 is probabilistic},
      $\transkernel{1}((\zip(\traceb_1,\traceb_2)\concat\entb,\auxa_1\concat\auxa_2),\validstates)
      % = \transkernel{1}((\zip(\traceb_1,\traceb_2)\concat\entb,\auxa_1\concat\auxa_2),\states)
      = 1$
    }
    \displaystyle
    \sum_{k_0=1}^{\infty}
    \sum_{n=1}^{\infty}
    \int_{\entsp^{\iauxsp(n)-\iauxsp(k_0)}}
    \int_{\nauxsp{k_0}}
    \int_{\entsp^{\iparsp(n)-\iparsp(k_0)}}
    \int_{\traces}
    \int_{\nsupport{\tree}{\iparsp{(k_0)}}}\
    \\
    [(\zip(\traceb_1, \traceb') \concat\entb' ,\auxa' \concat \auxb') \in \stateset \cap \validstates_n] \cdot
    \displaystyle\frac{1}{Z} \tree(\traceb_1) \cdot
    \nkernelpdf{k_0}(\zip(\traceb_1, \traceb'),\auxa')
    \\
    {\measure{\traces}}{(\dif\traceb_1)}\
    {\measure{\traces}}{(\dif\traceb')}\
    {\measure{\entsp^{\iparsp(n)-\iparsp(k_0)}}}{(\dif\entb')}\
    {\measure{\nauxsp{k_0}}}{(\dif\auxa')}\
    {\measure{\entsp^{\iauxsp(n)-\iauxsp(k_0)}}}{(\dif\auxb')}\
    \step[=]{
      Writing
      $(\enta^*, \auxa^*)\in \stateset \cap \validstates_n$ as
      $(\zip(\traceb_1, \traceb') \concat\entb' ,\auxa' \concat \auxb')$
      where \\
      $\traceb_1 \in \nsupport{\tree}{\iparsp{(k_0)}},
      \traceb' \in \traces,
      \entb' \in \entsp^{\iparsp(m)-\iparsp(k_0)},
      \auxa' \in \nauxsp{k_0},
      \auxb' \in \entsp^{\iauxsp(m)-\iauxsp(k_0)},
      n, k_0 \in \Nat$
    }
    \sdist(\stateset)
  \end{calculation}
\end{proof}

\paragraph{Accept/Reject Proposed State}

After constructing a valid state \codein{(x0,v0)}, say of dimension \codein{n},
\codein{eNPiMCMC} traverses the state space
via \codein{involution[n]} to obtain a proposal state \codein{(x,v)} (Line 13).
By \cref{prop: valid states are preserved by involutions},
\codein{(x,v)} must also be a
\codein{n}-dimensional valid state.
Say it has an instance \codein{t}
of dimension \codein{k} in the support
of \codein{w},
then (Line 14-22)
\codein{(x,v)} is accepted with
probability
\begin{align*}
  \accept(\enta_0, \auxa_0)
  & := \min \Big\{1,
    \displaystyle
    \frac
      {
        \tree(\traceb) \cdot
        \nkernelpdf{k}(\proj{n}{k}(\enta,\auxa)) \cdot
        \nparpdf{n}(\enta) \cdot
        \nauxpdf{n}(\auxa)
      }
      {
        \tree(\traceb_0) \cdot
        \nkernelpdf{k_0}(\proj{n}{k_0}(\enta_0,\auxa_0)) \cdot
        \nparpdf{n}(\enta_0) \cdot
        \nauxpdf{n}(\auxa_0)
      } \cdot
    \abs{\det{\big(\grad{\ninvo{n}}(\enta_0,\auxa_0)\big)}}
  \Big\} \\
  & = \min \Big\{1,
    \displaystyle
    \frac
      {
        \spdf(\enta,\auxa) \cdot
        \nparpdf{n}(\enta) \cdot
        \nauxpdf{n}(\auxa)
      }
      {
        \spdf(\enta_0,\auxa_0) \cdot
        \nparpdf{n}(\enta_0) \cdot
        \nauxpdf{n}(\auxa_0)
      } \cdot
      \abs{\det{\big(\grad{\ninvo{n}}(\enta_0,\auxa_0)\big)}}
  \Big\}.
\end{align*}

The transition kernel for Line 13-22 can be expressed as
\[
  \transkernel{2}(\state,\stateset)
  :=
  \accept(\state)\cdot[\ninvo{n}(\state) \in \stateset] +
  (1-\accept(\state))\cdot [\state \in \stateset]
\]
if {$\state \in \validstates_n$}
for some $n\in\Nat$;
and $0$ otherwise.

To show that the state distribution $\sdist$
is invariant against $\transkernel{2}$,
we consider a partition of the set of valid states.
Let $\stateij{\boolveca}{\boolvecb}{n}$ be the set
of $n$-dimensional valid states
where $\boolveca$
is the list of
boolean values in all
$\state\in\stateij{\boolveca}{\boolvecb}{n}$ and
$\ninvo{n}$ maps $\state$
to a (valid) state with boolean values given by
the list $\boolvecb$.
Note that both lists $\boolveca, \boolvecb$
of booleans must be of length
$\tilde{n} := \iparsp{(n)}+\iauxsp{(n)}$.
Formally,
\begin{align*}
  \stateij{\boolveca}{\boolvecb}{n} :=
  \{\state \in \validstates_n
    & \mid
    \state
    = \zip(\realveca, \boolveca)
    \text{ and }
    \ninvo{n}(\state)
    = \state'
    = \zip(\realvecb, \boolvecb)
    % \\
    % & \quad
    \text{ for some }
    \realveca, \realvecb \in \Real^{\tilde{n}}
  \}.
\end{align*}
Then, the set $\validstates$ of valid states
can written as
$\bigcup
\set{
  \stateij{\boolveca}{\boolvecb}{n}\mid
  \boolveca, \boolvecb \in \Bool^{\tilde{n}} \text{ and }
  n\in\Nat
}.$

\begin{proposition} \label{prop: acceptance ratio}
  Assuming \allhass{},
  for $n\in\Nat$,
  $\state \in \validstates_n$
  and $\state' = \ninvo{n}(\state) $, we have
  \[
    \accept(\state')\cdot
    \spdf'(\state')\cdot
    \abs{\det{\big(\grad{\ninvo{n}}(\state)\big)}}
    = \accept(\state) \cdot \spdf'(\state)
  \]
  where
  $\spdf'(\entc, \auxc) :=
  \spdf(\entc, \auxc)\cdot
  \nparpdf{m}(\entc) \cdot
  \nauxpdf{m}(\auxc)$
  for any $(\entc, \auxc) \in \states_m$.
\end{proposition}

\begin{proof}
  Let
  $\state \in \stateij{\boolveca}{\boolvecb}{n}$ where
  there are
  $\realveca, \realvecb \in \Real^{\tilde{n}}$,
  $\boolveca, \boolvecb \in \Bool^{\tilde{n}}$
  such that
  $\state
    = \zip(\realveca, \boolveca)$ and
  $\state' := \ninvo{n}(\state)
    = \zip(\realvecb, \boolvecb)$.
  Hence,
  taking the Jacobian determinant on
  both sides of the equation
  ${\ninvoij{\boolvecb}{\boolveca}{n}} \circ {\ninvoij{\boolveca}{\boolvecb}{n}} = \id$
  gives us
  \begin{equation} \label{eq: absdetjac}
    \abs{\det{\big(\grad{\ninvo{n}}(\state')\big)}}
    =
    \abs{\det{\big(\grad{\ninvoij{\boolvecb}{\boolveca}{n}}(\realvecb)\big)}}
    =
    \frac{1}{\abs{\det{\big(\grad{\ninvoij{\boolveca}{\boolvecb}{n}}(\realveca)\big)}}}
    =
    \frac{1}{\abs{\det{(\grad{\ninvo{n}}(\state))}}}.
  \end{equation}
  Moreover we can write
  the acceptance ratio
  in terms of $\spdf'$ as
  \[
  \accept(\state'') =
  \min \set{1,
    \displaystyle
    \frac
      {\spdf'(\ninvo{n}(\state''))}
      {\spdf'(\state'')} \cdot
    \abs{\det{\big(\grad{\ninvo{n}}(\state'')\big)}}
  }
  \text{ for any }
  \state'' \in \states_m.
  \]
  Hence
  given $\state' = \ninvo{n}(\state) $, we have
  \begin{align*}
    & \accept(\state')\cdot
      \spdf'(\state')\cdot
      \abs{\det{\big(\grad{\ninvo{n}}(\state)\big)}} \\
    % & = \accept(\state')\cdot
    %   \spdf'(\state')\cdot
    %   \abs{\det{\big(\grad{\ninvoij{\boolveca}{\boolvecb}{n}}(\realveca)\big)}}
    %   \tag{
    %     By definition of $\ninvo{n}$
    %   } \\
    & = \begin{cases}
      \displaystyle\frac
        {\spdf'(\state)}
        {\spdf'(\state')} \cdot
      \abs{\det{\big(\grad{\ninvo{n}}(\state')\big)}}\cdot
      % \abs{\det{\big(\grad{{\ninvoij{\boolvecb}{\boolveca}{n}}}(\realvecb)\big)}} \cdot
      \spdf'(\state')\cdot
      \abs{\det{(\grad{\ninvo{n}}(\state))}}
      % \abs{\det{\big(\grad{{\ninvoij{\boolveca}{\boolvecb}{n}}}(\realveca)\big)}}} \\
      & \text{if }
        \displaystyle\frac
        {\spdf'(\state)}
        {\spdf'(\state')} \cdot
      \abs{\det{\big(\grad{\ninvo{n}}(\state')\big)}}
      % \abs{\det{\big(\grad{{\ninvoij{\boolvecb}{\boolveca}{n}}}(\realvecb)\big)}}
      < 1 \\
      \spdf'(\state')\cdot
      \abs{\det{\big(\grad{\ninvo{n}}(\state)\big)}}
      % \abs{\det{\big(\grad{{\ninvoij{\boolveca}{\boolvecb}{n}}}(\realveca)\big)}}
      & \text{otherwise}
    \end{cases}
    \tag{$\state = \ninvo{n}(\state')$} \\
    & = \begin{cases}
      {\spdf'(\state)}
      & \text{if }
      \displaystyle\frac
        {\spdf'(\state')}
        {\spdf'(\state)}
        \cdot
        {\abs{\det{(\grad{\ninvo{n}}(\state))}}} > 1 \\
      \displaystyle\frac
        {\spdf'(\state')}
        {\spdf'(\state)}
        \cdot
        {\abs{\det{(\grad{\ninvo{n}}(\state))}}}
        \cdot
        \spdf'(\state)
      & \text{otherwise}
    \end{cases}
    \tag{By \cref{eq: absdetjac}} \\
    & = \accept(\state) \cdot \spdf'(\state)
  \end{align*}
\end{proof}

\begin{proposition}
  \label{prop: step 2 is invariant}
  Assuming \allhass{},
  the state distribution $\sdist$ is invariant
  against Line 13-22 in \codein{eNPiMCMC}.
\end{proposition}

\begin{proof}
  We aim to show:
  $\expandint{\states}{\transkernel{2}(\state,\stateset)}{\sdist}{\state} = \sdist(\stateset)$
  for all $\stateset \in \algebra{\states}$.

  Let $\state$ be a $n$-dimensional valid state
  and $\stateset \in \algebra{\states}$.
  Then we can write
  $\transkernel{2}(\state,\stateset)$
  as
  $[\state \in \stateset]
  + [\ninvo{n}(\state) \in \stateset]\cdot\accept(\state)
  - [\state \in \stateset]\cdot\accept(\state)$.
  Hence, it is enough to show that
  the integral
  of the second and third terms
  over all valid states are the same,
  i.e.
  \[
    \expandint{\validstates}
    {[\ninvo{n}(\state) \in \stateset]\cdot\accept(\state)}
    {\sdist}{\state}
    =
    \expandint{\validstates}
    {[\state \in \stateset]\cdot\accept(\state)}
    {\sdist}{\state}
  \]

  First we consider
  the valid states in
  $\stateij{\boolveca}{\boolvecb}{n}$ where
  $n\in\Nat$,
  $\boolveca, \boolvecb \in \Bool^{\tilde{n}}$
  and
  $\tilde{n} := \iparsp{(n)}+\iauxsp{(n)} $.
  These are $n$-dimensional valid states
  with boolean values given by $\boolveca$
  and are mapped by $\ninvo{n}$ to valid states
  with boolean values given by $\boolvecb$.
  Then we have
  $\inv{\zip(\placeholder,\boolvecb)}
  (\stateij{\boolvecb}{\boolveca}{n})
  = \ninvoij{\boolveca}{\boolvecb}{n}
  \big(\inv{\zip(\placeholder,\boolveca)}
  (\stateij{\boolveca}{\boolvecb}{n})\big)$
  where
  $\zip(\placeholder, \boolvecb):\Real^{\tilde{n}}
  \to \entsp^{\tilde{n}} $
  is a measurable function.
  % Let
  % $\ridij{\boolveca}{\boolvecb}{n}:
  % \states_n \to \Real^{\tilde{n}}$
  % be the function that
  % picks out the list of real components of a
  % $n$-dimensional state, formally
  % $\ridij{\boolveca}{\boolvecb}{n}(\state)
  % := \realveca$
  % if $\state = \zip(\realveca, \boolveca) $
  % for some $\boolveca \in \Bool^{\tilde{n}} $.
  % Then we have
  % $\ninvoij{\boolveca}{\boolvecb}{n}(\ridij{\boolveca}{\boolvecb}{n}(\stateij{\boolveca}{\boolvecb}{n}))
  % = \ridij{\boolvecb}{\boolveca}{n}(\stateij{\boolvecb}{\boolveca}{n})$.
  Writing $\spdf'(\entc, \auxc)$ for
  $\spdf(\entc, \auxc)\cdot
  \nparpdf{m}(\entc) \cdot
  \nauxpdf{m}(\auxc)$
  for any $(\entc, \auxc) \in \states_m$,
  we have
  \begin{align*}
    & \expandint{\stateij{\boolvecb}{\boolveca}{n}}{
      [\state \in \stateset] \cdot
      \accept(\state)
    }{\sdist}{\state} \\
    & = \expandint{\stateij{\boolvecb}{\boolveca}{n}}{
      [\state \in \stateset] \cdot
      \accept(\state) \cdot
      \spdf'(\state)
    }{\measure{\entsp^{\tilde{n}}}}{\state}
    \tag{Definition of $\sdist$} \\
    & = \expandint{\inv{\zip(\placeholder,\boolvecb)}
          (\stateij{\boolvecb}{\boolveca}{n})}{
      [\zip(\realveca,\boolvecb)\in \stateset]
      % [\realveca \in \ridij{\boolvecb}{\boolveca}{n}(\stateset)]
      \cdot
      \accept(\zip(\realveca,\boolvecb)) \cdot
      \spdf'(\zip(\realveca,\boolvecb))
    }{\measure{\Real^{\tilde{n}}}}{\realveca}
    \tag{$\zip(\placeholder,\boolvecb)_*\measure{\Real^{\tilde{n}}}
    = \measure{\entsp^{\tilde{n}}}$ on
    $\stateij{\boolvecb}{\boolveca}{n}$}
    \\
    & = \int_{\inv{\zip(\placeholder,\boolveca)}
          (\stateij{\boolveca}{\boolvecb}{n})}\
      [\zip(\ninvoij{\boolveca}{\boolvecb}{n}(\realvecb),
        \boolvecb)\in \stateset]
      % [\ninvoij{\boolveca}{\boolvecb}{n}(\realvecb) \in
        % \ridij{\boolvecb}{\boolveca}{n}(\stateset)]
      \cdot
      \accept(\zip(\ninvoij{\boolveca}{\boolvecb}{n}(\realvecb), \boolvecb)) \cdot
      \spdf'(\zip(\ninvoij{\boolveca}{\boolvecb}{n}(\realvecb), \boolvecb)) \cdot
      \abs{\det{\grad{\ninvoij{\boolveca}{\boolvecb}{n}(\realvecb) }}}\
    {\measure{\Real^{\tilde{n}}}}{(\dif\realvecb)}
    \tag{Change of variable where
      $\realveca = \ninvoij{\boolveca}{\boolvecb}{n}(\realvecb)$} \\
    & = \expandint{\inv{\zip(\placeholder,\boolveca)}
          (\stateij{\boolveca}{\boolvecb}{n})}{
      [\ninvo{n}(\zip(\realvecb,\boolveca))\in \stateset]
      % [\ninvoij{\boolveca}{\boolvecb}{n}(\realvecb) \in
      %   \ridij{\boolvecb}{\boolveca}{n}(\stateset)]
      \cdot
      \accept(\zip(\realvecb,\boolveca)) \cdot
      \spdf'(\zip(\realvecb,\boolveca))}
    {\measure{\Real^{\tilde{n}}}}{\realvecb}
    \tag{
      \cref{prop: acceptance ratio} as
      $\ninvo{n}(\zip(\realvecb,\boolveca))
      =
      \zip(\ninvoij{\boolveca}{\boolvecb}{n}(\realvecb), \boolvecb)$
      for
      $(\zip(\realvecb,\boolveca)) \in
      \stateij{\boolveca}{\boolvecb}{n}$
    } \\
    & = \expandint{\stateij{\boolveca}{\boolvecb}{n}}{
      [\ninvo{n}(\state) \in \stateset]
      \cdot
      \accept(\state) \cdot
      \spdf'(\state)}
    {\measure{\entsp^{\tilde{n}}}}{\state}
    \tag{$\zip(\placeholder,\boolveca)_*\measure{\Real^{\tilde{n}}}
    = \measure{\entsp^{\tilde{n}}}$ on
    $\stateij{\boolveca}{\boolvecb}{n}$} \\
    & = \expandint{\stateij{\boolveca}{\boolvecb}{n}}{
      [\ninvo{n}(\state) \in \stateset]
      \cdot
      \accept(\state)
    }{\sdist}{\state}
  \end{align*}

  Recall the set $\validstates$
  of all valid states
  can be written as
  $\bigcup\{\stateij{\boolveca}{\boolvecb}{n}
  \mid \boolveca, \boolvecb \in \Bool^{\tilde{n}}\text{ and }n\in\Nat
  \} $.
  Hence, we conclude our proof with
  \begin{align*}
    & \expandint{\validstates}{[\ninvo{n}(\state) \in \stateset] \cdot \accept(\state)}
    {\sdist}{\state}
    =
    \sum_{n=1}^{\infty}
      \sum_{\boolveca, \boolvecb \in \Bool^{\tilde{n}}}
      \expandint{\stateij{\boolveca}{\boolvecb}{n}}
      {[\ninvo{n}(\state) \in \stateset] \cdot \accept(\state) }{\sdist}{\state} \\
    & =
      \sum_{n=1}^{\infty}
      \sum_{\boolveca, \boolvecb \in \Bool^{\tilde{n}}}
      \expandint{\stateij{\boolvecb}{\boolveca}{n}}
      {[\state \in \stateset]\cdot \accept(\state) }{\sdist}{\state}
    = \expandint{\validstates}{[\state \in \stateset] \cdot \accept(\state)}
    {\sdist}{\state}.
  \end{align*}
\end{proof}

Since the transition kernel of
\codein{eNPiMCMC} is the composition of
$\transkernel{1}$ and $\transkernel{2}$
and both
$\transkernel{1}$ and $\transkernel{2}$ are
invariant against $\sdist$
(\cref{prop: step 1 is invariant,prop: step 2 is invariant}),
we deduce that
\codein{eNPiMCMC} preserves the state distribution $\sdist$.

\begin{restatable}[State Invariant]{lemma}{sinvariant}
  \label{lemma: e-np-imcmc invariant}
  $\sdist$ is the invariant distribution of the Markov chain generated by iterating \codein{eNPiMCMC}.
\end{restatable}

% \sinvariant*

\subsubsection{Marginalised Markov Chains}
\label{sec: Marginalised Markov Chains (np-imcmc)}
% !TEX root = ./../icml2022.tex

As discussed above, the Markov chain
$\set{\state_i}_{i\in\Nat}$
generated by iterating \codein{eNPiMCMC}
(which has invariant distribution $\sdist$
(\cref{lemma: e-np-imcmc invariant}))
has elements on the state space $\states$
and \emph{not} the trace space $\traces$.
The chain we are in fact interested in is
the \emph{marginalised} chain
$\set{\marg(\state_i)}_{i\in\Nat}$ where
the measurable function
$\marg:\validstates \to \traces$
takes a valid state $\state = (\enta,\auxa)$
and returns the instance of the
parameter variable $\enta$ that is in
the support of the target density function
$\tree$.

In this section we show that
this marginalised chain simulates the
target distribution $\tdist$.
Let $\transkernel{\codein{NPiMCMC}}:
\traces \kernelto {\traces}$ be a kernel
such that
\[
  \transkernel{\codein{NPiMCMC}}(\traceb, A)
  :=
  \begin{cases}
    \transkernel{\codein{eNPiMCMC}}(\state, \inv{\marg}(A))
    & \text{if }\traceb \in \support{\tree}
    \text{ and }\state \in \inv{\marg}(\set{\traceb})\\
    0 & \text{otherwise.}
  \end{cases}
\]
Comparing the commands of
\codein{NPiMCMC} and \codein{eNPiMCMC}
in \cref{code: np-imcmc,code: e np-imcmc},
we claim that $\transkernel{\codein{NPiMCMC}}$
is \emph{the} transition kernel of
\codein{NPiMCMC}.

\begin{proposition}
  \label{prop: properties of transk}
  We consider some basic properties of
  $\transkernel{\codein{NPiMCMC}}$.
  \begin{enumerate}
    \item $\transkernel{\codein{NPiMCMC}}$ is well-defined.

    \item $\transkernel{\codein{eNPiMCMC}}
    (\state,\inv{\marg}(\traceset))
    = \transkernel{\codein{NPiMCMC}}
    (\marg(\state),\traceset)$
    for all $\state \in \validstates$ and
    $\traceset \in \algebra{\traces}$.
  \end{enumerate}
\end{proposition}

\begin{proof}
  \begin{enumerate}
    \item Let $\traceb \in \support{\tree}$ and
    $\traceset \in \algebra{\traces}$.
    Say $\state,\state' \in \inv{\marg}(\set{\traceb})$.
    Since only the instance of the input
    state matters in \codein{eNPiMCMC}
    (\cref{code: e np-imcmc}),
    the value of $\transkernel{\codein{NPiMCMC}}(\traceb, \traceset)$
    given by $\state$ and $\state'$
    are the same, i.e.~$\transkernel{\codein{eNPiMCMC}}(\state, \inv{\marg}(\traceset))
    = \transkernel{\codein{eNPiMCMC}}(\state', \inv{\marg}(\traceset)).$

    \item Let $\state \in \validstates$ and
    $A \in \algebra{\traces}$. Then,
    $\transkernel{\codein{NPiMCMC}}
    (\marg(\state),A)
    = \transkernel{\codein{eNPiMCMC}}
    (\state',\inv{\marg}(A))$ for some
    $\state' \in \inv{\marg}(\set{\marg(\state)})$.
    Since $\state \in \inv{\marg}(\set{\marg(\state)})$,
    we have $\transkernel{\codein{eNPiMCMC}}
    (\state',\inv{\marg}(A))
    = \transkernel{\codein{eNPiMCMC}}
    (\state,\inv{\marg}(A))$.
  \end{enumerate}
\end{proof}

To show $\transkernel{\codein{NPiMCMC}}$
preserves the target distribution,
we consider a distribution $\sdist_n$
on each of the $n$-dimensional state space
$\nstates{n} := \nparsp{n}\times\nauxsp{n}$ with density
$\spdf_n$ (w.r.t.~$\measure{\nstates{n}}$)
given by
\[
  \spdf_n(\enta,\auxa) :=
  \begin{cases}
    \displaystyle\frac{1}{Z_n}\cdot
    \tree(\traceb) \cdot
    \nkernelpdf{k}(\proj{n}{k}(\enta, \auxa))
    & \text{if }
    \traceb \in \instances{\enta} \cap \support{\tree}
    \text{ has dimension } k \leq n \nonumber\\
    {0} &\qquad \text{otherwise}
  \end{cases}
\]
where $Z_n :=
\expandint{\traces}
{[\len{\traceb} \leq {\iparsp(n)}]
\cdot \tree(\traceb)}
{\measure{\traces}}{\traceb}$.
Notice that $Z_n\cdot \spdf_n$ and $Z \cdot \spdf$ are the same,
except on non-valid states.
The following proposition shows how the
state distribution $\sdist$ can be
represented using $\sdist_n$.

\begin{proposition}
  \label{prop: n-dim state dist}
  Let $n\in\Nat$.
  \begin{enumerate}
    \item $\sdist_n$ is a probability measure.

    \item For $k \leq n$,
    $Z_k\cdot\sdist_k
    = Z_n\cdot {\proj{n}{k}}_*\sdist_n$
    on $\validstates_k$.

    \item
    Let $g^{(n)}:\nstates{n}
    \partialto {\bigcup_{k=1}^n \validstates_k}$
    be the partial measurable function that
    returns the projection of the input state
    that is valid, if it exists. Formally,
    $g^{(n)}(\state) = \proj{n}{k}(\state)$
    if
    $\proj{n}{k}(\state) \in \validstates_k$.
    Then
    $Z\cdot\sdist = Z_n\cdot g^{(n)}_*\sdist_n$ on ${\bigcup_{k=1}^n \validstates_k}$.
  \end{enumerate}
\end{proposition}

\begin{proof}
  \begin{enumerate}
    \item Consider $\sdist_n(\nstates{n})$,
    \begin{align*}
      \sdist_n(\nstates{n})
      & = \sum_{k=1}^n \
      \int_{\nstates{n}}\
        [\traceb\in\instances{\enta}]\cdot
        [\len{\traceb}=\iparsp(k)]\cdot
        \frac{1}{Z_n}\cdot\tree(\traceb)\cdot
        \nkernelpdf{k}(\proj{n}{k}(\enta, \auxa))\
      {\measure{\nstates{n}}}{(\dif (\enta, \auxa))} \\
      & = \sum_{k=1}^n \
      \int_{\nsupport{\tree}{\iparsp(k)}}\
      \int_{\traces}\
      \int_{\nauxsp{k}}\
        \frac{1}{Z_n}\cdot\tree(\traceb)\cdot
        \nkernelpdf{k}(\zip(\traceb,\traceb'), \auxa'))\
      {\measure{\nauxsp{k}}}{(\dif \auxa')}\
      {\measure{\traces}}{(\dif \traceb')}\
      {\measure{\traces}}{(\dif \traceb)} \\
      & = \sum_{k=1}^n \
      \int_{\nsupport{\tree}{\iparsp(k)}}\
      \int_{\traces}\
        \frac{1}{Z_n}\cdot\tree(\traceb)\
      {\measure{\traces}}{(\dif \traceb')}\
      {\measure{\traces}}{(\dif \traceb)}
      \tag{$\nkernel{k}$ is a probability kernel} \\
      & = \int_{\traces}\
        [\len{\traceb}\leq \iparsp(n)]\cdot
        \frac{1}{Z_n}\cdot\tree(\traceb)\
      {\measure{\traces}}{(\dif \traceb)}
      = 1
    \end{align*}

    \item Let $\stateset \in \algebra{\states}$
    where $ \stateset \subseteq {\validstates_k}$.
    Hence
    $Z_k \cdot \tpdf_k(\state)
    = Z_k \cdot \tpdf_k(\state')$
    if $\state \in \stateset$
    and $\state = \proj{n}{k}(\state')$.
    Then,
    \begin{align*}
      & Z_n \cdot ({\proj{n}{k}}_*\sdist_n)(\stateset) \\
      & = Z_n \cdot \sdist_n(\inv{{\proj{n}{k}}}{(\stateset)}) \\
      & = \expandint{\nstates{n}}
      {[\proj{n}{k}(\state') \in \stateset] \cdot
      Z_n\cdot \spdf_n(\state')}
      {\measure{\nstates{n}}}{(\state')} \\
      & =
      \int_{\nstates{k}}\
      [(\state) \in \stateset] \cdot
      Z_k\cdot
      \spdf_k(\state) \cdot
      \ {\measure{\nstates{k}}}{(\dif(\state))} \\
      & =
      Z_k\cdot\sdist_k(\stateset)
    \end{align*}
    \label{item: partial pointwise auxiliary kernel is used}

    \item Let $\stateset \in \algebra{\states}$
    where $\stateset \subseteq
    {\bigcup_{k=1}^n \validstates_k}$.
    Then, $Z\cdot\tpdf(\state) = Z_k\cdot\tpdf_k(\state)$
    for all $\state \in \stateset \cap \validstates_k$.
    \begin{align*}
      Z\cdot \sdist(\stateset)
      & = \expandint{\stateset}
      {[\state \in \validstates]\cdot
      Z \cdot
      \spdf(\state)}
      {\measure{\states}}{\state} \\
      & = \sum_{k=1}^n
      \expandint{\stateset}
      {[\state \in \validstates_k]\cdot
      Z \cdot
      \spdf(\state)}
      {\measure{\nstates{k}}}{\state} \\
      & = \sum_{k=1}^n
      \expandint{\stateset}
      {[\state \in \validstates_k]\cdot
      Z_k \cdot
      \spdf_k(\state)}
      {\measure{\nstates{k}}}{\state} \\
      & = \sum_{k=1}^n Z_k \cdot \sdist_k(\stateset\cap \validstates_k) \\
      & = Z_n \sum_{k=1}^n {\proj{n}{k}}_* \sdist_n (\stateset\cap \validstates_k) \tag{i} \\
      & = Z_n \cdot
      \sdist_n (\bigcup_{k=1}^n
      \set{ \state \in \nstates{n}
      \mid \proj{n}{k}(\state) \in \stateset\cap \validstates_k}) \\
      & = Z_n \cdot g^{(n)}_*\sdist_n (\stateset).
    \end{align*}
  \end{enumerate}
\end{proof}

\begin{restatable}[Invariant]{lemma}{invariant}
  \label{lemma: marginalised distribution is the target distribution}
  Assuming \allhass{},
  $\tdist$ is the invariant distribution of
  the Markov chain generated by
  iterating the Hyrbid NP-iMCMC algorithm (\cref{sec: np-imcmc algorithm}).
\end{restatable}

% \invariant*

\lo{Check: $\nu$ is the unique(?) invariant distribution of the induced Markov chain.}

%\lo{Note the standard result about existence and uniqueness of invariant distribution:}
\begin{proof}
  Assuming
  (1) $\tdist = \marg_*\sdist$ on $\traces$ and
  (2) $\tmeasure = \marg_*\smeasure$ on $\support{w}$, we have
  for any $A \in \algebra{\traces}$,
  \begin{align*}
    \tdist(A) & = \marg_*\sdist(A) \tag{1} \\
    & = \expandint{\states}{\transkernel{\codein{eNPiMCMC}}(\state, \inv{\marg}(A))}{\smeasure}{\state} \tag{\cref{lemma: e-np-imcmc invariant}} \\
    & = \expandint{\validstates}{\transkernel{\codein{eNPiMCMC}}(\state, \inv{\marg}(A))}{\smeasure}{\state} \\
    & = \expandint{\validstates}{\transkernel{\codein{NPiMCMC}}(\marg(\state), A)}{\smeasure}{\state} \tag{\cref{prop: properties of transk}.ii} \\
    & = \expandint{\support{\tree}}{\transkernel{\codein{NPiMCMC}} (\traceb, A)}{\marg_*\smeasure}{\traceb} \\
    & = \expandint{\support{\tree}}{\transkernel{\codein{NPiMCMC}} (\traceb, A)}{\tmeasure}{\traceb} \tag{2} \\
    & = \expandint{\traces}{\transkernel{\codein{NPiMCMC}} (\traceb, A)}{\tmeasure}{\traceb}.
  \end{align*}
  It is enough to show (1) and (2).

  \begin{enumerate}
    \item
    Let $A \in \algebra{\traces}$ where
    $A \subseteq \nsupport{\tree}{\iparsp(n)} $ and
    $\delta > 0$.
    Then partitioning $\inv{\marg}(\traceset)$ using $\validstates_k$, we have {for sufficiently large $m$},
    \begin{align*}
      & \marg_*\sdist(\traceset) \\
      & = \sdist\left(\bigcup_{k=1}^m \inv{\marg}(\traceset)\cap \validstates_k\right)
        + \sdist\left(\bigcup_{k=m+1}^\infty \inv{\marg}(\traceset)\cap \validstates_k\right) \\
      & < \frac{Z_m}{Z} \cdot
        g^{(m)}_*\sdist_m\left(\bigcup_{k=1}^m \inv{\marg}(\traceset)\cap \validstates_k\right) + \delta
      \tag{\cref{prop: properties of transk}.iii,
      \cref{prop: n-dim state dist}.ii} \\
      & \leq \frac{Z_m}{Z} \cdot \sdist_m(
      \set{(\zip(\traceb, \traceb')\concat \entb, \auxa)
      \mid
        \traceb \in \traceset,
        \traceb'\in \traces,
        \entb \in \entsp^{\iparsp(m)-\iparsp(n)},
        \auxa \in \nauxsp{m}}
      ) + \delta \\
      & = \displaystyle\frac{1}{Z}
      \int_{\traceset}\
      \int_{\traces}\
      \int_{\entsp^{\iparsp(m)-\iparsp(n)}}\
      \tree(\traceb)\cdot \\
      & \qquad
        \big(\expandint{\nauxsp{m}}
        {\nkernelpdf{n}(\proj{m}{n}
          (\zip(\traceb,\traceb')\concat\entb ,\auxa))}
        {\measure{\nauxsp{m}}}{\auxa}\big)\\
      & \quad
      {\measure{\entsp^{\iparsp(m)-\iparsp(n)}}}{(\dif\entb)}
      {\measure{\traces}}{(\dif\traceb')}
      {\measure{\traces}}{(\dif\traceb)}
      + \delta \\
      % & = \displaystyle\frac{1}{Z}
      % \int_{\set{\zip(\traceb, \traceb')\concat\entb
      %   \mid \traceb \in \traceset,
      %   \traceb'\in \traces,
      %   \entb \in \entsp^{\iparsp(m)-\iparsp(n)}}}\
      %   \tree(\traceb)\cdot \\
      % & \quad
      %   \big(\expandint{\nauxsp{m}}
      %   {\nkernelpdf{n}(\proj{m}{n}(\enta,\auxa))}
      %   {\measure{\nauxsp{m}}}{\auxa}\big)\
      % {\measure{\nparsp{m}}}{(\dif\enta)}
      % + \delta \\
      & = \frac{1}{Z} \expandint{\traceset}{\tree(\traceb)}{\measure{\traces}}{\traceb}
      + \delta
      \tag{$\nkernel{n}$ is a probability kernel}
      \\
      & = \tdist(\traceset) + \delta.
    \end{align*}
    For any measurable set $\traceset \in \Sigma_{\traces}$,
    we have
    $
    \marg_*\sdist(\traceset)
    = \marg_*\sdist(\traceset \cap \support{\tree})
    = \sum_{n=1}^\infty \marg_*\sdist
      (\traceset \cap \nsupport{\tree}{\iparsp(n)})
    \leq \sum_{n=1}^\infty
      \tdist(\traceset \cap \nsupport{\tree}{\iparsp(n)})
    = \tdist(\traceset \cap \support{\tree})
    = \tdist(\traceset)$.
    Since both $\tdist$ and $\sdist$ are
    {probability} distributions, we also have
    $\tdist(\traceset) = 1 - \tdist(\traces \setminus \traceset)
    \leq 1 - \marg_*\sdist(\traces \setminus \traceset)
    = 1 - (1- \marg_*\sdist(\traceset))
    = \marg_*\sdist(\traceset)$.
    Hence
    $\marg_*\sdist = \tdist$ on $\traces$.

    \item
    Similarly, let $\traceset \in \algebra{\traces}$
    where $\traceset \subseteq \nsupport{\tree}{\iparsp(n)} $
    and $\delta > 0$.
    Then by \cref{prop: properties of transk}.iii,
    {for sufficiently large $m$}, we must have
    $\smeasure(\bigcup_{k=m+1}^\infty \validstates_k)
    = \smeasure(\validstates \setminus \validstates_{\leq m})
    < \delta$.
    Hence,
    \begin{align*}
      & \marg_*\smeasure(\traceset) \\
      & = \smeasure\left(\bigcup_{k=1}^m
            \inv{\marg}(\traceset)\cap \validstates_k\right)
        + \smeasure\left(\bigcup_{k=m+1}^\infty
            \inv{\marg}(\traceset)\cap \validstates_k\right) \\
      & < \sum_{k=1}^m \measure{\nstates{k}} (\inv{\marg}(\traceset)\cap \validstates_k) + \delta \\
      & = \sum_{k=1}^m \measure{\nstates{m}} (\set{(\enta,\auxa)\in \nstates{m} \mid \proj{m}{k}(\enta, \auxa) \in \inv{\marg}(\traceset)\cap \validstates_k}) + \delta \\
      & = \measure{\nstates{m}} (\bigcup_{k=1}^m \set{(\enta,\auxa)\in \nstates{m} \mid \proj{m}{k}(\enta, \auxa) \in \inv{\marg}(\traceset)\cap \validstates_k}) + \delta \\
      & \leq \measure{\nstates{m}}(\set{(\zip(\traceb, \traceb')\concat \entb, \auxa) \mid \traceb \in \traceset, \traceb'\in \traces, \entb \in \entsp^{\iparsp(m)-\iparsp(n)}, \auxa \in \nauxsp{m}}) + \delta \\
      & = \tmeasure(\traceset) + \delta.
    \end{align*}
    Then the proof proceeds as in (1). Note that since $\tree$ almost surely terminating (\cref{hass: almost surely terminating tree}),
    $\marg_*\measure{\states}(\support{\tree})
    = \measure{\traces}(\support{\tree}) = 1$
  \end{enumerate}
\end{proof}

\subsubsection{Correctness of NP-iMCMC}

The correctness of the NP-iMCMC sampler (\cref{fig:np-imcmc algo})
can be deduce from \cref{lemma: marginalised distribution is the target distribution}
and the fact that Hyrbid NP-iMCMC is a generalisation of NP-iMCMC

\NPiMCMCinvariant*

\clearpage

\section{Transforming Nonparametric Involutive MCMC}
\label{app: variants}
% !TEX root = ./../icml2022.tex

In this section,
we discuss how the techniques
discussed in \cite{DBLP:conf/icml/NeklyudovWEV20}
can be applied to the Hybrid NP-iMCMC
sampler presented in \cref{app: hybrid np-imcmc}.
Hence
instances of the Hybrid NP-iMCMC sampler,
such as NP-MH and NP-HMC,
can be extended
using these techniques
to become more flexible and efficient.

We assume the input
target density function $\tree:\traces\to\pReal$
is tree representable,
integrable (\cref{hass: integrable tree}) and
almost surely terminating (\cref{hass: almost surely terminating tree}).

\subsection{State-dependent Hybrid NP-iMCMC Mixture}
\label{sec: state-dependent mixture of np-imcmc}

Say we want to use multiple Hybrid NP-iMCMC samplers
to simulate the posterior given by
the target density function $\tree$.
The following technique allows us to `mix'
Hybrid NP-iMCMC samplers
in such a way that the resulting sampler still
preserves the posterior.

Given a collection of Hybrid NP-iMCMC samplers,
indexed by $\mixa \in \mixsp$,
for some $\alpha \in \Nat$,
each with
auxiliary kernels
$\set{\nkernel{n}_\mixa:\nparsp{n} \kernelto \nauxsp{n}}_{n\in\Nat}$ and
involutions
$\set{\ninvo{n}_\mixa:\nparsp{n} \times\nauxsp{n}
\to \nparsp{n} \times\nauxsp{n} }_{n\in\Nat}$
satisfying the projection commutation property
(\cref{hass: partial block diagonal inv}),
the \defn{State-dependent Hybrid NP-iMCMC Mixture} sampler
determines which Hybrid NP-iMCMC sampler to use by
drawing an indicator $\mixa \in \mixsp$
from a probability measure
$\mixkernel(\enta_0, \placeholder)$
where $\mixkernel:\bigcup_{n\in\Nat}\nparsp{n} \kernelto \entsp^m$
is a probability kernel and
$\enta_0$ is the entropy vector constructed from
the current sample $\traceb_0$ at the initialisation step
(\cref{hnp-imcmc step: initialisation} of Hybrid NP-iMCMC).
Then, using the $\mixa$-indexed Hybrid NP-iMCMC sampler,
a proposal $\traceb$ is generated and
accepted with a modified probability that
includes the probability of picking $\mixa$,
namely
\begin{align*}
  \min\bigg\{1; \;
  & \frac
    {
      \tree{(\traceb)}\cdot
      \nkernelpdf{k}_\mixa(\proj{n}{k}(\enta, \auxa)) \cdot
      \nparpdf{n}(\enta)\cdot\nauxpdf{n}(\auxa)}
    {
      \tree{(\traceb_0)}\cdot
      \nkernelpdf{k_0}_\mixa(\proj{n}{k_0}(\enta_0, \auxa_0)) \cdot
      \nparpdf{n}(\enta_0)\cdot\nauxpdf{n}(\auxa_0)}
    \cdot
  %   \\
  % & \qquad\qquad\qquad\qquad\qquad\qquad\qquad
  \frac
    {\mixpdf(\range{\enta_0}{1}{k_0}, \mixa)}
    {\mixpdf(\range{\enta}{1}{k}, \mixa)}
  \cdot
  \abs{\det(
      \grad{\ninvo{n}_\mixa(
        \enta_0, \auxa_0)})}
  \bigg\}
\end{align*}
where
$(\enta_0, \auxa_0)$ is the (possibly extended) initial state,
$(\enta, \auxa)$ is the new state,
$n = \dim{(\enta_0)} = \dim{(\auxa_0)}$,
$k_0$ is the dimension of ${\traceb_0}$
(i.e.~$\len{\traceb_0} = \iparsp{(k_0)}$) and
$k$ is the dimension of $\traceb$
(i.e.~$\len{\traceb} = \iparsp{(k)}$).

\paragraph{Pseudocode}

This sampler can be implemented in SPCF as
the \codein{MixtureNPiMCMC} function in \cref{code: mixture np-imcmc}.
(Terms specific to this technique
are highlighted.)
We assume the following SPCF terms exists:
\codein{mixkernel} of type \codein{List(X) -> (R*B)^l}
implements the mixture kernel $\mixkernel:\bigcup_{n\in\Nat}\nparsp{n} \kernelto \mixsp$;
\codein{pdfmixkernel} of type
\codein{List(X)*(R*B)^l -> R}
implements
the probability density function
$\mixpdf:\bigcup_{n\in\Nat}\nparsp{n}\times\mixsp\to\Real$
; and
for each $\mixa \in \mixsp$ and $n\in\Nat$,
\codein{auxkernel[n][m]}
implements the auxiliary kernel $\nkernel{n}_\mixa$;
\codein{pdfauxkernel[n][m]} and
implements the pdf $\nkernelpdf{n}_\mixa$;
\codein{involution[n][m]}
implements the involution $\ninvo{n}_\mixa$; and
\codein{absdetjacinv[n][m]} implements
the absolute value of the Jacobian determinant of
$\ninvo{n}_\mixa$.

\begin{figure}
\begin{code}[label={code: mixture np-imcmc}, caption={Pseudocode of the State-dependent Hybrid NP-iMCMC Mixture algorithm}]
def MixtureNPiMCMC(t0):
  k0 = dim(t0)                                  # initialisation step
  x0 = [(e, coin) if Type(e) in R else (normal, e) for e in t0]
  (@@)`m = mixkernel(x0)`                             # mixture step
  v0 = auxkernel[k0][m](x0)                     # stochastic step
  (x,v) = involution[k0][m](x0,v0)              # deterministic step
  n = k0                                        # extend step
  while not intersect(instance(x),support(w)):
    x0 = x0 + [(normal, coin)]*(indexX(n+1)-indexX(n))
    v0 = v0 + [(normal, coin)]*(indexY(n+1)-indexY(n))
    n = n + 1
    (x,v) = involution[n][m](x0,v0)
  t = intersect(instance(x),support(w))[0]      # accept/reject step
  k = dim(t)
  return t if uniform < min{1, w(t)/w(t0) *
                               pdfauxkernel[k][m](proj((x,v),k))/
                                 pdfauxkernel[k0][m](proj((x0,v0),k0)) *
                               pdfpar[n](x)/pdfpar[n](x0) *
                               pdfaux[n](v)/pdfaux[n](v0) *
                               (@@)`pdfmixkernel(proj(x,k),m)/`
                                 (@@)`pdfmixkernel(proj(x0,k0),m)` *
                               absdetjacinv[n][m](x0,v0)}
           else t0
\end{code}
% \end{figure}
% \begin{figure}
\begin{code}[label={code: mixture np-imcmc (correctness)},
  caption={Pseudocode for the correctness of the State-dependent Hybrid NP-iMCMC Mixture algorithm}]
def mixauxkernel[n](x0)
  m = mixkernel(x0)
  v0 = auxkernel[n][m](x0)
  return m + v0

def mixinvolution[n](x0,mixv0)
  m = mixv0[:l]
  v0 = mixv0[l:]
  (x,v) = involution[n][m](x0,v0)
  return (x,m + v)

def mixindexX(n): return indexX(n)
def mixindexY(n): return l + indexY(n)
def mixproj((x,v),k): return (x[:mixindexX(k)],v[:mixindexY(k)])
\end{code}
\end{figure}

\paragraph{Correctness}

Similar to the correctness arguments in
\cite{DBLP:conf/icml/NeklyudovWEV20},
we show that
the State-dependent Hybrid NP-iMCMC Mixture sampler
is correct
by formulating \codein{MixtureNPiMCMC}
as an instance of \codein{NPiMCMC}
(\cref{code: np-imcmc}).
This means specifying
\codein{auxkernel[n]}
and \codein{involution[n]}
in \codein{NPiMCMC}
and arguing that the resulting
\codein{NPiMCMC} function is equivalent to \codein{MixtureNPiMCMC}.

The SPCF terms \codein{mixauxkernel[n]}
and \codein{mixinvolution[n]}
given in \cref{code: mixture np-imcmc (correctness)}
should suffice.
The auxiliary space is expanded to embed
the indicator \codein{m} in such a way that
the auxiliary variable \codein{mixv}
is in the space $\mixsp \times \nauxsp{n}$
where its first $\ell$-th components \codein{mixv[:l]} gives \codein{m}
and the rest \codein{mixv[l:]} gives \codein{v}.
Since the auxiliary space is expanded
to include the indicator,
the maps
\codein{mixindexX} and \codein{mixindexY}
and the projection \codein{mixproj}
are modified accordingly.

To see how the \codein{NPiMCMC} function with
\codein{auxkernel[n]} replaced by \codein{mixauxkernel[n]} and
\codein{involution[n]} replaced by \codein{mixinvolution[n]}
is equivalent to \codein{MixtureNPiMCMC},
we onyl need to consider the probability density of
\codein{mixauxkernel[k]} at
\codein{mixproj((x,mixv),k)}.

\begin{code}[numbers=none]
    pdfmixauxkernel[k](x[:mixindexX(k)], mixv[:mixindexY(k)])
  = pdfmixauxkernel[k](x[:indexX(k)], mixv[:l+indexY(k)])
  = pdfmixkernel(x[:indexX(k)], mixv[:l]) * pdfauxkernel[k][mixv[:l]](x[:indexX(k)], mixv[l:l+indexY(k)])
  = pdfmixkernel(x[:indexX(k)], m) * pdfauxkernel[k][m](x[:indexX(k)], v[:indexY(k)])
  = pdfmixkernel(proj(x,k),m) * pdfauxkernel[k][m](proj((x,v),k))
\end{code}
where \codein{m = mixv[:l]} and
\codein{v = mixv[l:]}.
This shows that
the acceptance probability in \codein{NPiMCMC}
is identical to that in \codein{MixtureNPiMCMC}
and hence the two algorithms are equivalent.

\subsection{Direction Hybrid NP-iMCMC Algorithm}
\label{sec: auxiliary direction np-imcmc}

Sometimes it is difficult to specify involutions
that explores the model fully.
The following technique tells us that
bijections are good enough.

Given endofunctions
$\nbij{n}$ on $\nparsp{n} \times\nauxsp{n}$
that are differentiable almost everywhere and bijective
for each $n\in\Nat$
such that
the sets $\set{\nbij{n}}_n$ and
$\set{\inv{\nbij{n}}}_n$
satisfy the projection commutative property
(\cref{hass: partial block diagonal inv}),
the \defn{Direction Hybrid NP-iMCMC} algorithm
randomly use either $\nbij{n}$ or
$\inv{\nbij{n}}$ to move around the state space
and proposes a new sample.

\paragraph{Pseudocode}

This sampler can be expressed in SPCF as the
\codein{DirectionNPiMCMC} function in \cref{code: direction np-imcmc}.
(Terms specific to this technique are highlighted.)
We assume for each $n\in\Nat$ and $\dira \in \Bool$,
there is a SPCF term
\codein{bijection[n][d]} where
\codein{bijection[n][True]} implements
the bijection $\nbij{n}$ and
\codein{bijection[n][False]} implements
the inverse $\inv{\nbij{n}}$ and
SPCF term
\codein{absdetjacbij[n][d]} that implements
the absolute value of the Jacobian determinant of
$\nbij{n}$ if \codein{d = True} and
the inverse $\inv{\nbij{n}}$ otherwise.

\begin{figure}
\begin{code}[label={code: direction np-imcmc}, caption={Pseudocode of the Direction Hybrid NP-iMCMC algorithm}]
def DirectionNPiMCMC(t0):
  (@@)`d0 = coin`                                     # direction step
  k0 = dim(t0)                                  # initialisation step
  x0 = [(e, coin) if Type(e) in R else (normal, e) for e in t0]
  v0 = auxkernel[k0](x0)                        # stochastic step
  (x,v) = `bijection[k0][d0]`(x0,v0)              # deterministic step
  n = k0                                        # extend step
  while not intersect(instance(x),support(w)):
    x0 = x0 + [(normal, coin)]*(indexX(n+1)-indexX(n))
    v0 = v0 + [(normal, coin)]*(indexY(n+1)-indexY(n))
    n = n + 1
    (x,v) = `bijection[n][d0]`(x0,v0)
  t = intersect(instance(x),support(w))[0]      # accept/reject step
  k = dim(t)
  return t if uniform < min{1, w(t)/w(t0) *
                               pdfauxkernel[k](proj((x,v),k))/
                                 pdfauxkernel[k0](proj((x0,v0),k0)) *
                               pdfpar[n](x)/pdfpar[n](x0) *
                               pdfaux[n](v)/pdfaux[n](v0) *
                               (@@)`absdetjacbij[n][d0]`(x0,v0)}
           else t0
\end{code}
% \end{figure}
% \begin{figure}
\begin{code}[label={code: direction np-imcmc (correctness)},
  caption={Pseudocode for \texttt{dirauxkernel} and \texttt{dirinvolution}}]
def dirauxkernel[n](x0)
  d0 = coin
  v0 = auxkernel[n](x0)
  return [(normal, d0)] + v0

def dirinvolution[n](x0,dirv0)
  d0 = dirv0[0][1]
  v0 = dirv0[1:]
  (x,v) = bijection[n][d0](x0,v0)
  d = not d0
  return (x, [(dirv0[0][0],d)] + v)

def dirindexX(n): return indexX(n)
def dirindexY(n): return 1+indexY(n)
def dirproj((x,v),k): return (x[:dirindexX(k)], v[:dirindexY(k)])
\end{code}
\end{figure}

\paragraph{Correctness}

We show that
\codein{DirectionNPiMCMC}
can be formulated as an instance of
\codein{NPiMCMC} (\cref{code: np-imcmc})
with a specification of
\codein{auxkernel[n]} and \codein{involution[n]}.

The SPCF terms
\codein{dirauxkernel[n]} and
\codein{dirinvolution[n]}
in \cref{code: direction np-imcmc (correctness)}
would work.
The auxiliary space is expanded to include
the direction variable \codein{d0}
so that
the auxiliary variable \codein{dirv0}
is in the space $\entsp \times \nauxsp{n}$
where
the Boolean-component
\codein{dirv0[0][1]}
of its first coordinate gives \codein{d0}
and
the second to last coordinates \codein{dirv0[1:]} gives \codein{v0}.
(Note the value of \codein{dirv0[0][0]} is redundant
and is only used to make \codein{dirv0[0]} an entropy.)
Since the auxiliary space is expanded,
the maps
\codein{dirindexX} and \codein{dirindexY}
and the projection
\codein{dirproj}
are modified accordingly.

To see how the \codein{NPiMCMC} function with
\codein{auxkernel[n]} replaced by \codein{dirauxkernel[n]} and
\codein{involution[n]} replaced by \codein{dirinvolution[n]}
is equivalent to \codein{DirectionNPiMCMC},
we first consider
the density of
\codein{dirauxkernel[k0]}
at \codein{dirproj((x0,dirv0),k0)}.

\begin{code}[numbers=none]
    pdfdirauxkernel[k0](x0[:dirindexX(k0)], dirv0[:dirindexY(k0)])
  = pdfdirauxkernel[k0](x0[:indexX(k0)], dirv0[:1+indexY(k0)])
  = pdfcoin(dirv0[0][1]) * pdfnormal(dirv0[0][0]) * pdfauxkernel[k0](x0[:indexX(k0)], dirv0[1:1+indexY(k0)])
  = 0.5 * pdfnormal(dirv0[0][0]) * pdfauxkernel[k0](proj((x0,v0),k0))
\end{code}
where
\codein{v0 = dirv0[1:]}.
A similar argument can be made for
\codein{pdfdirauxkernel[k](dirproj((x,dirv),k))},
which makes the acceptance probability in \codein{NPiMCMC}
identical to that in \codein{DirectionNPiMCMC}.
Moreover,
writing \codein{d0} for \codein{dirv0[0][1]},
the absolute value of the Jacobian determinant
of \codein{dirinvolution[n]} at
\codein{(x0,dirv0)} is
\codein{absdetjacbij[n][d0](x0,v0)}.
Most importantly,
\codein{dirinvolution[n]} is now
involutive.
Hence, \codein{NPiMCMC} is the same as
\codein{DirectionNPiMCMC}.

\subsection{Persistent Hybrid NP-iMCMC Algorithm}
\label{sec: persistent np-imcmc}

It is known that irreversible transition kernels
(those that does not satisfy detailed balance)
have better mixing times,
i.e.~converges more quickly to the target distribution,
compared to reversible ones.
The following technique gives us a method to
transform Hybrid NP-iMCMC algorithms to
irreversible ones that still preserves the target distribution.
The key is to compose
the Hybrid NP-iMCMC sampler with a transition kernel
so that
the resulting algorithm does not satisfy detailed balance.
% but still preserves the invariant distribution.

The \defn{Persistent Hybrid NP-iMCMC} algorithm
is a MCMC algorithm similar to
the Direction Hybrid NP-iMCMC sampler
in which the direction variable is used
to determine auxiliary kernels
($\set{\nkernel{n}_1:\nparsp{n}\kernelto\nauxsp{n}}_n $
or
$\set{\nkernel{n}_2:\nparsp{n}\kernelto\nauxsp{n}}_n $) and
bijections
($\set{\nbij{n}:\nparsp{n}\times\nauxsp{n}\to\nparsp{n}\times\nauxsp{n}}_n $
or $\set{\inv{\nbij{n}}:\nparsp{n}\times\nauxsp{n}\to\nparsp{n}\times\nauxsp{n}}_n $)
being used.
The difference is that
Persistent Hybrid NP-iMCMC keeps track of the direction
(instead of sampling a fresh one in each iteration)
and flips it strategically to
make the resulting algorithm irreversible.

\paragraph{Pseudocode}

This sampler can be expressed in SPCF as
\codein{PersistentNPiMCMC}
in \cref{code: persistent np-imcmc}.
(Terms specific to this technique are highlighted.)
In addition to the SPCF terms in \codein{DirectionNPiMCMC},
we assume there is a SPCF term
\codein{auxkernel[n][d]} such that
\codein{auxkernel[n][True]} implements
the auxiliary kernel
$\nkernel{n}_1: \nparsp{n} \kernelto \nauxsp{n}$ and
\codein{pdfauxkernel[n][True]}
its pdf $\nkernelpdf{n}_1$
and similarly for
\codein{auxkernel[n][False]}
and
\codein{pdfauxkernel[n][False]}.
Note that
\codein{PersistentNPiMCMC}
updates samples
on the space $\nparsp{n} \times \Bool$,
which can easily be marginalised to $\nparsp{n}$
by taking the first $\iparsp(n)$ components.

% Importantly the
% accepted proposal \codein{x} is returned
% with an unchanged direction variable
% \codein{d0}
% whereas the repeated current sample \codein{x0}
% is returned with a negated direction \codein{-d0}.
% Since both the auxiliary kernel and bijection
% depends on \codein{d0},
% persisting \codein{d0} is meant to construct
% a irreversible transition kernel.

\begin{figure}
\begin{code}[
  label={code: persistent np-imcmc},
  caption={Pseudocode of the Persistent Hybrid NP-iMCMC algorithm}]
def PersistentNPiMCMC(t0,`d0`):
  k0 = dim(t0)                                  # initialisation step
  x0 = [(e, coin) if Type(e) in R else (normal, e) for e in t0]
  v0 = `auxkernel[k0][d0]`(x0)                     # stochastic step
  (x,v) = `bijection[k0][d0]`(x0,v0)              # deterministic step
  n = k0                                        # extend step
  while not intersect(instance(x),support(w)):
    x0 = x0 + [(normal, coin)]*(indexX(n+1)-indexX(n))
    v0 = v0 + [(normal, coin)]*(indexY(n+1)-indexY(n))
    n = n + 1
    (x,v) = `bijection[n][d0]`(x0,v0)
  (@@)`d = not d0`
  t = intersect(instance(x),support(w))[0]      # accept/reject step
  k = dim(t)
  return (t, `not d`) if uniform < min{1, w(t)/w(t0) *
                               pdfauxkernel[k][d](proj((x,v),k))/
                                 pdfauxkernel[k0][d0](proj((x0,v0),k0)) *
                               pdfpar[n](x)/pdfpar[n](x0) *
                               pdfaux[n](v)/pdfaux[n](v0) *
                               absdetjacbij[n][d0](x0,v0)}
           else (t0, `d`)
\end{code}
% \end{figure}
% \begin{figure}
\begin{code}[label={code: persistent np-imcmc (correctness)},
  caption={Pseudocode for \texttt{perauxkernel} and \texttt{perinvolution}}]
def perauxkernel[n](perx0)
  d0 = perx0[0][1]
  x0 = perx0[1:]
  v0 = auxkernel[n][d0](x0)
  return v0

def perinvolution[n](perx0,v0)
  d0 = perx0[0][1]
  x0 = perx0[1:]
  (x,v) = bijection[n][d0](x0,v0)
  d = not d0
  return ([(perx0[0][0],d)] + x, v)

def perindexX(n): return 1+indexX(n)
def perindexY(n): return indexY(n)
def perproj((x,v),k): return (x[:perindexX(k)],v[:perindexY(k)])

def flipdir(perx0):
  d0 = perx0[0][1]
  perx0[0][1] = not d0
  return perx0
\end{code}
\end{figure}

\paragraph{Correctness}

We show that
\codein{PersistentNPiMCMC}
can be formulated as a composition
of two instances of
\codein{NPiMCMC} (\cref{code: np-imcmc}).

Consider the \codein{NPiMCMC} with
auxiliary kernel
\codein{perauxkernel[n]} and
involution
\codein{perinvolution[n]}
in \cref{code: persistent np-imcmc (correctness)}.
In this case,
the parameter space is expanded to include
the direction variable
so that a parameter variable \codein{perx}
is on the space $\entsp \times \nparsp{n}$
where \codein{perx[0][1]} gives \codein{d}
and \codein{perx[1:]} gives \codein{x}.
Since the parameter space is expanded,
the maps
\codein{perindexX} and \codein{perindexY}
and projection \codein{perproj}
are modified accordingly.

Again, we first consider
the density of
\codein{perauxkernel[k0]}
at \codein{perproj((perx0,v0),k0)}.

\begin{code}[numbers=none]
    pdfperauxkernel[k0](perx0[:perindexX(k0)], v0[:perindexY(k0)])
  = pdfauxkernel[k0][perx[0][1]](perx0[1:1+indexX(k0)], v0[:indexY(k0)])
  = pdfauxkernel[k0][d0](x0[:indexX(k0)], v0[:indexY(k0)])
  = pdfauxkernel[k0][d0](proj((x0,v0),k0))
\end{code}
where
\codein{d0 = perx0[0][1]} and
\codein{x0 = perx0[1:]}.
A similar argument can be made for
\codein{pdfperauxkernel[k](perproj((perx,v),k))}.
Moreover, the absolute value of the Jacobian determinant
of \codein{perinvolution[n]} at
\codein{(perx0,v0)}
is
\codein{absdetjacbij[n][d0](x0,v0)}.
Hence,
the acceptance probability in \codein{NPiMCMC}
is identical to that in \codein{PersistentNPiMCMC}.

The \codein{NPiMCMC} function with
\codein{auxkernel[n]} replaced by \codein{perauxkernel[n]} and
\codein{involution[n]} replaced by \codein{perinvolution[n]}
is \emph{almost} equivalent to
\codein{PersistentNPiMCMC},
except
\codein{NPiMCMC}
induces a transition kernel on $\entsp \times \nparsp{n}$
whereas \codein{PersistentNPiMCMC}
induces a transition kernel on $\Bool \times \nparsp{n}$; and
when the proposal \codein{t} is accepted,
\codein{NPiMCMC}
returns \codein{d} whereas
\codein{PersistentNPiMCMC} returns
\codein{not d}.

These differences can be reconciled by
composing \codein{NPiMCMC}
with \codein{flipdir},
which is an instance of \codein{NPiMCMC}
which skips the stochastic step
and has an involution
that flips the direction variable stored
in \codein{perx0[0][1]}.
The composition generates a Markov chain
on $\entsp \times \nparsp{n}$ and
marginalising it to a Markov chain on
$\Bool \times \nparsp{n}$
gives us the same result as
\codein{PersistentNPiMCMC}.

\clearpage

\section{Multiple Step Nonparametric Involutive MCMC}
\label{app: multiple step np-imcmc}
% !TEX root = ./../icml2022.tex

In this section, we study
the \defn{Multiple Step NP-iMCMC} sampler,
a generalisation of the Hybrid NP-iMCMC sampler (\cref{sec: np-imcmc algorithm})
(and also of NP-iMCMC (\cref{fig:np-imcmc algo})),
where the involution is applied
multiple times to generate a proposed state.

\subsection{Motivation}

\cref{hnp-imcmc step: extend} in the
Hyrbid NP-iMCMC sampler may seem inefficient. While it terminates almost surely (thanks to \cref{hass: almost surely terminating tree}),
This is because
whenever the dimension of the state is changed,
the algorithm has to ``re-run'' the involution again
(\cref{hnp-imcmc step: extend}.ii).
This means
the expected number of iterations may be infinite.

To remedy this problem, we introduce two
new concepts:
\begin{itemize}
  \item The \emph{slice function} which
  might make ``re-runs''
  (\cref{hnp-imcmc step: extend}.ii) quicker.

  \item The \emph{Multiple Step NP-iMCMC}
  sampler, a generaliation
  of Hyrbid NP-iMCMC,
  which uses a list of bijections
  to move around the state space.
\end{itemize}

\subsection{Slice function}
\label{sec: slice function}
% !TEX root = ./../icml2022.tex

For each dimension $n\in\Nat$,
we call the measurable function
$\slicefn{n}: \nstates{n}\to
  \entsp^{\iparsp(n)-\iparsp(n-1)} \times
  \entsp^{\iauxsp(n)-\iauxsp(n-1)}$
a \defn{slice} of
the endofunction $\ninvo{n}$ on $\nstates{n}$ if
it captures the movement of the $n$-th
dimensional states with
an instance of dimension \emph{lower than} $n$.
Formally, this means
\[
  \slicefn{n}(\enta, \auxa)
  = (\drop{n-1} \circ \ninvo{n})(\enta, \auxa)
  \quad\text{if }
  \traceb\in\instances{\enta}\cap\support{\tree}
  \text{ and }
  \len{\traceb} < \iparsp(n).
\]
Note we can always define
a slice of $\ninvo{n}$
by setting
$\slicefn{n} := \drop{n-1} \circ \ninvo{n}$.

With the slice function $\slicefn{n}$
defined for each involution $\ninvo{n}$,
\cref{hnp-imcmc step: extend}.ii in the
Hyrbid NP-iMCMC algorithm
(\cref{sec: np-imcmc algorithm}):
\begin{enumerate}[i.]
  \item[(Step 4.ii)]
  Move around the $n+1$-dimensional state
  space $\nparsp{n+1}\times\nauxsp{n+1}$ and
  compute the new state
  by applying the involution $\ninvo{n+1}$
  to the initial state
  $(\enta_0 \concat \entb_0, \auxa_0 \concat \auxb_0)$;
\end{enumerate}
can be replaced by the following Step 4.ii':
\begin{itemize}
  \item[(Step 4.ii')]
  Replace and extend the
  $n$-dimensional new state
  from $(\enta, \auxa)$
  to
  a state
  $(\enta \concat \entb, \auxa \concat \auxb)$
  of dimension $n+1$
  where $(\entb,\auxb)$
  is the result of
  $\slicefn{n+1}(\enta_0 \concat \entb_0,
  \auxa_0 \concat \auxb_0) $.
\end{itemize}

By \cref{hass: partial block diagonal inv},
the first $n$ components of the new
$n+1$-dimensional state
$\ninvo{n+1}(\enta_0 \concat \entb_0,
  \auxa_0 \concat \auxb_0)$
is
\[
  \take{n}(\ninvo{n+1}(\enta_0 \concat \entb_0,
  \auxa_0 \concat \auxb_0))
  =
  \ninvo{n}(\take{n}(\enta_0 \concat \entb_0,
    \auxa_0 \concat \auxb_0))
  =
  \ninvo{n}(\enta_0, \auxa_0)
  = (\enta, \auxa)
\]
and by the definition of slice
the $(n+1)$-th component of the new state is
\[
  \drop{n}(\ninvo{n+1}(\enta_0 \concat \entb_0,
  \auxa_0 \concat \auxb_0))
  =
  \slicefn{n+1}(\enta_0 \concat \entb_0,
  \auxa_0 \concat \auxb_0).
\]
Hence the new state computed by
Step 4.ii and Step 4.ii' are the same.

The slice function $\slicefn{n}$ is useful when
the involution is computationally expensive
but has a light slice function.
After Step 4.ii is replaced by Step 4.ii',
the Hyrbid NP-iMCMC sampler
need only to run the involution once
(\cref{hnp-imcmc step: deterministic})
and any subsequent ``re-runs'' (\cref{hnp-imcmc step: extend})
can be performed by the slice function.

If the slice function $\slicefn{n}$ is
implemented as \codein{slice[n]} in SPCF,
Line 11 in \codein{NPiMCMC}
can be changed from
\codein{(x,v) = involution[n](x0,v0)}
to
\begin{code}[numbers=none]
(x',v') = slice[n](x0,v0); (x,v) = (x + x', v + v')
\end{code}

\subsubsection{Example (HMC)}
\label{sec: hmc slice}

Momentum update is the most computationally
heavy component in the HMC sampler.
Hence it would be useful if it has a
lightweight slice function.

In the setting of Hyrbid NP-iMCMC,
we assume the trace space
$\traces$ is a list measurable space of
the Real measurable space $\Real$.
Then, the $n$-dimensional momentum update
$\momstep_{k}$ is an endofunction on
$\Real^n \times \Real^n$ defined as
\[
  \momstep_{k}(\posa,\moma) := (\posa,\moma-k\grad U (\posa))
\]
where
$U (\posa) := -\log\
\max\{\tree(\traceb) \mid \traceb \in \instances{\posa} \}.$
is the $n$-dimensional potential energy.

Given a $n$-dimensional state
$(\posa,\moma)$ where
$\traceb \in \instances{\posa}\cap \support{\tree}$
has dimension lower than $n$,
the gradient of the potential energy $U$ at $\posa$
w.r.t.~the $n$-th coordinate is zero.
Hence,
\[
  (\drop{n-1}\circ\momstep_{k})(\posa,\moma)
  = \drop{n-1}(\posa,\moma-k\grad U (\posa))
  = (\seqindex{\posa}{n},\seqindex{\moma}{n}),
\]
and the slice of the momentum update
$\momstep_{k}$
is simply the projection
$\drop{n-1}(\posa,\moma) :=
(\seqindex{\posa}{n},\seqindex{\moma}{n}).$

However, not \emph{all} $2L$
momentum updates in the re-runs of
the leapfrog function $\leapfrog$ can be replaced
by its slice $\drop{n-1}$.
This is because
when the dimension increments to say $n+1$,
only the extended initial state
$(\enta_0\concat\entb_0, \auxa_0\concat\auxb_0)$
has the property that it has an instance
with dimension lower than $n+1$
and not the intermediate states.

\subsection{Multiple Step NP-iMCMC}
\label{sec: multiple step npimcmc}
% !TEX root = ./../icml2022.tex

Say the involution of a Hybrid NP-iMCMC sampler
is comprised of a list of
bijective endofunctions on $\nstates{n}$,
namely $\ninvo{n} :=
\nelem{n}{L} \circ
\dots \circ \nelem{n}{2} \circ \nelem{n}{1}$.
To compute the new state,
we can either
\begin{itemize}
  \item
  apply the involution $\ninvo{n}$
  to the initial state
  $(\enta_0,\auxa_0)$
  in one go and
  check whether the result $(\enta, \auxa)$
  has an instance in the
  support of $\tree$, or

  \item
  for each $\ell =1,\dots, L$,
  apply the endofunction $\nelem{n}{\ell}$
  to
  $(\enta_{\ell-1},\auxa_{\ell-1})$
  and (immediately)
  check whether the intermediate state
  $(\enta_\ell,\auxa_\ell)$ has
  an instance in the support of $\tree$.
\end{itemize}

The Hybrid NP-iMCMC sampler presented in \cref{sec: np-imcmc algorithm}
takes the first option as
it is conceptually simpler.
However,
the second option is just as valid
and more importantly
give us the requirements needed to
replace each endofunction
by its slice in any subsequent ``re-runs''.

\subsubsection{The Multiple Step NP-iMCMC Algorithm}

Assume the target density $\tree$
satisfies \cref{vass: integrable tree,vass: almost surely terminating tree}; and
for each $n\in\Nat$,
there is a probability kernel $\nkernel{n}$ and
a list of $L$ bijective endofunctions
$\set{\nelem{n}{\ell}:\nstates{n}\to\nstates{n}
\mid \ell=1,\dots,L}_n$
such that for each $\ell$,
$\set{\nelem{n}{\ell}}_n$ satisfies the
\invoass{}
(\cref{vass: partial block diagonal inv})
and for each $n\in\Nat$,
their composition
$\nelem{n}{L} \circ \dots \circ \nelem{n}{1}$
is involutive.

Let $\slicefn{n}_{\ell}$ be a slice of
the endofunction $\nelem{n}{\ell}$.
Given a SPCF program $\terma$ with
weight function $\tree$ on the trace space,
the \defn{Multiple Step NP-iMCMC} sampler
generates a Markov chain as follows.
Given a current sample $\traceb_0$ of dimension $k_0$,

\begin{enumerate}[1.]
  \item (Initialisation Step)
  Form a ${k_0}$-dimensional parameter variable
  $\enta_0 \in \nparsp{k_0}$ by
  pairing each value
  $\seqindex{\traceb_0}{i}$ in
  $\traceb_0$ with
  a randomly drawn value $\trace$
  of the other type to make
  a pair
  $(\seqindex{\traceb_0}{i}, \trace)$ or $(\trace, \seqindex{\traceb_0}{i})$
  in the entropy space $\entsp$.

  \item (Stochastic Step)
  Introduce randomness to the sampler by
  drawing a ${k_0}$-dimensional value
  $\auxa_0 \in \nauxsp{k_0}$
  from the probability measure
  $\nkernel{{k_0}}(\enta_0, \placeholder)$.

  \item (Multiple Step)
  Initialise $\ell = 1$.
  If $\ell = L$,
  proceed to Step 4
  with $\traceb$ as the proposed sample;
  otherwise

  \begin{enumerate}[3.1.]
    \item (Deterministic Step)
    Compute the $\ell$-th state
    $(\enta_{\ell}, \auxa_{\ell})$ by applying
    the endofunction $\nelem{n}{\ell}$ to
    $(\enta_{\ell-1}, \auxa_{\ell-1})$ where
    $n = \dim{(\enta_{\ell-1})}$.

    \item (Extend Step)
    Test whether any instance $\traceb$ of
    $\enta_{\ell}$ is
    in the support of $\tree$.
    If so,
    go to Step 3 with an incremented $\ell$;
    otherwise
    (none of the instances of $\enta_{\ell}$ is
    in the support of $\tree$),
    \begin{enumerate}[3.2.i]
      \item
      Extend and replace
      the $n$-dimensional initial state
      from $(\enta_0,\auxa_0)$
      to a state $(\enta_0 \concat \entb_0, \auxa_0 \concat \auxb_0)$
      of dimension $n+1$
      where $\entb_0$ and $\auxb_0$ are values
      drawn randomly
      from $\measure{\entsp^{\iparsp{(n+1)}-\iparsp{(n)}}}$ and
      $\measure{\entsp^{\iauxsp{(n+1)}-\iauxsp{(n)}}}$
      respectively.

      \item
      For each $i=1,\dots,\ell$,
      extend and replace
      the $n$-dimensional $i$-th
      intermediate state from
      $(\enta_i, \auxa_i)$
      to a state $(\enta_i \concat \entb_i,
      \auxa_i \concat \auxb_i)$
      of dimension $n+1$
      where $(\entb_i, \auxb_i)$
      is the result of
      $\slicefn{n+1}_{i}
      (\enta_{i-1}, \auxa_{i-1})$.

      \item
      Go to Step 3.2 with the
      extended $n+1$-dimensional states $(\enta_i, \auxa_i)$
      for $i=0,\dots,\ell$.
    \end{enumerate}
  \end{enumerate}

  \item (Accept/reject Step)
  Accept the proposed sample $\traceb$
  as the next sample with probability
  \begin{align*}
    \label{eq: ms np-imcmc acceptance ratio}
    \min\bigg\{1; \;
    & \frac
      {
        \tree{(\traceb)}\cdot
        \nkernelpdf{k}(\proj{n}{k}(\enta_L, \auxa_L)) \cdot
        \nparpdf{n}(\enta_L)\cdot\nauxpdf{n}(\auxa_L)}
      {
        \tree{(\traceb_0)}\cdot
        \nkernelpdf{k_0}(\proj{n}{k_0}(\enta_0, \auxa_0)) \cdot
        \nparpdf{n}(\enta_0)\cdot\nauxpdf{n}(\auxa_0)}
      \cdot
    %   \\
    % & \qquad\qquad\qquad\qquad\qquad\qquad\qquad\qquad\qquad
    \prod_{\ell=1}^{L}
    \abs{\det(
        \grad{\nelem{n}{\ell}(
          \enta_{\ell-1} , \auxa_{\ell-1})})}
    \bigg\}
  \end{align*}
  where $n = \dim{(\enta_0)} = \dim{(\auxa_0)}$,
  $k$ is the dimension of $\traceb$ and
  $k_0$ is the dimension of ${\traceb_0}$;
  otherwise reject the proposal and repeat $\traceb_0$.
\end{enumerate}

Unlike in Hybrid NP-iMCMC,
the Multiple Step NP-iMCMC sampler computes
the intermediate states
$\set{(\enta_{\ell}, \auxa_{\ell})}_{\ell=1,\dots,L}$
one-by-one, making sure in Step 3.2 that
each of these state $(\enta_{\ell}, \auxa_{\ell})$
has an instance in the support of $\tree$.
Hence when the dimension is incremented
from $n$ to $n+1$
we can use the slice functions
to extend intermediate states to states of
dimension $n+1$.

\begin{remark}
  The Multiple Step NP-iMCMC sampler can be
  seen as
  a generalisation of Hybrid NP-iMCMC
  (and hence a generalisation of NP-iMCMC (\cref{sec: Hybrid NP-iMCMC is a Generalisation of NP-iMCMC}))
  as we can recover
  Hybrid NP-iMCMC by setting $L$ to one and taking the
  involution $\ninvo{n}$ as the only endofunction in Multiple Step NP-iMCMC.
\end{remark}

\subsubsection{Pseudocode of Multiple Step NP-iMCMC Algorithm}

\cref{code: multistep np-imcmc} gives
a SPCF implementation of
Multiple Step NP-iMCMC as
the function \codein{MultistepNPiMCMC}
with
target density \codein{w};
auxiliary kernel \codein{auxkernel[n]}
and its density \codein{pdfauxkernel[n]} and
\codein{L} number of endofunctions
\codein{f[n][l]}
(\codein{l} ranges from \codein{1} to \codein{L})
for each dimension \codein{n} with slice
\codein{slice[n][l]} and the absolute value of
its Jacobian determinant
\codein{absdetjacf[n][l]};
parameter and auxiliary index maps
\codein{indexX} and \codein{indexY} and
projection \codein{proj}.

\begin{figure}[t]
\begin{code}[caption={Code for Multiple Step NP-iMCMC},label={code: multistep np-imcmc}]
def MultistepNPiMCMC(t0):
  k0 = dim(t0)                                          # initialisation step
  x0 = [(e, coin) if Type(e) in R else (normal, e) for e in t0]
  v0 = auxkernel[k0](x0)                                # stochastic step

  # start of multiple step
  n = k0
  (x[0],v[0]) = (x0,v0)
  for l in range(1,L+1):
    (@@)`(x[l],v[l]) = f[n][l](x[l-1],v[l-1])`                # deterministic step
    while not intersect(instance(x[l]),support(w)):     # extend step
      x[0] = x[0] + [(normal, coin)]*(indexX(n+1)-indexX(n))
      v[0] = v[0] + [(normal, coin)]*(indexY(n+1)-indexY(n))
      for i in range(1,l+1):
        (@@)`(y,u) = slice[n+1][i](x[i-1],v[i-1])`
        (@@)`(x[i],v[i]) = (x[i]+y, v[i]+u)`
      n = n + 1
  (x0,v0) = (x[0],v[0])
  (x,v) = (x[L],v[L])
  # end of multiple step

  t = intersect(instance(x),support(w))[0]              # accept/reject step
  k = dim(t)
  return t if uniform < min{1,w(t)/w(t0) *
                        pdfauxkernel[k](proj((x,v),k))/
                          pdfauxkernel[k0](proj((x0,v0),k0)) *
                        pdfpar[n](x)/pdfpar[n](x0) *
                        pdfaux[n](v)/pdfaux[n](v0) *
                        product([absdetjacf[n][l](x[l-1],v[l-1]) for l in range(1,L+1)])}
        else t0
\end{code}
\end{figure}

\subsubsection{Correctness of Multiple Step NP-iMCMC Algorithm}

The Multiple Step NP-iMCMC sampler cannot be formulated
as an instance of Hybrid NP-iMCMC
and requires a separate proof.
Nonetheless, the arguments are similar.

\begin{itemize}
  \item
  \cref{prop: tree and extree} tells us that
  as long as $\tree$ almost surely
  terminates (\cref{vass: almost surely terminating tree}),
  the measure of a $n$-dimensional
  parameter variable not having any instances
  in the support of $\tree$ tends to zero
  as the dimension $n$ tends to infinity.
  As $\nelem{n}{\ell}$ is bijective
  (and hence invertible),
  the Multiple Step NP-iMCMC sampler
  almost surely satisfies
  the condition in the loop
  in Step 3.2
  and hence almost surely terminates.

  \item Next, we identify the state distribution
  of Multiple Step NP-iMCMC.
  We say a $n$-dimensional state
  $(\enta,\auxa)$ is \defn{valid} if
  \begin{enumerate}[(i)]
    \item For all $\ell=1,\dots, L$,
    $\instances{\enta_{\ell}} \cap \support{\tree} \not=\varnothing$
    where
    $(\enta_0, \auxa_0) := (\enta, \auxa)$ and
    $(\enta_{\ell},\auxa_{\ell})
    := \nelem{n}{\ell}
      (\enta_{\ell-1},\auxa_{\ell-1})$; and

    \item For all $\ell=1,\dots, L$,
    $\instances{\entb_{\ell}} \cap \support{\tree} \not=\varnothing$
    where
    $(\entb_0, \auxb_0) := (\enta, \auxa)$ and
    $(\entb_{L-\ell+1},\auxb_{L-\ell+1})
    := \inv{\nelem{n}{\ell}}
      (\entb_{L-\ell},\auxb_{L-\ell})$; and

    \item For all $k < n$,
    $\take{k}(\enta,\auxa)$ is not a valid state.
  \end{enumerate}
  Then, we can define the state distribution
  and show that the state movement in Multiple
  Step NP-iMCMC is invariant against this distribution.

  \item Finally,
  we conclude by
  \cref{lemma: marginalised distribution is the target distribution} that
  the Multiple Step NP-iMCMC sampler is correct.
\end{itemize}

\subsubsection{Transforming Multiple Step NP-iMCMC Sampler}
\label{sec: tecniques on ms np-imcmc}

Recall we discussed three techniques in
\cref{app: variants},
each when applied to the Hybrid NP-iMCMC sampler,
improve its flexibility and/or efficiency.
We now see how these techniques can be applied
to Multiple Step NP-iMCMC.

\subsubsection{State-dependent Multiple Step NP-iMCMC Mixture}
\label{sec: state-dependent mixture of ms np-imcmc}

\begin{figure}
\begin{code}[label={code: mixture ms np-imcmc}, caption={Pseudocode of the State-dependent Multiple Step NP-iMCMC Mixture algorithm}]
def MixtureMSNPiMCMC(t0):
  k0 = dim(t0)                                  # initialisation step
  x0 = [(e, coin) if Type(e) in R else (normal, e) for e in t0]
  (@@)`m = mixkernel(x0)`                             # mixture step
  v0 = auxkernel[k0][m](x0)                     # stochastic step

  # start of multiple step
  n = k0
  (x[0],v[0]) = (x0,v0)
  for l in range(1,L+1):
    (x[l],v[l]) = `f[n][l][m]`(x[l-1],v[l-1])                # deterministic step
    while not intersect(instance(x[l]),support(w)):     # extend step
      x[0] = x[0] + [(normal, coin)]*(indexX(n+1)-indexX(n))
      v[0] = v[0] + [(normal, coin)]*(indexY(n+1)-indexY(n))
      for i in range(1,l+1):
        (y,u) = `slice[n+1][i][m]`(x[i-1],v[i-1])
        (x[i],v[i]) = (x[i]+y, v[i]+u)
      n = n + 1
  (x0,v0) = (x[0],v[0])
  (x,v) = (x[L],v[L])
  # end of multiple step

  t = intersect(instance(x),support(w))[0]      # accept/reject step
  k = dim(t)
  return t if uniform < min{1, w(t)/w(t0) *
                               pdfauxkernel[k][m](proj((x,v),k))/
                                 pdfauxkernel[k0][m](proj((x0,v0),k0)) *
                               pdfpar[n](x)/pdfpar[n](x0) *
                               pdfaux[n](v)/pdfaux[n](v0) *
                               (@@)`pdfmixkernel(proj(x,k),m)/`
                                 (@@)`pdfmixkernel(proj(x0,k0),m)` *
                               product([absdetjacf[n][m](x[l-1],v[l-1]) for l in range(1,L+1)])}
           else t0
\end{code}
% \end{figure}
% \begin{figure}
\begin{code}[label={code: mixture ms np-imcmc (correctness)},
  caption={Pseudocode for the correctness of State-dependent Multiple Step NP-iMCMC Mixture}]
def mixauxkernel[n](x0):
  m = mixkernel(x0)
  v0 = auxkernel[n][m](x0)
  return m + v0

def mixf[n][l](x0,mixv0):
  m = mixv0[:a]
  v0 = mixv0[a:]
  (x,v) = f[n][l][m](x0,v0)
  return (x,m + v)

def mixslice[n][l](x0,mixv0):
  m = mixv0[:a]
  v0 = mixv0[a:]
  (y,u) = slice[n][l][m](x0,v0)
  return (y,u)

def mixindexX(n): return indexX(n)
def mixindexY(n): return a + indexY(n)
def mixproj((x,v),k): return (x[:mixindexX(k)],v[:mixindexY(k)])
\end{code}
\end{figure}

This technique allows us to `mix'
Multiple Step NP-iMCMC samplers
in such a way that the resulting sampler still
preserves the posterior.
Given a collection of
Multiple Step NP-iMCMC samplers,
indexed by $\mixa \in \mixsp$ for some $\alpha \in \Nat$,
the \emph{State-dependent Multiple Step NP-iMCMC Mixture} sampler
draws an indicator $\mixa \in \mixsp$
from a probability measure
$\mixkernel(\enta_0, \placeholder)$ on $\mixsp$
where $\mixkernel:\bigcup_{n\in\Nat}\nparsp{n} \kernelto \mixsp$
is a probability kernel and
$\enta_0$ is the parameter variable
constructed from
the current sample $\traceb_0$ in
\cref{hnp-imcmc step: initialisation}.
A proposal $\traceb$ is then generated by running
Steps 2 and 3 of
the $\mixa$-indexed
Multiple Step NP-iMCMC sampler,
and is accepted with
a modified probability that
includes the probability of picking $\mixa$.

\paragraph{Pseudocode}

\cref{code: mixture ms np-imcmc} gives the
SPCF implementation of this sampler as
the \codein{MixtureMSNPiMCMC} function.
(Terms specific to this technique
are highlighted.)
We assume the SPCF term
\codein{mixkernel}
implements the mixture kernel $\mixkernel$;
\codein{pdfmixkernel} implements
the probability density function
$\mixpdf$; and
for each $\mixa \in \mixsp$ and $n\in\Nat$,
\codein{auxkernel[n][m]}
implements the auxiliary kernel and
\codein{pdfauxkernel[n][m]}
implements its density;
\codein{f[n][l][m]}
implements the endofunction
\codein{slice[n][l][m]}
implements its slice
and
\codein{absdetjacf[n][l][m]} implements
the absolute value of the
Jacobian determinant of
the endofunction
of the \codein{m}-indexed
Multiple Step NP-iMCMC sampler.

\paragraph{Correctness}

\codein{MixtureMSNPiMCMC} can be formulated as
an instance of \codein{MultistepNPiMCMC}
with
auxiliary kernel \codein{mixauxkernel[n]}
and its density \codein{mixpdfauxkernel[n]} and
\codein{L} number of endofunctions
\codein{mixf[n][l]}
(\codein{l} ranges from \codein{1} to \codein{L})
for each dimension \codein{n} with slice
\codein{mixslice[n][l]} and the absolute value of
its Jacobian determinant
\codein{absdetjacmixf[n][l]};
parameter and auxiliary index maps
\codein{mixindexX} and \codein{mixindexY} and
projection \codein{mixproj}
given in
\cref{code: mixture ms np-imcmc (correctness)}.

\subsubsection{Direction Multiple Step NP-iMCMC}
\label{sec: auxiliary direction ms np-imcmc}

\begin{figure}
\begin{code}[label={code: direction ms np-imcmc}, caption={Pseudocode of the Direction Multiple Step NP-iMCMC algorithm}]
def DirectionMSNPiMCMC(t0):
  (@@)`d0 = coin`                                     # direction step
  k0 = dim(t0)                                  # initialisation step
  x0 = [(e, coin) if Type(e) in R else (normal, e) for e in t0]
  v0 = auxkernel[k0](x0)                        # stochastic step
  n = k0                                        # multiple step
  (x[0],v[0]) = (x0,v0)
  for l in range(1,L+1):
    (x[l],v[l]) = `f[n][l][d0]`(x[l-1],v[l-1])            # deterministic step
    while not intersect(instance(x[l]),support(w)):     # extend step
      x[0] = x[0] + [(normal, coin)]*(indexX(n+1)-indexX(n))
      v[0] = v[0] + [(normal, coin)]*(indexY(n+1)-indexY(n))
      for i in range(1,l+1):
        (y,u) = `slice[n+1][i][d0]`(x[i-1],v[i-1])
        (x[i],v[i]) = (x[i]+y, v[i]+u)
      n = n + 1
  (x0,v0) = (x[0],v[0])
  (x,v) = (x[L],v[L])
  (@@)`d = not d0`                                            # flip direction (not used)
  t = intersect(instance(x),support(w))[0]      # accept/reject step
  k = dim(t)
  return t if uniform < min{1, w(t)/w(t0) *
                               pdfauxkernel[k](proj((x,v),k))/
                                 pdfauxkernel[k0](proj((x0,v0),k0)) *
                               pdfpar[n](x)/pdfpar[n](x0) *
                               pdfaux[n](v)/pdfaux[n](v0) *
                               product([`absdetjacf[n][l][d0]`(x[l-1],v[l-1]) for l in range(1,L+1)])}
           else t0
\end{code}
% \end{figure}
% \begin{figure}
\begin{code}[label={code: direction ms np-imcmc (correctness)},
  caption={Pseudocode for the correctness of Direction Multiple Step NP-iMCMC}]
def dirauxkernel[n](x0):
  d0 = coin
  v0 = auxkernel[n](x0)
  return [(normal, d0)] + v0

def dirf[n][l](x,dirv):
  d = dirv[0][1]
  v = dirv[1:]
  if l == dirL: return (x,[(dirv[0][0],not d)] + v)
  else:
    (x,v) = f[n][l][d](x,v)
    return (x, [(dirv[0][0],d)] + v)

def dirslice[n][l](x,dirv):
  d = dirv[0][1]
  v = dirv[1:]
  return slice[n][l][d](x,v)

dirL = L+1
def dirindexX(n): return indexX(n)
def dirindexY(n): return 1+indexY(n)
def dirproj((x,v),k): return (x[:dirindexX(k)], v[:dirindexY(k)])
\end{code}
\end{figure}

This technique allows us to relax
the assumption that the composition
$\nelem{n}{L} \circ
\dots \circ \nelem{n}{2} \circ \nelem{n}{1}$
is involutive.
Assume for $\ell=1,\dots,L$, both sets
$\set{\nelem{n}{\ell}}_n$ and
$\set{\inv{\nelem{n}{\ell}}}_n$
satisfy the \invoass{}
(\cref{vass: partial block diagonal inv}),
the \defn{Direction Multiple Step NP-iMCMC}
sampler randomly employ either
$\nelem{n}{L} \circ
\dots \circ \nelem{n}{2} \circ \nelem{n}{1}$
or
$\inv{\nelem{n}{1}} \circ
\inv{\nelem{n}{2}} \circ
\dots \circ
\inv{\nelem{n}{L}}$
to move around the $n$-dimensional state space
and proposes a new sample.

\paragraph{Pseudocode}

\cref{code: direction ms np-imcmc} gives the
SPCF implementation of this sampler as
\codein{DirectionMSNPiMCMC} function.
(Terms specific to this technique are highlighted.)
We assume for each $n\in\Nat$ and
$\dira \in \Bool$,
the SPCF term
\codein{f[n][l][True]} implements
the endofunction $\nelem{n}{\ell}$ and
\codein{f[n][l][False]} implements
the inverse $\inv{\nelem{n}{L-\ell+1}}$;
\codein{slice[n][l][True]} implements
the slice of $\nelem{n}{\ell}$ and
\codein{slice[n][l][False]} implements
the slice of $\inv{\nelem{n}{L-\ell+1}}$; and
\codein{absdetjacf[n][l][True]} implements
the absolute value of the Jacobian determinant of $\nelem{n}{\ell}$ and
\codein{absdetjacf[n][l][False]} implements
that of $\nelem{n}{L-\ell+1}$.

\paragraph{Correctness}

\codein{DirectionMSNPiMCMC}
can be formulated as an instance of
\codein{MultistepNPiMCMC}
with
auxiliary kernel \codein{dirauxkernel[n]}
and its density \codein{pdfdirauxkernel[n]} and
\codein{dirL} number of endofunctions
\codein{dirf[n][l]}
(\codein{l} ranges from \codein{1} to \codein{dirL})
for each dimension \codein{n} with slice
\codein{dirslice[n][l]} and the absolute value of
its Jacobian determinant
\codein{absdetjacf[n][l]};
parameter and auxiliary index maps
\codein{dirindexX} and \codein{dirindexY} and
projection \codein{dirproj}
given in
\cref{code: direction ms np-imcmc (correctness)}.
Note the \codein{dirf[n]} function denotes the
composition that
flips the direction after applying
the endofunctions $\nelem{n}{\ell}$ for
$\ell=1,\dots,L$
with an inverse
the flips the direction and then
apply the endofunctions $\nelem{n}{L-\ell+1}$ for
$\ell=1,\dots,L$.

\subsubsection{Persistent Multiple Step NP-iMCMC Algorithm}
\label{sec: persistent ms np-imcmc}

\begin{figure}
\begin{code}[
  label={code: persistent ms np-imcmc},
  caption={Pseudocode of the Persistent Multiple Step NP-iMCMC algorithm}]
def PersistentMSNPiMCMC(t0,`d0`):
  k0 = dim(t0)                                  # initialisation step
  x0 = [(e, coin) if Type(e) in R else (normal, e) for e in t0]
  v0 = `auxkernel[k0][d0]`(x0)                    # stochastic step
  n = k0                                        # multiple step
  (x[0],v[0]) = (x0,v0)
  for l in range(1,L+1):
    (x[l],v[l]) = `f[n][l][d0]`(x[l-1],v[l-1])            # deterministic step
    while not intersect(instance(x[l]),support(w)):     # extend step
      x[0] = x[0] + [(normal, coin)]*(indexX(n+1)-indexX(n))
      v[0] = v[0] + [(normal, coin)]*(indexY(n+1)-indexY(n))
      for i in range(1,l+1):
        (y,u) = `slice[n+1][i][d0]`(x[i-1],v[i-1])
        (x[i],v[i]) = (x[i]+y, v[i]+u)
      n = n + 1
  (x0,v0) = (x[0],v[0])
  (x,v) = (x[L],v[L])
  (@@)`d = not d0`                                            # flip direction
  t = intersect(instance(x),support(w))[0]      # accept/reject step
  k = dim(t)
  return (t, `not d`) if uniform < min{1, w(t)/w(t0) *
                       pdfauxkernel[k][d](proj((x,v),k))/
                         pdfauxkernel[k0][d0](proj((x0,v0),k0)) *
                       pdfpar[n](x)/pdfpar[n](x0) *
                       pdfaux[n](v)/pdfaux[n](v0) *
                       product([absdetjacf[n][l][d0](x[l-1],v[l-1]) for l in range(1,L+1)])}
           else (t0, `d`)
\end{code}
% \end{figure}
% \begin{figure}
\begin{code}[label={code: persistent ms np-imcmc (correctness)},
  caption={Pseudocode for the correctness of Persistent Multiple Step NP-iMCMC}]
def perauxkernel[n](perx0)
  d0 = perx0[0][1]; x0 = perx0[1:]
  v0 = auxkernel[n][d0](x0)
  return v0

def perf[n][l](perx,v)
  d = perx[0][1]; x = perx[1:]
  if l == perL: return ([(perx[0][0],not d)] + x, v)
  else:
    (x,v) = f[n][l][d](x,v)
    return ([(perx[0][0],d)] + x, v)

def perslice[n][l](perx,v)
  d = perx[0][1]; x = perx[1:]
  return slice[n][l][d](x,v)

perL = L+1
def perindexX(n): return 1+indexX(n)
def perindexY(n): return indexY(n)
def perproj((x,v),k): return (x[:perindexX(k)],v[:perindexY(k)])

def flipdir(perx0,v0):
  perx0[0][1] = not perx0[0][1]
  return (perx0,v0)
\end{code}
\end{figure}

This technique gives us a method to
construct irreversible Multiple Step NP-iMCMC samplers.
The key is to persist the direction
from a previous iteration.

The \defn{Persistent Multiple Step NP-iMCMC}
sampler keeps trace of a direction variable
$\dira_0\in\Bool$
(instead of sampling a fresh one at the start)
and use it to determine
the auxiliary kernel
($\nkernel{n}_\true:\nparsp{n}\kernelto\nauxsp{n}$
or
$\nkernel{n}_\false:\nparsp{n}\kernelto\nauxsp{n}$)
and list of endofunctions
($\nelem{n}{L} \circ
\dots \circ
\nelem{n}{1}$ or
$\inv{\nelem{n}{1}} \circ
\dots \circ
\inv{\nelem{n}{L}}$)
employed.
This direction variable is flipped
strategically to
make the resulting algorithm irreversible.

\paragraph{Pseudocode}

\cref{code: persistent ms np-imcmc} gives
the SPCF implementation of this sampler as
the function
\codein{PersistentMSNPiMCMC}.
(Terms specific to this technique are highlighted.)
In addition to the SPCF terms in \codein{DirectionMSNPiMCMC},
the SPCF term
\codein{auxkernel[n][True]} implements
the auxiliary kernel $\nkernel{n}_\true$ and
\codein{pdfauxkernel[n][True]}
implements its density $\nkernelpdf{n}_\true$
and
\codein{auxkernel[n][False]} implements
the auxiliary kernel $\nkernel{n}_\false$ and
\codein{pdfauxkernel[n][False]}
implements its density $\nkernelpdf{n}_\false$.
Note that
\codein{PersistentMSNPiMCMC}
updates samples
on the space $\nparsp{n} \times \Bool$,
which can easily be marginalised to $\nparsp{n}$
by taking the first $\iparsp(n)$ components.

\paragraph{Correctness}

Consider the \codein{MultistepNPiMCMC} function
with
auxiliary kernel \codein{perauxkernel[n]}
and its density \codein{pdfperauxkernel[n]} and
\codein{perL} number of endofunctions
\codein{perf[n][l]}
(\codein{l} ranges from \codein{1} to \codein{perL})
for each dimension \codein{n} with slice
\codein{perslice[n][l]} and the absolute value of
its Jacobian determinant
\codein{absdetjacperf[n][l]};
parameter and auxiliary index maps
\codein{perindexX} and \codein{perindexY} and
projection \codein{perproj}
given in
\cref{code: persistent ms np-imcmc (correctness)}.

The \codein{MultistepNPiMCMC} function with
the primitives indicated in \cref{code: persistent ms np-imcmc (correctness)}
is \emph{almost} equivalent to
\codein{PersistentMSNPiMCMC},
except
\codein{MultistepNPiMCMC}
induces a transition kernel on $\entsp \times \nparsp{n}$
whereas \codein{PersistentMSNPiMCMC}
induces a transition kernel on $\Bool \times \nparsp{n}$; and
when the proposal \codein{t} is accepted,
\codein{MultistepNPiMCMC}
returns \codein{d} whereas
\codein{PersistentMSNPiMCMC} returns
\codein{not d}.

By composing \codein{MultistepNPiMCMC}
with \codein{flipdir} which flips the direction and
marginalising the Markov chain generated
by the composition from
$\entsp \times \nparsp{n}$ to
$\Bool \times \nparsp{n}$,
we get
\codein{PersistentMSNPiMCMC}.

\clearpage

\section{Examples of Nonparametric Involutive MCMC}
\label{app: examples}
% !TEX root = ./../icml2022.tex

In this section, we design \emph{novel} nonparametric samplers using the
Hybrid NP-iMCMC method described in \cref{app: hybrid np-imcmc}
or the Multiple Step NP-iMCMC method
described in \cref{app: multiple step np-imcmc}.

We assume the target density function
$\tree$ on the
trace space $\traces$ is tree representable
and satisfies \cref{hass: integrable tree,hass: almost surely terminating tree}.
Specifications of the auxiliary kernels and involutions are given for each sampler.

\subsection{Nonparametric Metropolis-Hastings}
\label{app: np-mh}

As discussed in \cref{sec: imcmc},
the standard MH sampler can be seen as an instance of
the iMCMC sampler
with the proposal distribution $q$ as the auxiliary kernel and
a swap function as the involution.

Suppose
a proposal kernel $q^{(n)}:\entsp^{n} \kernelto\entsp^{n}$ exists
for each dimension $n\in\Nat$.
Setting
both $\iparsp$ and $\iauxsp$ to be identities
(which means
$\nparsp{n} = \nauxsp{n} = \entsp^n$
for all $n\in\Nat$),
the Hybrid NP-iMCMC method
(\cref{sec: np-imcmc algorithm})
gives an nonparametric extension of the MH sampler.

\begin{code}[label={code: np-mh}, caption={Pseudocode of the NP-MH algorithm}]
def NPMH(t0):
  k0 = dim(t0)                                  # initialisation step
  x0 = [(e, coin) if Type(e) in R else (normal, e) for e in t0]
  v0 = q[k0](x0)                                # stochastic step
  (x,v) = (v0,x0)                               # deterministic step
  while not intersect(instance(x),support(w)):  # extend step
    x0 = x0 + [(normal, coin)]
    v0 = v0 + [(normal, coin)]
    (x,v) = (v0,x0)
  t = intersect(instance(x),support(w))[0]      # accept/reject step
  k = dim(t)
  return t if uniform < min{1, w(t)/w(t0) * pdfq[k](proj((x,v),k))/pdfq[k0](proj((x0,v0),k0))}
           else t0
\end{code}

The \codein{NPMH} function in \cref{code: np-mh}
is a SPCF implementation of this sampler.
It can seen as an instance of
the \codein{NPiMCMC} function
with
\codein{auxkernel[n]} replaced by
the proposal distribution \codein{q[n]},
\codein{pdfauxkernel[n]} replaced by
the pdf of the proposal distribution \codein{pdfq[n]},
\codein{involution[n]} replaced by a swap function, and
\codein{indexX} and \codein{indexY} replaced by identities,
alongside a simplified acceptance ratio as
$(\enta,\auxa) = (\auxa_0,\enta_0) $,
$\nparpdf{n} = \nauxpdf{n}$, and
\[
  \frac{\nparpdf{n}(\enta)}{\nparpdf{n}(\enta_0)} \cdot
  \frac{\nauxpdf{n}(\auxa)}{\nauxpdf{n}(\auxa_0)} \cdot
  \abs{\det{\grad{\ninvo{n}}(\enta_0,\auxa_0)}}
  = 1.
\]

\subsection{Nonparametric Metropolis-Hastings with Persistence}
\label{app: lifted np-mh}

Following the persistent technique described in \cref{sec: persistent np-imcmc} for Hybrid NP-iMCMC,
we can form a nonreversible variant of
the NP-MH sampler described in \cref{app: np-mh}.
We call the resulting algorithm
the Nonparametric Metropolis-Hastings with Persistence (NP-MH-P) sampler.

Suppose
a proposal kernel $q^{(n)}:\entsp^{n} \kernelto\entsp^{n}$ exists
for each dimension $n\in\Nat$.
Similar to NP-MH,
both $\iparsp$ and $\iauxsp$ are
set to be identities
(which means
$\nparsp{n} = \nauxsp{n} = \entsp^n$
for all $n\in\Nat$).
Following \cite{DBLP:journals/corr/abs-0809-0916},
given a parameter variable $\enta \in \entsp^{n}$,
we can partition the auxiliary space $\entsp^{n}$ into two sets
$\auxsetb_{\enta, +} := \set{ \auxa \in \entsp^{n} \mid \eta(\auxa) \geq \eta(\enta)}$ and
$\auxsetb_{\enta, -} := \set{ \auxa \in \entsp^{n} \mid \eta(\auxa) < \enta(\enta)}$
where
$\eta:\entsp^{n} \to \Real$ is some measurable function;
and form two kernels
$\nauxkernel{n}_+$ and $\nauxkernel{n}_{-}$ from $q^{(n)}$ defined as
\begin{equation}
  \label{eq: np-mh persistent auxkernels}
  \nauxkernel{n}_+(\enta, \auxseta)
  :=
  \frac
    {q^{(n)}(\auxseta \cap \auxsetb_{\enta, +})}
    {q^{(n)}(\auxsetb_{\enta, +})}
  \qquad\text{ and }\qquad
  \nauxkernel{n}_{-}(\enta, \auxseta)
  :=
  \frac
    {q^{(n)}(\auxseta \cap \auxsetb_{\enta, -})}
    {q^{(n)}(\auxsetb_{\enta, -})}.
\end{equation}

Using the Persistent (Hybrid) NP-iMCMC sampler as described in \cref{sec: persistent np-imcmc},
a nonreversible variant of NP-MH can be
formed.

\begin{code}[label={code: lifted np-mh}, caption={Pseudocode of the NP-MH with Persistence algorithm}]
def NPMHwP(t0,`d0`):
  k0 = dim(t0)                                  # initialisation step
  x0 = [(e, coin) if Type(e) in R else (normal, e) for e in t0]
  v0 = `auxkernel[k0][d0]`(x0)                    # stochastic step
  (x,v) = (v0,x0)                               # deterministic step
  n = k0                                        # extend step
  while not intersect(instance(x),support(w)):
    x0 = x0 + [(normal, coin)]
    v0 = v0 + [(normal, coin)]
    n = n + 1
    (x,v) = (v0,x0)
  (@@)`d = not d0`
  t = intersect(instance(x),support(w))[0]      # accept/reject step
  k = dim(t)
  return (t, `not d`) if uniform < min{1, w(t)/w(t0) *
                               pdfauxkernel[k][d](proj((x,v),k))/
                                 pdfauxkernel[k0][d0](proj((x0,v0),k0))}
           else (t0, `d`)
\end{code}

The \codein{LiftedNPMH} function in \cref{code: lifted np-mh}
is a SPCF implementation of this sampler.
It can seen as an instance of
the \codein{PersistentNPiMCMC} function
(\cref{code: persistent np-imcmc})
with
\codein{auxkernel[n][True]} implementing
$\nauxkernel{n}_+$ and
\codein{auxkernel[n][False]} implementing
$\nauxkernel{n}_-$.

See how the direction \codein{d0}
(and hence the family of auxiliary kernels) is persisted
if the proposal \codein{t} is accepted.

\subsection{Nonparametric Hamiltonian Monte Carlo}
\label{app: np-hmc}
% !TEX root = ./../icml2022.tex

\begin{figure}
\begin{code}[caption={Pseudocode of the leapfrog step and its slice in NP-HMC}, label={code: steps in np-hmc}]
def leapfrog[n][m][d0](x,v):
  if d0:
    if m % 3 == 0 or m % 3 == 2:
      return (x, v-ep/2*grad(U)(x))       # half momentum update
    else: return (x+ep*v, v)              # full position update
  else:
    if m % 3 == 0 or m % 3 == 2:
      return (x, v+ep/2*grad(U)(x))       # inverse of half momentum update
    else: return (x-ep*v, v)              # inverse of full position update

def leapfrogslice[n][m][d0](x,v):
  if d0:
    if m % 3 == 0 or m % 3 == 2:
      return (x[-1], v[-1])               # slice of half momentum update
    else: return (x[-1]+ep*v[-1], v[-1])  # slice of full position update
  else:
    if m % 3 == 0 or m % 3 == 2:
      return (x[-1], v[-1])               # slice of inverse of half momentum
    else: return (x[-1]-ep*v[-1], v[-1])  # slice of inverse of full position
\end{code}
% \end{figure}
% \begin{figure}[h]
\begin{code}[caption={Pseudocode for NP-HMC}, label={code: np-hmc}]
def NPHMC(t0):
  d0 = coin                                     # direction step
  k0 = dim(t0)                                  # initialisation step
  x0 = t0
  v0 = [normal]*k0                              # stochastic step
  # start of multiple step
  n = k0
  (x[0],v[0]) = (x0,v0)
  for m in range(1,3L+1):
    (x[m],v[m]) = leapfrog[n][m][d0](x[m-1],v[m-1]) # deterministic step
    while not intersect(instance(x[m]),support(w)): # extend step
      x[0] = x[0] + [normal]
      v[0] = v[0] + [normal]
      for i in range(1,m+1):
        (y,u) = leapfrogslice[n][m][d0](x[m-1],v[m-1])
        (x[i],v[i]) = (x[i]+y, v[i]+u)
      n = n + 1
  (x0,v0) = (x[0],v[0])
  (x,v) = (x[3L],v[3L])
  d = d0
  # end of multiple step
  t = intersect(instance(x),support(w))[0]      # accept/reject step
  k = dim(t)
  return t if uniform < min{1,w(t)/w(t0) * pdfnormal[n](x)/pdfnormal[n](x0) *
                                           pdfnormal[n](v)/pdfnormal[n](v0)}
        else t0
\end{code}
\end{figure}

The Nonparametric Hamiltonian Monte Carlo (NP-HMC)
is a MCMC sampler
introduced by \cite{DBLP:conf/icml/MakZO21}
for probabilistic programming.
Here we show that it is
an instance of the
Direction Multiple Step NP-iMCMC sampler
(\cref{sec: auxiliary direction ms np-imcmc}).

Typically, the Hamiltonian Monte Carlo (HMC) sampler takes
a target density on $\Real^n$
and proposes
a new state by
simulating $L$ leapfrog steps:
\[
  \leapfrog :=
  (\momstep_{\epsilon/2} \circ \posstep_{\epsilon} \circ \momstep_{\epsilon/2})^L
\]
where
$\momstep_{\epsilon}(\enta,\auxa) :=
(\enta,\auxa-\epsilon \grad{U}(\enta))$ and
$\posstep_{\epsilon}(\enta,\auxa) :=
(\enta + \epsilon \auxa,\auxa)$
are the momentum and position updates
with step size $\epsilon$ respectively.
Notice that that the momentum and position updates
satisfy \invoass{} (\cref{hass: partial block diagonal inv}),
have inverses
$\inv{(\momstep_{\epsilon})} = \momflip \circ \momstep_{\epsilon} \circ \momflip$ and
$\inv{(\posstep_{\epsilon})} = \momflip \circ \posstep_{\epsilon} \circ \momflip$
where $\momflip(\enta, \auxa) := (\enta,-\auxa)$
and slices
$\drop{n-1}$ (for $\momstep_{\epsilon/2}$,
see \cref{sec: hmc slice} for more details) and
$(\enta, \auxa)\mapsto (\seqindex{\enta}{n} +
\epsilon\seqindex{\auxa}{n},\seqindex{\auxa}{n})$
(for $\posstep_{\epsilon}$)
respectively.
Moreover,
the absolute value of the Jacobian determinant
of both updates are
$\abs{\det{\grad{\momstep_{\epsilon}(\enta, \auxa)}}}
= \abs{\det{\grad{\posstep_{\epsilon/2}(\enta, \auxa)}}} = 1$.

Given a target density $\tree$ on
$\bigcup_{n\in\Nat}\Real^n$,
the HMC sampler can be extended using the
Direction Multiple Step NP-iMCMC sampler.
Given an input sample $\traceb_0 \in \Real^{k_0}$,
a $k_0$-dimensional initial state
$(\enta_0, \auxa_0)$ is formed
where $\enta_0 := \traceb_0$
and $\auxa_0$ drawn from
$\nkernel{n}(\enta, \placeholder) := \Gau_n$.
A direction variable $\dira_0$ is drawn to
determine
whether the leapfrog steps $\leapfrog$
or its inverse $\inv{\leapfrog}$
is performed on the initial state
$(\enta_0, \auxa_0)$,
one update at a time,
extending the dimension as required.
Say the initial state is extended to a
$n$-dimensional state
$(\enta_0, \auxa_0)$ and
is traversed to the $n$-dimensional new state
$(\enta^*, \auxa^*)$
which has an instance $\traceb$
in the support of $\tree$.
$\traceb$ is returned with probability
\begin{align*}
  \min\bigg\{1; \;
  & \frac
    {
      \tree{(\traceb)}\cdot
      \pdfGau_{n}(\enta^*)\cdot\pdfGau_{n}(\auxa^*)}
    {
      \tree{(\traceb_0)}\cdot
      \pdfGau_{n}(\enta_0)\cdot\pdfGau_{n}(\auxa_0)}
  \bigg\}.
\end{align*}

\paragraph{Pseudocode of NP-HMC}

\cref{code: steps in np-hmc} gives the
SPCF implementations
\codein{leapfrog[n][m]}
and \codein{leapfrogslice[n][m]},
where
\codein{leapfrog[n][m][True]}
and \codein{leapfrogslice[n][m][True]}
return the \codein{m}-th endofunction
and its slice in
the composition
$(\momstep_{\epsilon/2}\circ
\posstep_{\epsilon}\circ
\momstep_{\epsilon/2})^L$
of $3L$ updates respectively;
and similarly,
\codein{leapfrog[n][m][False]}
and \codein{leapfrogslice[n][m][False]}
return the \codein{m}-th endofunction
and its slice in
$(\inv{\momstep_{\epsilon/2}}\circ
\inv{\posstep_{\epsilon}}\circ
\inv{\momstep_{\epsilon/2}})^L.$

\cref{code: np-hmc} gives the SPCF implementation \codein{NPHMC}
of the NP-HMC sampler
as an instance of the Direction
Multiple Step NP-iMCMC sampler.
Importantly, the expensive
\codein{leapfrog[n][m]} function is called once
for each \codein{m} ranging from \codein{1} to \codein{3L}
the lightweight \codein{leapfrogslice}
is called in any subsequent re-runs.

\paragraph{Correctness}

Since both
$\momstep_{\epsilon/2}$ and
$\posstep_{\epsilon}$
are bijective and satisfies the \invoass{}
(\cref{hass: partial block diagonal inv}),
the correctness of NP-HMC is implied by
the correctness of Direction Multiple Step
NP-iMCMC sampler.

\subsection{Nonparametric Hamiltonian Monte Carlo with Persistence}
\label{app: gen np-hmc}
% !TEX root = ./../icml2022.tex

\begin{figure}
\begin{code}[
  label={code: gen nphmc},
  caption={Pseudocode for the NP-HMC with Persistence algorithm}
  ]
NPHMCwPersistent((x0,v0),d0) = PersistMom(CorruptMom((x0,v0),d0))

def HMCw(x,v): return w(x)

def CorruptMom((x0,v0),d0):
  u = [normal(v0[i]*sqrt(1-alpha^2), alpha^2) for i in range(len(v0))]
  return ((x0,u),d0)

def PersistMom((x0,v0),d0):
  k0 = dim(x0)                            # initialisation step
  # start of multiple step
  n = k0
  (x[0],v[0]) = (x0,v0)
  for m in range(1,3L+1):
    (x[m],v[m]) = leapfrog[n][m][d0](x[m-1],v[m-1])          # deterministic step
    while not intersect(instance(x[m],v[m]),support(HMCw)):  # extend step
      x[0] = x[0] + [normal]
      v[0] = v[0] + [normal]
      for i in range(1,m+1):
        (y,u) = leapfrogslice[n+1][i][d0](x[i-1],v[i-1])
        (x[i],v[i]) = (x[i]+y, v[i]+u)
      n = n + 1
  (@@)`d = not d0`                              # flip direction
  # end of multiple step
  (x,v) = intersect(instance(x[3L],v[3L]),support(HMCw))[0]  # accept/reject step
  return ((x,v), `not d`) if uniform < min{1, HMCw(x,v)/HMCw(x0,v0) *
                                            pdfnormal[n](x[3L])/pdfnormal[n](x[0]) *
                                            pdfnormal[n](v[3L])/pdfnormal[n](v[0]) }
           else ((x0,v0), `d`)
\end{code}
\end{figure}

With the catalogue of techniques
explored in \cref{sec: tecniques on ms np-imcmc},
novel irreversible variants of
the NP-HMC algorithm can be formed.
Here we focus on
the NP-HMC with Persistence algorithm
which can be seen as a nonparametric extension
of the Generalised HMC algorithm
\cite{HOROWITZ1991247}.

\subsubsection{Generalised HMC}

\citet{HOROWITZ1991247} made two changes to
the conventional HMC algorithm in order to
generate an irreversible Markov chain
on $\Real^n \times \Real^n$
and improve its performance.
\begin{enumerate}[1.]
  \item
  A ``corrupted'' momentum is
  used to move round the state space.

  \item
  The direction is ``persisted''
  if the proposal is accepted;
  otherwise it is negated.
\end{enumerate}
The resulting sampler is called the
\emph{Generalised HMC} algorithm as it is a
generalisation of the typical HMC sampler.

As shown in \cite{DBLP:conf/icml/NeklyudovWEV20},
the Generalised HMC algorithm can be
presented as a composition of
an iMCMC algorithm that
``corrupts'' the momentum and
a Persistent iMCMC algorithm
that uses Hamiltonian dynamics to find a new state
with a persisting direction.
We consider a similar approach in
our construction of a nonparametric extension
of Generalised HMC.

\subsubsection{NP-HMC with Persistence}

\paragraph{State Density}

Let the state $(\enta, \auxa)\in\Real^n\times\Real^n$ has density
$\tree'(\enta, \auxa) := \tree(\enta) $
w.r.t.~the normal distribution $\Gau_{2n}$.
It is clear that this density
$\tree'$ is integrable (\cref{hass: integrable tree})
and almost surely terminating (\cref{hass: almost surely terminating tree}).
By setting the parameter index map
to $\iparsp(n) := 2n$
and parameter space $\nparsp{n} :=
\Real^n\times\Real^n$,
the state $(\enta, \auxa)$ of length $2n$
is a $n$-dimensional parameter variable.
\codein{HMCw} in \cref{code: gen nphmc}
is a SPCF implementation of $\tree'$.

\paragraph{Corrupt Momentum}

Given the current state $(\enta_0,\auxa_0) \in \Real^n\times\Real^n$
with direction $\dira_0\in\Bool$,
a new momentum is drawn from the
distribution
${\Gau_n}(\auxa_0 \sqrt{1-\alpha^2},\alpha^2)$
for a hyper-parameter $\alpha\in[0,1)$.

This can be presented in the NP-iMCMC format
with the auxiliary variable $\auxb$ sampled from
${\Gau_n}(\auxa_0\sqrt{1-\alpha^2},\alpha^2)$
and the swap
$(((\enta_0, \auxa_0), \dira_0),\auxb) \mapsto
(((\enta_0, \auxb), \dira_0),\auxa_0)$
as the involution.
Since the new state $(\enta_0, \auxb)$
always have an instance in
the support of $\tree'$,
and the acceptance ratio is
\begin{align*}
  \min\Big\{1,
    \frac
    {
      \tree'(\enta_0,\auxb)\cdot
      \pdfGau_{2n}(\enta_0,\auxb)\cdot
      \pdf{\Bool}(\dira_0)\cdot
      \pdfGau_n(\auxa_0 \mid
      \auxb \sqrt{1-\alpha^2},\alpha^2)
    }
    {
      \tree'(\enta_0,\auxa_0)\cdot
      \pdfGau_{2n}(\enta_0,\auxa_0)\cdot
      \pdf{\Bool}(\dira_0)\cdot
      \pdfGau_n(\auxb \mid
      \auxa_0 \sqrt{1-\alpha^2},\alpha^2)
    }
  \Big\}
  = 1,
\end{align*}
the extend step (\cref{np-imcmc step: extend})
and
the accept/reject step (\cref{np-imcmc step: accept/reject})
of the NP-iMCMC sampler
can both be skipped.
This results in a sampler that has the
SPCF implementation
\codein{CorruptMom} in \cref{code: gen nphmc}.

\paragraph{Persist Momentum}

We consider the Persistent Multiple Step NP-iMCMC
algorithm (\cref{sec: persistent ms np-imcmc})
with the target density
$\tree'$ as follows.

Given a $k_0$-dimensional
parameter
$(\enta_0, \auxa_0)\in \nparsp{n} := \Real^n \times\Real^n$ and
direction $\dira_0 \in \Bool$,
a \emph{dummy} auxiliary variable $\auxb \in
\nauxsp{n} := \Real^n$
is sampled from
$\nkernel{n}((\enta_0, \auxa_0), \placeholder)
:= \Gau_n$
to form an initial state
$((\enta_0, \auxa_0), \auxb)$.
Depending on the direction $\dira_0$,
either
$\big(
(\momstep_{\epsilon/2}\times\id_{\Real^n}) \circ
(\posstep_{\epsilon}\times\id_{\Real^n}) \circ
(\momstep_{\epsilon/2}\times\id_{\Real^n})
\big)^L$
or its inverse
is performed on
$((\enta_0, \auxa_0), \auxb)$,
one update at a time,
extending the dimension as required.
Say the initial state is extended to a
$n$-dimensional
$((\enta_0^*, \auxa_0^*), \auxb^*)$ and
is traversed to the $n$-dimensional new state
$((\enta^*, \auxa^*), \auxb^*)$.
Then, the instance
$(\enta, \auxa) \in \support{\tree'}$ of
the $n$-dimensional parameter
$(\enta^*, \auxa^*)$
is returned
alongside the direction variable $\dira_0$
with probability
\begin{align*}
  \min\bigg\{1; \;
  & \frac
    {
      \tree'{(\enta,\auxa)}\cdot
      \pdfGau_{n}(\enta^*)\cdot\pdfGau_{n}(\auxa^*)}
    {
      \tree'{(\enta_0,\auxa_0)}\cdot
      \pdfGau_{n}(\enta_0^*)\cdot\pdfGau_{n}(\auxa_0^*)}
  \bigg\}.
\end{align*}

Note that the auxiliary variable $\auxb$
has no effect on the sampler.
Hence, \cref{code: gen nphmc}
gives a SPCF implementation
\codein{PersistMom} where
the stochastic step (Step 2) is skipped.

\paragraph{NP-HMC with Persistence}

Composing the samplers
which ``corrupts''
and persists the momentum
gives us the \defn{NP-HMC with Persistence}
algorithm, which is an nonparametric
extension of Generalised HMC.
\cref{code: gen nphmc} gives the SPCF
implementation \codein{NPHMCwPersistent}
by composing \codein{CorruptMom}
and \codein{PersistMom}.

\subsection{Nonparametric Look Ahead Hamiltonian Monte Carlo}
\label{app: look ahead np-hmc}
% !TEX root = ./../icml2022.tex

\begin{figure}
\begin{code}[
  label={code: look ahead nphmc},
  caption={Pseudocode for the NP Look Ahead HMC algorithm}
  ]
NPLookAheadHMC((x0,v0),d0) = ExtraLeapfrog(CorruptMom((x0,v0),d0)))

def ExtraLeapfrog((x0,v0),d0):
  k0 = dim(x0)                                  # initialisation step
  u = uniform                                   # stochastic step
  # start of multiple step
  n = k0
  m = 0
  (x[m],v[m]) = (x0,v0)

  stop = False
  while not stop:
    j = 1
    M = j*3*L
    # perform a set of leapfrog steps, i.e. to compute (x[i],v[i]) for i in range(m,M)
    while m < M+1:
      (x[m],v[m]) = leapfrog[n][m][d0](x[m-1],v[m-1])          # deterministic step
      while not intersect(instance(x[m],v[m]),support(HMCw)):  # extend step
        x[0] = x[0] + [normal]
        v[0] = v[0] + [normal]
        for i in range(1,m+1):
          (y,u) = leapfrogslice[n+1][i][d0](x[i-1],v[i-1])
          (x[i],v[i]) = (x[i]+y, v[i]+u)
        n = n + 1
      m = m + 1
    (x,v) = intersect(instance(x[M],v[M]),support(HMCw))[0]
    if u > min{1,HMCw(x,v)/HMCw(x0,v0) *
                 pdfnormal[n](x[M])/pdfnormal[n](x[0]) *
                 pdfnormal[n](v[M])/pdfnormal[n](v[0]) }:
      if j <= J:
        # perform an extra set of leapfrog steps
        j = j + 1
      else:
        # no leapfrog steps is performed
        (x,v) = (x0,v0)
        stop = True
        d = d0
    else:
      # enough leapfrog steps are performed
      stop = True
      d = not d0
  # end of multiple step
  return ((x,v), `not d`)
\end{code}
\end{figure}

Last but not least, we extend the
Look Ahead HMC algorithm \cite{DBLP:conf/icml/Sohl-DicksteinMD14},
which is equivalent to
the Extra Chance Generalised HMC algorithm \cite{DBLP:journals/jcphy/CamposS15}.

\subsubsection{Look Ahead HMC}

The Look Ahead HMC sampler modifies the
Generalised HMC algorithm by
performing \emph{extra} leapfrog steps when the proposal state is rejected.
This has the effect of increasing the acceptance rate for each proposal.

To see Look Ahead HMC as an instance of
Persistent iMCMC,
we consider the involution
$\iinv$ on $\Real^n\times\Real^n\times [0,1)\times \Bool$ given by
\begin{align*}
  \iinv(\enta,\auxa,u,\true)
  & := \begin{cases}
    (\leapfrog^{j}(\enta,\auxa),\displaystyle\frac{u}{\sigma_j},\false)
    & \text{if }
    \max\{\sigma_i\mid i < j\} \leq u < \min\{1, \sigma_j\} \\
    (\enta,\auxa,u,\true)
    & \text{if }
    \max\{\sigma_j\mid j \leq J\} \leq u
  \end{cases} \\
  \iinv(\enta,\auxa,u,\false)
  & := \begin{cases}
    (\leapfrog^{-j}(\enta,\auxa),\displaystyle\frac{u}{\sigma_j'},\true)
    & \text{if }
    \max\{\sigma_i'\mid i < j\} \leq u < \min\{1, \sigma_j'\} \\
    (\enta,\auxa,u,\false)
    & \text{if }
    \max\{\sigma_j'\mid j \leq J\} \leq u
  \end{cases}
\end{align*}
where
\[
  \sigma_j :=
  \displaystyle\frac
    {\spdf(\leapfrog^j(\enta,\auxa))}
    {\spdf(\enta,\auxa)},
  \qquad
  \sigma_j' :=
  \displaystyle\frac
    {\spdf(\leapfrog^{-j}(\enta,\auxa))}
    {\spdf(\enta,\auxa)},
  \qquad
  \leapfrog^{-j} := (\inv{\leapfrog})^j
\]
and $\spdf$ is the state density in HMC.
% \[
%   \sigma_k := \begin{cases}
%     \displaystyle\frac{\spdf(\leapfrog^j(\enta_0,\auxa_0))}
%          {\spdf(\enta_0,\auxa_0)}
%     & \text{if }{d}_0 > 0\\
%     \displaystyle\frac{\spdf(\momflip(\leapfrog^j(\momflip(\enta_0,\auxa_0)))}
%          {\spdf(\enta_0,\auxa_0)}
%     & \text{otherwise.}
%   \end{cases}
% \]

Note that in the involution,
$u$ determines how many sets ($j$) of leapfrog steps
are to be performed.
Say the direction is $\true$.
If the values of $\min\{1,\sigma_j\}$
for $j = 1,\dots, J$ are marked
on the unit interval $[0,1]$,
then
the probability that
$\leapfrog^j$ is performed
% a uniformly distributed
% random variable $u$ is in
% $[\max\{\sigma_i\mid i < j\}, \min\{1, \sigma_j\})$
can be represented by
the distance between $\min\{1,\sigma_j\}$
and the highest of $\sigma_i$ for $i < j$,
if it is non-negative.
\cref{fig: look ahead kernel}
gives an example of the result of
$\iinv(\enta, \auxa,u,\true)$
for varying $u \in [0,1]$.

\begin{figure}
\centering
\begin{tikzpicture}
  \draw[|-, thick] (-5,0) node[below,yshift=-5pt] {$0$}
  -- (-3,0) node[below,yshift=-5pt] {$\sigma_1$};
  \draw[|-, thick] (-3,0) -- (-1.5,0) node[below,yshift=-5pt] {$\sigma_3$};
  \draw[|-, thick] (-1.5,0) -- (2,0) node[below,yshift=-5pt] {$\sigma_2$};
  \draw[|-, thick] (2,0) -- (4,0) node[below,yshift=-5pt] {$\sigma_4$};
  \draw[|-|, thick] (4,0) -- (5,0) node[below,yshift=-5pt] {$1$};

  % 0 to s1
  \draw[<->] (-5,0.7) node[above,xshift=1cm]
  {$(\leapfrog^1(\enta, \auxa),\frac{u}{\sigma_1}, \false)$} -- (-3,0.7);
  % s1 to s2
  \draw[<->] (-3,0.5) node[above,xshift=3.5cm]
  {$(\leapfrog^2(\enta, \auxa),\frac{u}{\sigma_2}, \false)$} -- (2,0.5);
  % s2 to s4
  \draw[<->] (2,1.3) node[above,xshift=1cm]
  {$(\leapfrog^4(\enta, \auxa),\frac{u}{\sigma_4}, \false)$} -- (4,1.3);
  % s4 to 1
  \draw[<->] (4,0.5) node[above,xshift=0.5cm]
  {$(\leapfrog^0(\enta, \auxa),\frac{u}{\sigma_0}, \false)$} -- (5,0.5);
  % s3 to nth
  \draw[-] (-1.5,0) -- (-1.5,1.5) node[above]
  {$(\leapfrog^3(\enta, \auxa),\frac{u}{\sigma_3}, \false)$};
\end{tikzpicture}
\label{fig: look ahead kernel}
\caption{Result of $\iinv(\enta, \auxa,u,\true)$
for varying $u \in [0,1]$.}
\end{figure}

The Look Ahead HMC sampler
can be formulated as a
Persistent iMCMC sampler
with
the auxiliary kernel
$\ikernel((\enta,\auxa),\placeholder) := \Uni[0,1)$
and above involution $\iinv$.
Note that
the sampler always accept the proposal since
for $u \in [\max\{\sigma_i\mid i < j\}, \min\{1, \sigma_j\})$
with $j \in \set{1,\dots,J}$,
the acceptance ratio is
\begin{align*}
  \min\{1,
    \frac
    {\spdf(\leapfrog^j(\enta, \auxa))}
    {\spdf(\enta,\auxa)}
    \cdot
    \abs{\det{\grad{\iinv(\enta, \auxa,u,\true)}}}
  \}
  & =
  \min\{1,
    \sigma_j \cdot
    \abs{(\det{\grad{\leapfrog^j(\enta, \auxa)}})
    \cdot \frac{1}{\sigma_j}}
  \}
  =
  1
  % \\
  % \min\{1,
  %   \frac
  %   {\spdf(\leapfrog^{-j}(\enta, \auxa))}
  %   {\spdf(\enta,\auxa)}
  %   \cdot
  %   \abs{\det{\grad{\iinv(\enta, \auxa,u,\false)}}}
  % \}
  % & =
  % \min\{1,
  %   \sigma_j' \cdot
  %   \abs{(\det{\grad{\leapfrog^{-j}(\enta, \auxa)}})
  %   \cdot \frac{1}{\sigma_j'}}
  % \}
  % =
  % 1.
\end{align*}
and for $u \in [\max\{\sigma_j\mid j \leq J\},1]$,
the acceptance ratio is also $1$.
A similar argument can be made when the direction
is $\false$.

\subsubsection{NP Look Ahead HMC}

\paragraph{Extra Leapfrog}

Similar to the NP-HMC with Persistence,
we consider
the Persistent Multiple Step iMCMC algorithm
(\cref{sec: persistent ms np-imcmc})
that applies a random number of leapfrog function
(or its inverse)
to the current state
with the target density $\tree'(\enta,\auxa) := \tree(\enta)$.

Given a $k_0$-dimensional
parameter
$(\enta_0, \auxa_0)\in \Real^n \times\Real^n$ and
direction $\dira_0 \in \Bool$,
a random variable $u\in[0,1)$ and
a \emph{dummy} auxiliary variable
$\auxb_0 \in \Real^n$
are sampled from
the uniform distribution $\Uni(0,1)$ and
$\nkernel{n}((\enta_0, \auxa_0), \placeholder)
:= \Gau_n$ respectively
to form an initial state
$((\enta_0, \auxa_0), (u,\auxb_0))$.

If the direction $\dira_0$ is $\true$ and
$u \in [\max\{\sigma_i\mid i < j\}, \min\{1, \sigma_j\})$
for some $j>0$ where
\[
\sigma_j :=
\frac
{\max\{\tree'(\traceb,\auxb) \mid
\traceb \in \instances{\leapfrog^j(\enta,\auxa)}\}
\cdot\pdfGau_{2n}(\leapfrog^j(\enta,\auxa))}
{\max\{\tree'(\traceb,\auxb) \mid
\traceb \in \instances{(\enta,\auxa)}\}
\cdot\pdfGau_{2n}(\enta,\auxa)},
\]
leapfrog steps
$(\id_{\Real^n\times\Real^n}\times(\frac{1}{\sigma_j})\times\id_{\Real^n})\circ
(\leapfrog^j\times\id_{[0,1)\times\Real^n})$
% $(\id_{\Real^n\times\Real^n}\times(\frac{1}{\sigma_j})\times\id_{\Real^n})\circ
% \big(
% (\momstep_{\epsilon/2}\times\id_{[0,1)\times\Real^n}) \circ
% (\posstep_{\epsilon}\times\id_{[0,1)\times\Real^n}) \circ
% (\momstep_{\epsilon/2}\times\id_{[0,1)\times\Real^n})
% \big)^{jL}$
are performed on
$((\enta_0, \auxa_0), (u,\auxb_0))$,
one update at a time,
extending the dimension as required
with a flipped direction $\false$.
Otherwise,
$u \geq \max\{\sigma_j\mid j \leq J\}$ and
no leapfrog steps is performed;
$((\enta_0, \auxa_0), (u,\auxb_0))$ is returned
with the direction $\true$ remains unchanged.
The treatment when $\dira_0$ is $\false$
is similar.

Say
the initial state with direction $\dira_0$
is extended to a
$n$-dimensional
$((\enta_0^*, \auxa_0^*), (u,\auxb_0^*))$ and
is traversed to the $n$-dimensional new state
$((\enta^*, \auxa^*), (u^*,\auxb^*))$ with
direction $\dira$.
The instance
$(\enta, \auxa) \in \support{\tree'}$ of
the $n$-dimensional parameter
$(\enta^*, \auxa^*)$
is returned
alongside a flipped direction $\mathsf{not}\ \dira$
with probability
\begin{align*}
  \min\bigg\{1; \;
  & \frac
    {
      \tree'{(\enta,\auxa)}\cdot
      \pdfGau_{2n}(\enta^*,\auxa^*)\cdot
      \pdf{\Uni(0,1)}(u^*)\cdot
      \pdfGau_{n}(\auxb^*)}
    {
      \tree'{(\enta_0,\auxa_0)}\cdot
      \pdfGau_{2n}(\enta_0^*,\auxa_0^*)\cdot
      \pdf{\Uni(0,1)}(u)\cdot
      \pdfGau_{n}(\auxb_0^*)}
    \cdot
    \abs{\frac{1}{\sigma_j}}\cdot
    \abs{\det{\grad{\leapfrog^j}(\enta_0,\auxa_0)}}
  \bigg\} = 1
\end{align*}
if $j > 0$.
Otherwise ($j=0$),
the acceptance ratio is also $1$.

Note that the auxiliary variable $\auxb_0$
has no effect
on the sampler.
Hence, \cref{code: look ahead nphmc}
gives a SPCF implementation
\codein{ExtraLeapfrog} where
the sampling of the auxiliary variable $\auxb_0$ is skipped.

\paragraph{NP Look Ahead HMC}
Combining \codein{ExtraLeapfrog} with \codein{CorruptMom},
the \codein{NPLookAheadHMC} function
in \cref{code: look ahead nphmc}
implements the
NP Look Ahead HMC sampler.

\clearpage

\iffalse
\section{Additional Proofs}
\label{app: proofs}
\input{section/proofs}
\clearpage
\fi
\else
\fi
%%%%%%%%%%%%%%%%%%%%%%%%%%%%%%%%%%%%%%%%%%%%%%%%%%%%%%%%%%%%%%%%%%%%%%%%%%%%%%%
%%%%%%%%%%%%%%%%%%%%%%%%%%%%%%%%%%%%%%%%%%%%%%%%%%%%%%%%%%%%%%%%%%%%%%%%%%%%%%%

\lo{\paragraph{Convention}
\begin{compactenum}
  \item git pull before starting to make any changes (helps prevent merge conflicts)
  \item one sentence per line (makes diffs more readable, and \LaTeX\ backward search more precise)
  \item use \textbackslash changed[authorname] when making changes you want others to check;
  after checking, uncolour the changes
  \item use cleveref (\textbackslash cref command) for referencing (helps with consistency)
  \item use macros judiciously: only if used in several places; give them understandable names (or if that would be too long, add a comment to the definition)
  \item use \textbackslash[...\textbackslash] for equations, not \$\$...\$\$ (for consistent spacing)
  \item captions and titles of (sub)section, definitions, theorems, etc.: all lowercase except opening word
  %\item Spelling: for consistency, happy to use ``...ize'' (e.g.~normalize) rather than ``...ise'' (though my preference is the latter).
  \item Use $:=$ for definition, and highlight the definiendum using \textbackslash {\tt defn}.
\end{compactenum}}

\end{document}